\documentclass[notes=show,beamer]{beamer}%
\usepackage{amsmath}
\usepackage{amssymb}
\usepackage{mathpazo}
\usepackage{hyperref}
\usepackage{multimedia}
\usepackage{amsfonts}
\usepackage{graphicx}
\providecommand{\U}[1]{\protect\rule{.1in}{.1in}}
\begin{document}

\title{\textsl{Watersheds, waterfalls, on edge or node weighted graphs}}
\author{Fernand Meyer\\Centre de Morphologie Math\'{e}matique\\Mines-ParisTech}
\date{2012 February 29}
\maketitle

\section{Introduction}%

\begin{frame}%
%

\frametitle{Introduction}%

The watershed transform is one of the major image segmentation tools
\cite{beucher79}, used in the community of mathematical morphology and beyond.
If the watershed is a successful concept, there is another side of the coin: a
number of definitions and algorithms coexist, claiming to construct a
wartershed line or catchment basins, although they obviously are not
equivalent. We have presented how the idea was conceptualized and implemented
as algorithms or hardware solutions in a brief note :" The watershed concept
and its use in segmentation : a brief history" (arXiv:1202.0216v1), which
contains an extensive bibliography.\ See also \cite{Roerdink01thewatershed}
for an extensive review on the watershed concepts and construction modes.%

\end{frame}%
%

\begin{frame}%
%

\frametitle{Introduction}%

The present work studies the topography of a relief defined on a node or edge
weighted graph with morphological tools.\ The catchment basin of a minimum is
defined as the sets of points it is possible to reach by a non descending path
starting from this minimum.\ The watershed zone, i.e. the points linked with
two distinct minima through a non ascending path, is more and more reduced as
one considers steeper paths.%

\end{frame}%
%

\begin{frame}%
%

\frametitle{Outline}%

\begin{itemize}
\item Reminders on weighted graphs

\item Distances on node or edge weighted graphs

\item Adjunctions between nodes and edges

\item The flooding adjunction.\ 

\item Invariants of the flooding and closing adjunction

\item Flooding graph
\end{itemize}

%

\end{frame}%
%

\begin{frame}%
%

\frametitle{Outline}%

\begin{itemize}
\item Paths of steepest descent and k-steep graphs

\item The scissor operator, minimum spanning forests and watershed partitions

\item Lexicographic distances and SKIZ

\item The waterfall hierarchy.\ 

\item Emergence and role of the minimum spanning tree

\item Discussion and conclusion
\end{itemize}

%

\end{frame}%
%

\begin{frame}%

\begin{center}
{\Large \alert{Reminder on adjunctions}}
\end{center}

The following section contains a reminder on adjunctions linking erosions and
dilations by pairs and from which openings and closings are derived
\cite{serra88},\cite{heijmans90},\cite{heijmans94}%

\end{frame}%
%

\begin{frame}%
%

\frametitle{Preamble on adjunctions}%

Let $\mathcal{T}$ a complete totally ordered lattice, and $\mathcal{D}%
$,$\mathcal{E}$ be arbitrary sets \newline$O:$ the smallest element and
$\Omega\,:$ the largest element of $\mathcal{T}$\newline$\operatorname*{Fun}%
(\mathcal{D}$,$\mathcal{T)}$ : the image defined on the support $\mathcal{D}$
with value in $\mathcal{T}\newline\operatorname*{Fun}(\mathcal{E}%
$,$\mathcal{T)}$ : the image defined on the support $\mathcal{E}$ with value
in $\mathcal{T}$

\bigskip

Let $f$ be a function of $\operatorname*{Fun}(\mathcal{D}$,$\mathcal{T)}$ and
$g$ be a function of $\operatorname*{Fun}(\mathcal{E}$,$\mathcal{T)}$%
\newline$\alpha$ : $\operatorname*{Fun}(\mathcal{D}$,$\mathcal{T)}%
\rightarrow\operatorname*{Fun}(\mathcal{E}$,$\mathcal{T)}$ and \newline$\beta$
: $\operatorname*{Fun}(\mathcal{E}$,$\mathcal{T)}\rightarrow
\operatorname*{Fun}(\mathcal{D}$,$\mathcal{T)}$ be two operators

\begin{definition}
$\alpha$ and $\beta$ form an adjunction if and only if :\newline for any $f$
in $\operatorname*{Fun}(\mathcal{D}$,$\mathcal{T)}$ and $g$ in
$\operatorname*{Fun}(\mathcal{E}$,$\mathcal{T)}$ : $\alpha f<g\Leftrightarrow
f<\beta g$
\end{definition}

%

\end{frame}%
%

\begin{frame}%
%

\frametitle{Erosions and dilations}%

\begin{theorem}
If ($\alpha,\beta)$ form an adjunction, then $\alpha$ is a dilation (it
commutes with the supremum of functions in $\operatorname*{Fun}(\mathcal{D}%
$,$\mathcal{T)}$ )\newline and $\beta$ is an erosion (it commutes with the
infimum of functions in $\operatorname*{Fun}(\mathcal{E}$,$\mathcal{T)}$ )
\end{theorem}

\textbf{Proof: }Let $f$ be a function of $\operatorname*{Fun}(\mathcal{D}%
$,$\mathcal{T)}$ and $\left(  g\right)  _{i}$ functions of
$\operatorname*{Fun}(\mathcal{E}$,$\mathcal{T)}$\newline Then $f<\beta%
{\textstyle\bigwedge\limits_{i}}
g_{i}\Leftrightarrow\alpha f<%
{\textstyle\bigwedge\limits_{i}}
g_{i}\Leftrightarrow\forall i:\alpha f<g_{i}\Leftrightarrow\forall i:f<\beta
g_{i}\Leftrightarrow f<%
{\textstyle\bigwedge\limits_{i}}
\beta g_{i}$\newline Reading these equivalences from left to right and
replacing $f$ by $\beta%
{\textstyle\bigwedge\limits_{i}}
g_{i}$ implies that $\beta%
{\textstyle\bigwedge\limits_{i}}
g_{i}<%
{\textstyle\bigwedge\limits_{i}}
\beta g_{i}$.\newline Reading these equivalences from right to left and
replacing $f$ by $%
{\textstyle\bigwedge\limits_{i}}
\beta g_{i}$ implies that $%
{\textstyle\bigwedge\limits_{i}}
\beta g_{i}<\beta%
{\textstyle\bigwedge\limits_{i}}
g_{i}$\newline establishing that $\beta%
{\textstyle\bigwedge\limits_{i}}
g_{i}=%
{\textstyle\bigwedge\limits_{i}}
\beta g_{i}.\ $%

\end{frame}%
%

\begin{frame}%
%

\frametitle{Erosions and dilations (2)}%

Similarly we show that $\delta$ is a dilation, i.e.\ commutes with the union.

\textbf{Remark: }Calling $O\ $a constant function equal to $O,$ we have for
any $f$ : $O<\beta f$ implying $\alpha O<f.\ $This relation is true for all
$f,$ indicating that $\alpha O=O.$\newline Similarly, we show that
$\beta\Omega=\Omega,$ where $\Omega$ is the constant function equal to
$\Omega.$%

\end{frame}%
%

\begin{frame}%
%

\frametitle{Increasing operators}%

\begin{Lemma}
$\alpha $ and $\beta $ are increasing operators.
\end{Lemma}

\begin{proof}
If $f<g$ then $\beta(f\wedge g)=\beta(f)=\beta(f)\wedge\beta(g)<\beta(g)$
\end{proof}

%

\end{frame}%
%

\begin{frame}%
%

\frametitle{From one operator to the other}%

From the relation $\alpha f<g\Leftrightarrow f<\beta g$ one derives
expressions of one operator in terms of the other one.\ 

If $\alpha$ is known, $\beta$ may be expressed as $\beta g=%
{\textstyle\bigvee}
\left\{  f\mid\alpha f<g\right\}  .$

Inversely if $\beta$ is known, then $\alpha f=%
{\textstyle\bigwedge}
\left\{  g\mid f<\beta g\right\}  $%

\end{frame}%
%

\begin{frame}%
%

\frametitle{Opening and closing}%

\begin{lemma}
The operator $\beta\alpha$ is a closing : increasing, extensive and
idempotent.\ Similarly the operator $\alpha\beta$ is an opening : increasing,
anti-extensive and idempotent.\ 
\end{lemma}

\textbf{Proof: }Openings and closings, obtained by composition of increasing
operators, are increasing.\ \newline By adjunction we obtain $\alpha f<\alpha
f\Rightarrow f<\beta\alpha f$ showing that the closing is extensive and
$\alpha\beta f<f\Longleftarrow\beta f<\beta f$ showing that the opening is
anti-extensive. \newline In particular, applying an opening to $\alpha f$
yields $\alpha f>\alpha\beta\alpha f$ .\newline On the other hand, $\alpha$
being increasing, and the closing extensive :\newline$f<\beta\alpha
f\Rightarrow\alpha f<\alpha\beta\alpha f$ showing that $\alpha f=\alpha
\beta\alpha f$\newline Applying $\beta$ on both sides yields $\beta\alpha
f=\beta\alpha\beta\alpha f$%

\end{frame}%
%

\begin{frame}%

\begin{center}
{\Large \alert{The family of invariants of an opening or a closing}}
\end{center}

\bigskip

\bigskip

We call $\operatorname*{Inv}(\gamma)$ and \ $\operatorname*{Inv}(\varphi)$ the
family of invariants of $\gamma$ and $\varphi.$%

\end{frame}%
%

\begin{frame}%
%

\frametitle{The family of invariants of an opening}%

\begin{lemma}
The family of invariants of an opening is closed by union. And the family of
closings is closed by intersection.
\end{lemma}

\textbf{Proof: }Let us prove it for openings ; the result for closings being
obtained by duality.\ Suppose that $g_{1}$ and $g_{2}$ are invariant by the
opening $\gamma$ : $\gamma(g_{1})=g_{1}$ and $\gamma(g_{2})=g_{2}.$

Then $\gamma(g_{1}\vee g_{2})\leq g_{1}\vee g_{2}=\gamma(g_{1})\vee
\gamma(g_{2})$ by antiextensivity

And $\gamma(g_{1}\vee g_{2})\geq\gamma(g_{1})\vee\gamma(g_{2})$ as $\gamma$ is increasing.%

\end{frame}%
%

\begin{frame}%
%

\frametitle{Invariants of an opening or a closing}%

\begin{itemize}
\item $\alpha f\in\operatorname*{Inv}(\gamma)$\ as $\alpha f=\alpha\beta\alpha
f=(\alpha\beta)\alpha f$

\item $\beta f\in\operatorname*{Inv}(\varphi)$\ as $\beta f=\beta\alpha\beta
f=(\beta\alpha)\beta f$
\end{itemize}

%

\end{frame}%
%

\begin{frame}%
%

\frametitle{The opening as pseudo-inverse operator of the erosion}%

* The erosion has no inverse as $\alpha\beta g\leq g.$

\bigskip

* The opening is the pseudo inverse of the erosion : $\alpha\beta g\ $is the
smallest set having the same erosion as $g.$\newline

\textbf{Proof : }Suppose that $f$ verifies $\beta f=\beta g$ which implies
$\alpha\beta f=\alpha\beta g.\ $If $f\leq\alpha\beta g,$ we have $\alpha\beta
f\leq f\leq\alpha\beta g$ showing that $f=\alpha\beta g$

\bigskip

* if If $g\in\operatorname*{Inv}(\gamma)$, we have $\gamma g=\alpha\beta g=g$
: on $\operatorname*{Inv}(\gamma),$ $\alpha=\beta^{-1}$%

\end{frame}%
%

\begin{frame}%
%

\frametitle{The closing as pseudo-inverse operator of the dilation}%

* The dilation has no inverse as $\beta\alpha g\geq g.$

\bigskip

* The closing is the pseudo inverse of the dilation : $\beta\alpha g\ $is the
largest set having the same dilation as $g.$

\bigskip

* if If $g\in\operatorname*{Inv}(\varphi)$, we have $\varphi g=\beta\alpha
g=g$ : on $\operatorname*{Inv}(\varphi),$ $\beta=\alpha^{-1}$%

\end{frame}%
%

\begin{frame}%

\begin{center}
{\Large \alert{Reminder on graphs}}
\end{center}

%

\end{frame}%
%

\begin{frame}%
%

\frametitle{Graphs : General definitions}%

A \textit{non oriented graph} $G=\left[  N,E\right]  $ : $N$ = vertices or
nodes ; $E$ = edges : an edge $u\in E$ = pair of vertices (see \cite{berge85}%
,\cite{gondranminoux})

\bigskip

\textit{A chain} of length $n$ is a sequence of $n$ edges $L=\left\{
u_{1},u_{2,}\ldots,u_{n}\right\}  $, such that each edge $u_{i}$ of the
sequence $\left(  2\leq i\leq n-1\right)  $ shares one extremity with the edge
$u_{i-1}$ $(u_{i-1}\neq u_{i})$, and the other extremity with $u_{i+1}$
$(u_{i+1}\neq u_{i})$.

\bigskip

A \textit{path }between two nodes $x$ and $y$ is a sequence of nodes
$(n_{1}=x,n_{2},...,n_{k}=y)$ such that two successive nodes $n_{i}$ and
$n_{i+1}$ are linked by an edge.\ 

\bigskip

\textit{A cycle} is a chain or a path whose extremities coincide.

\bigskip

\textit{A cocycle} is the set of all edges with one extremity in a subset $Y$
and the other in the complementary set $\overline{Y}.$%

\end{frame}%
%

\begin{frame}%
%

\frametitle{Graphs : partial graphs and subgraphs}%

The subgraph spanning a set $A\subset N:G_{A}=[A,E_{A}]$, where $E_{A}$ are
the edges linking two nodes of $A.$

The partial graph associated to the edges $E^{\prime}\subset E$ is $G^{\prime
}=[N,E^{\prime}]$

For \textit{contracting} an edge $(i,j)$ in a graph $G$, one suppresses this
edge $u$ and its two extremities are merged into a unique node $k.\ $All edges
incident to $i$ or to $j$ become edges incident to the new node $k$.

\textit{Contraction : }

If $H$ is a subgraph of $G$, The operator $\kappa:(G,H)\rightarrow G^{\prime}$
contracts all edges of $H$ in $G$%

\end{frame}%
%

\begin{frame}%
%

\frametitle{Graphs : Connectivity}%

\textit{A connected graph }is a graph where each pair of nodes is connected by
a path

\bigskip

\textit{A tree} is a connected graph without cycle.

\bigskip

A\textit{\ spanning tree} is a tree containing all nodes.

\bigskip

\textit{A forest } is a collection of trees.\ 

\bigskip

\textit{Labelling a graph = }extracting the maximal connected subgraphs%

\end{frame}%
%

\begin{frame}%
%

\frametitle{Weighted graphs}%

In a graph $\left[  N,E\right]  $, $N$ represents the nodes, $E$ represents
the edges and $\mathcal{T}$ is the set of weights of edges or nodes.\ ($O$
being the lowest weight and $\Omega$ the largest)

Edges and nodes may be weighted : $e_{ij}$ is the weight of the edge $(i,j)$
and $n_{i}$ the weight of the node $i.$

The following weight distributions are possible$:$

$(-,\diamond)$ : no weights

$(-,n)$ : weights on the nodes

$(e,\diamond)$ : weights on the edges

$(e,n)$ : weights on edges and nodes \bigskip

\bigskip

Various types of graphs are met in image processing.%

\end{frame}%
%

\begin{frame}%
%

\frametitle{Which types of graphs ?}%

In "pixel graphs" the nodes are the pixels and the edges connect neighboring
pixels. The weights of the pixels are their value and the weights of edges may
be for instance a gradient value computed between their extremities.%

\begin{figure}
[ptb]
\begin{center}
\includegraphics[
height=1.9539in,
width=2.2332in
]%
{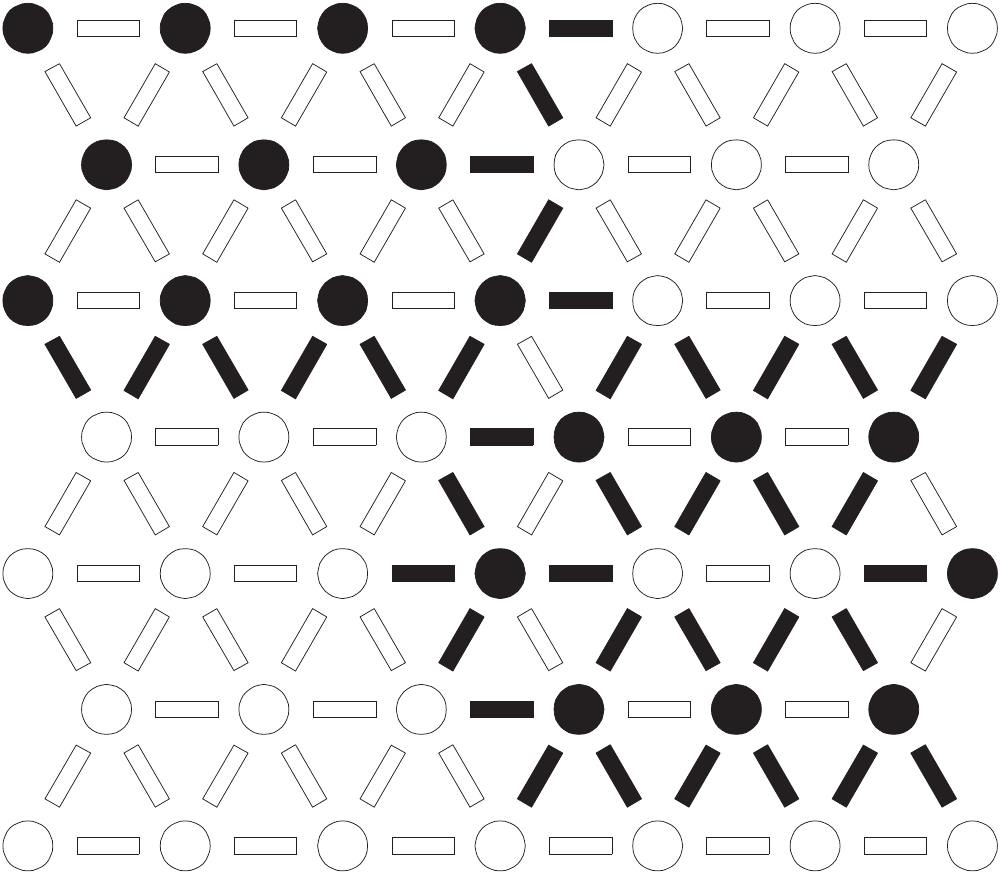}%
\end{center}
\end{figure}
%

\end{frame}%
%

\begin{frame}%
%

\frametitle{Which types of graphs ?}%

We will work with "neighborhood graphs" where the nodes are the catchment
basins and the edges connect neighboring bassins.\ The edges are weighted by a
dissimilarity measure between adjacent catchment basins; the simplest being
the altitude of the path-point between two basins.

\begin{figure}[ptb]
\begin{center}
\includegraphics[
natheight=5.040100in,
natwidth=4.161500in,
height=1.8919in,
width=1.5682in
]{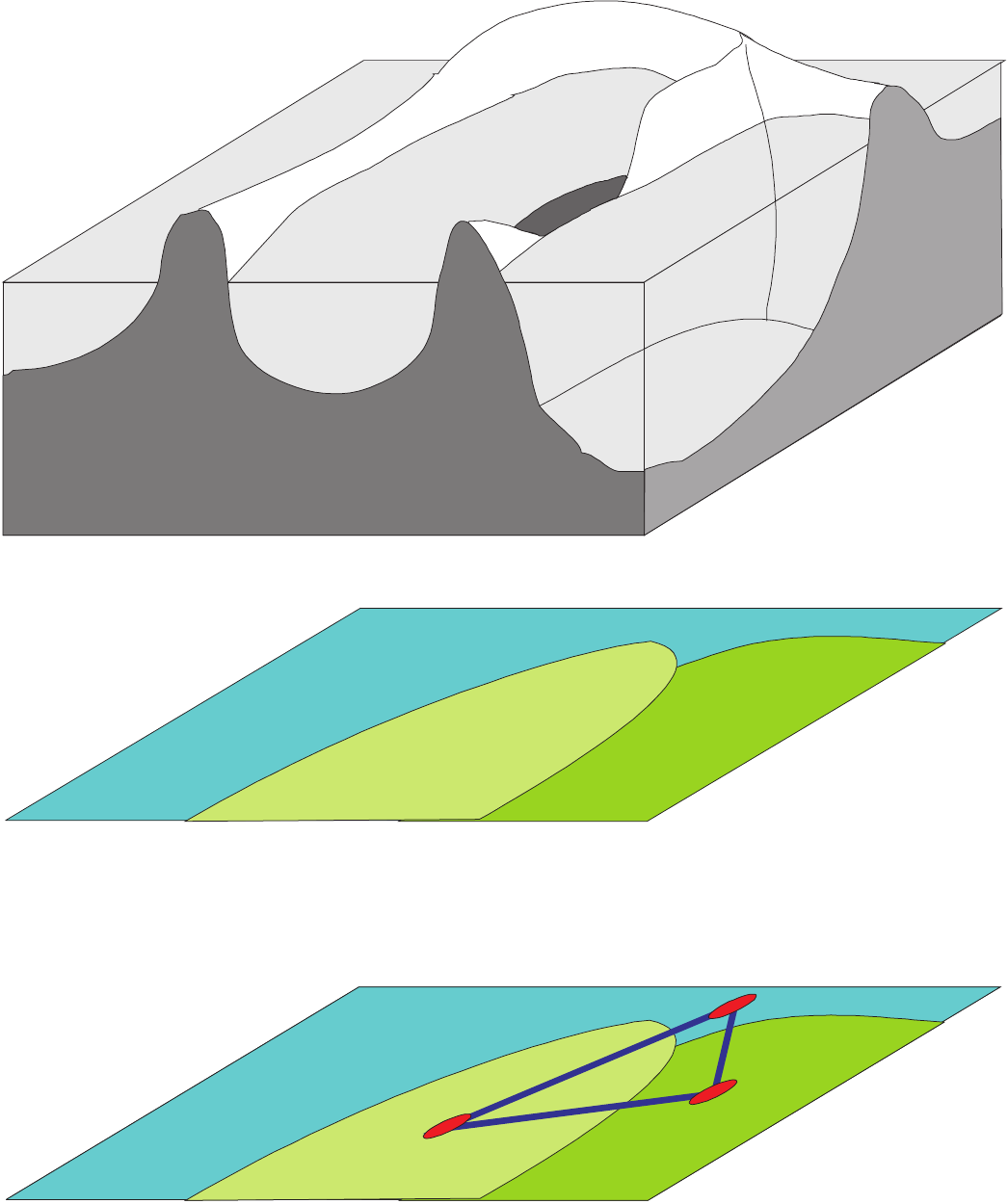}
\end{center}
\end{figure}
%

\end{frame}%
%

\begin{frame}%
%

\frametitle{The gabriel graph}%

Gabriel and Delaunay graphs permit to define neighborhood relations between
the sets and nodes of a population.

\begin{figure}[ptb]
\begin{center}
\includegraphics[
natheight=4.575400in,
natwidth=7.825000in,
height=1.7208in,
width=2.9217in
]{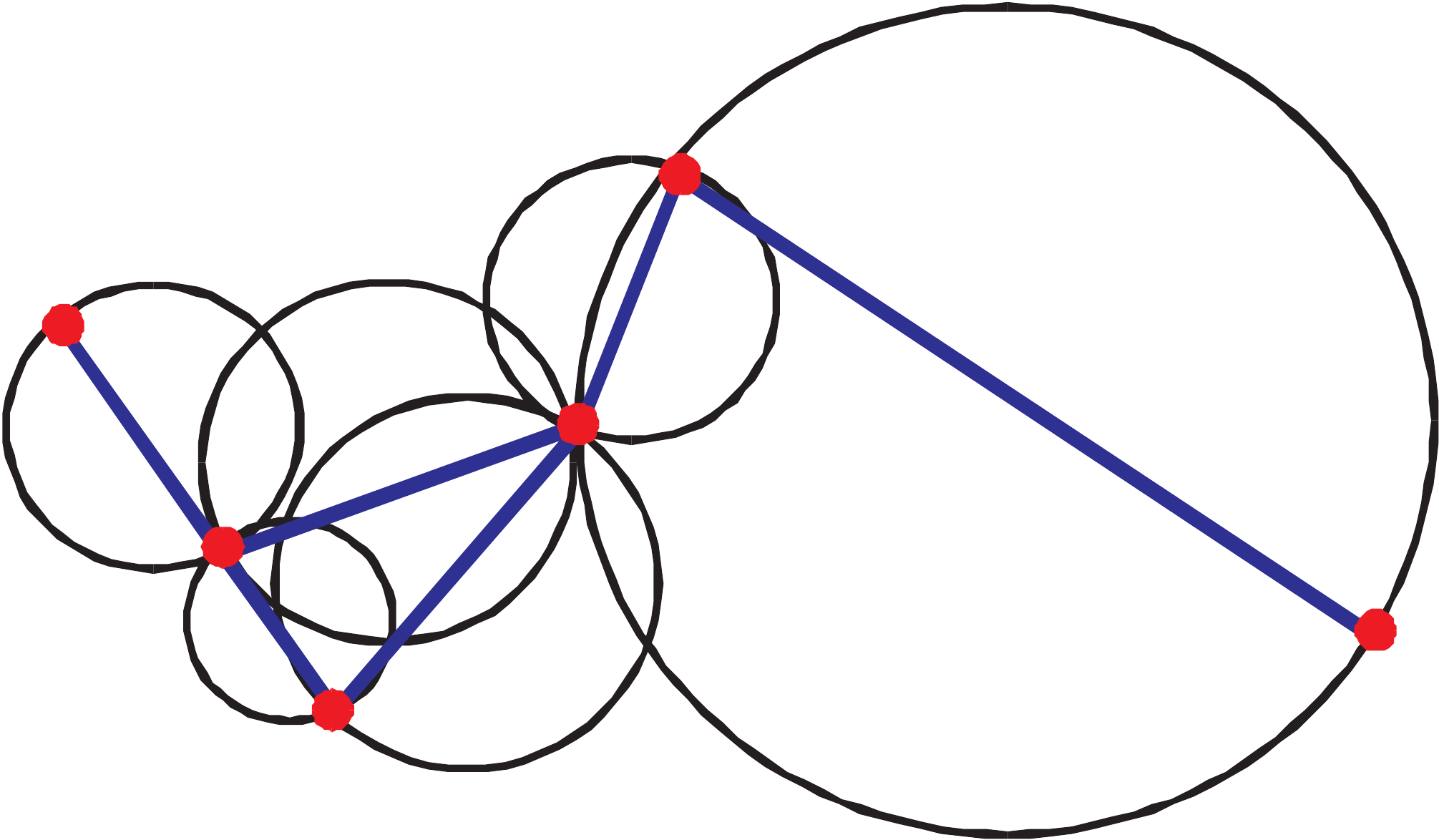}
\end{center}
\end{figure}

Two nodes x and y are linked by an edge if there exists no other node in the
disk of diameter (xy).\ Node weights will express features of the nodes, edge
weights relationships between adjacent nodes.%

\end{frame}%
%

\begin{frame}%
%

\frametitle{Gabriel graph modeling a lymphnode}%
The nuclei of a lymphnode are segmented and their Gabriel graph constructed.%

\begin{columns}[5cm]%
\column{5cm}%
\begin{figure}
[ptb]
\begin{center}
\includegraphics[
natheight=12.134600in,
natwidth=11.697700in,
height=2.3322in,
width=2.2491in
]%
{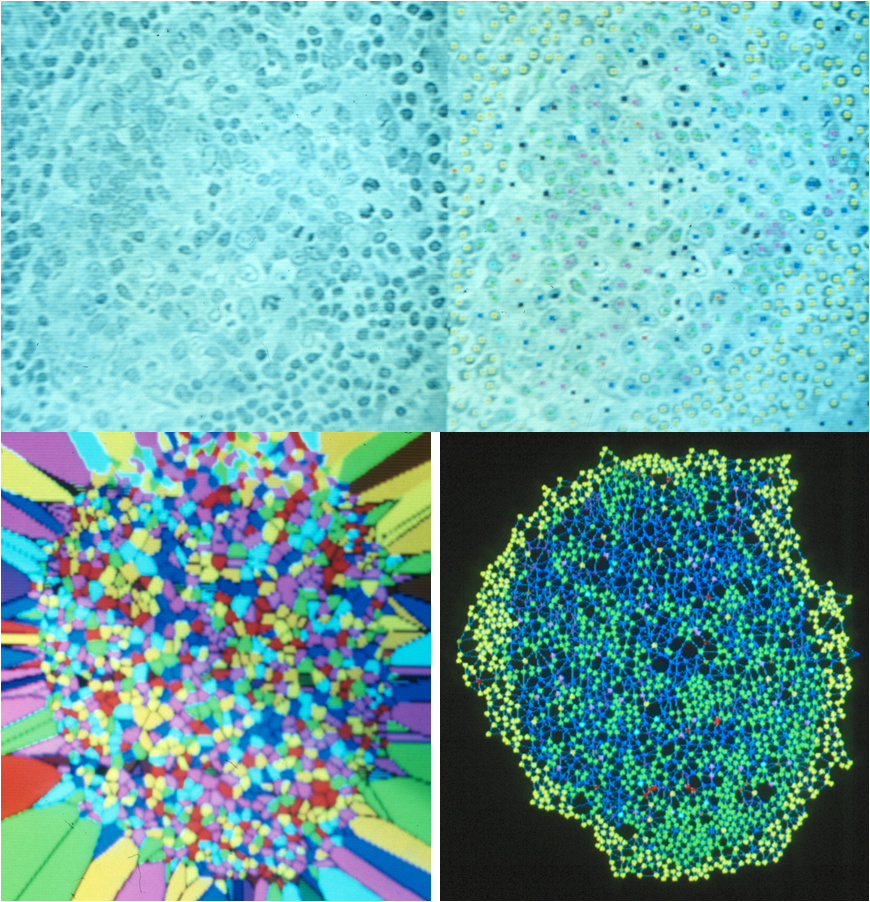}%
\end{center}
\end{figure}
%

\column{5cm}%

\begin{itemize}
\item 1.1\ : histological section.

\item 1.2 : marking the nuclei.

\item 2.1 : catchment basins of the nuclei on the original image.

\item 2.2 : node weighted graph: the node weights represent the nucleus type.
\end{itemize}

%

\end{columns}%
%

\end{frame}%
%

\begin{frame}%
%

\frametitle{Minimum spanning trees and forests in an edge weighted graph}%

A \textit{minimum spanning tree} is a spanning tree for which the sum of the
edges is minimal.

A \textit{spanning forest }is \textit{a collection of trees spanning all
nodes.}

In a \textit{minimum spanning forest, }the sum of the edges is minimal.\ 

The \textit{minimum minimorum spanning forest (MSF) }has only nodes. Wih an
additional constraint, one gets particular forests:

\begin{itemize}
\item a MSF with a fixed number of trees, useful in hierarchical segmentation
\cite{forests}

\item a MSF where each tree is rooted in predefined nodes, useful in marker
based segmentation.
\end{itemize}

%

\end{frame}%
%

\begin{frame}%
%

\frametitle{Flat zones and regional minima on edge weighted graphs}%

A subgraph $G^{\prime}$ of an edge weighted graph $G$ is a flat zone, if any
two nodes of $G^{\prime}$ are connected by a chain of uniform altitude.

A subgraph $G^{\prime}$ of a graph $G$ is a regional minimum if $G^{\prime} $
is a flat zone and all edges in its cocycle have a higher altitude

$\mu_{e}:G\rightarrow\ \mu_{e}G:$ extracts all regional minima of the graph
$G$%

\end{frame}%
%

\begin{frame}%
%

\frametitle{Flat zones and regional minima on node weighted graphs}%

A subgraph $G^{\prime}$ of a node weighted graph $G$ is a flat zone, if any
two nodes of $G^{\prime}$ are connected by a path where all nodes have the
same altitude.

A subgraph $G^{\prime}$ of a graph $G$ is a regional minimum if $G^{\prime} $
is a flat zone and all neighboring nodes have a higher altitude

$\mu_{n}:G\rightarrow\ \mu_{n}G:$ extracts all regional minima of the graph
$G$%

\end{frame}%
%

\begin{frame}%
%

\frametitle{Contracting or expanding graphs}%

\textbf{Contraction:}\textit{ }

For \textit{contracting} an edge $(i,j)$ in a graph $G$, one suppresses this
edge $u$ and its two extremities are merged into a unique node $k.\ $All edges
incident to $i$ or to $j$ become edges incident to the new node $k$ and keep
their weights \bigskip

If $H$ is a subgraph of $G$, The operator $\kappa:(G,H)\rightarrow G^{\prime}$
contracts all edges of $H$ in $G$ \bigskip

\textbf{Expansion: }

The operator $\multimap\ :(-,n)\rightarrow\ \multimap G$ creates for each
isolated regional minimum $i$ a dummy node with the same weight, linked by an
edge with $i.$%

\end{frame}%
%

\begin{frame}%
%

\frametitle{Extracting partial graphs}%

The following operators suppress a subset of the edges in a graph:

$\chi\ :(-,\diamond)\rightarrow\ \chi G$ keeps for each node only one adjacent edges.

$\downarrow\ :(e,\diamond)\rightarrow\ \downarrow G$ keeps for each node only
its lowest adjacent edges.

$\Downarrow\ :(-,n)\rightarrow\ \Downarrow G$ keeps only the edges linking a
node with its lowest adjacent nodes.\bigskip

These operators may be concatenated:

$\chi\downarrow G$ keeps for each node only one lowest adjacent edge.\ 

$\chi\Downarrow G$ keeps for each node only one edge linking it with its
lowest adjacent nodes.%

\end{frame}%
%

\begin{frame}%

\begin{center}
{\Large \alert{Distances on a graph}}

\bigskip

\textbf{Case of edge weighed graphs}
\end{center}

%

\end{frame}%
%

\begin{frame}%
%

\frametitle{Constructing distances on an edge weighted graph.}%

Distances on an edge weighted graph have chains as support :

1) Definition of the weight of a chain, as a measure derived from the edge
weights of the chain elements (example : sum, maximum, etc.)

2) Comparison of two chains by their weight.\ The chain with the smallest
weight is called the shortest.\ 

The distance $d(x,y)$ between two nodes $x$ and $y$ of a graph is $\infty$ if
there is no chain linking these two nodes and equal to the weight of the
shortest chain if such a chain exists.

Given three nodes $(x,y,z)$ the concatenation of the shortest chain $\pi_{xy}$
between $x$ and $y$ and the shortest chain $\pi_{yz}$ between $y$ and $z$ is a
chain $\pi_{xz}$ between $x$ and $z$, whose weight is smaller or equal to the
weight of the shortest chain between $x$ and $z.$ To each distance corresponds
a particular triangular inequality : $d(x,z)\leq weight( $ $\pi_{xy}\rhd$
$\pi_{yz})$ where $\pi_{xy}\rhd$ $\pi_{yz}$ represents the concatenation of
both chains.%

\end{frame}%
%

\begin{frame}%
%

\frametitle{Distance on an edge weighted graph based a the length of the
shortest chain}%

\textbf{Length of a chain:} The length of a chain between two nodes $x$ and
$y$ is defined as the sum of the weights of its edges.

\textbf{Distance:} The distance $d(x,y)$ between two nodes $x$ and $y$ is the
minimal length of all chains between $x$ and $y$. If there is no chain between
them, the distance is equal to $\infty$.

\textbf{Triangular inequality :} For $(x,y,z):d(x,z)\leq d(x,y)+d(y,z)$%

\end{frame}%
%

\begin{frame}%
%

\frametitle{Distance on a graph based on the maximal edge weight along the
chain}%

The weights are assigned to the edges, and represent their altitudes.

\textbf{Altitude of a chain: }The altitude of a chain is equal to the highest
weight of the edges along the chain.

\textbf{Flooding distance between two nodes: }The flooding distance
$\operatorname*{fldist}(x,y)$ between nodes $x$ and $y$ is equal to the
minimal altitude of all chains between $x$ and $y$. During a flooding process,
in which a source is placed at location $x,$ the flood would proceed along
this chain of minimal highest altitude to reach the pixel $y$. If there is no
chain between them, the level distance is equal to $\infty$.

\textbf{Triangular inequality :} For $(x,y,z):d(x,z)\leq d(x,y)\vee d(y,z)$ :
ultrametric inequality%

\end{frame}%
%

\begin{frame}%
%

\frametitle{The flooding distance is an ultrametric distance}%

An ultrametric distance verifies

* reflexivity : $d(x,x)=0$

* symmetry: $d(x,y)=d(y,x)$

* ultrametric inequality: for all $x,y,z:d(x,y)\leq max\{d(x,z),d(z,y)$\} :
the lowest lake containing both $x$ and $y$ is lower or equal than the lowest
lake containing $x,$ $y$ and $z.$%

\end{frame}%
%

\begin{frame}%
%

\frametitle{Distances on a graph : sum and maximum of the edge weights}%
\begin{figure}
[ptb]
\begin{center}
\includegraphics[
height=1.0491in,
width=1.4885in
]%
{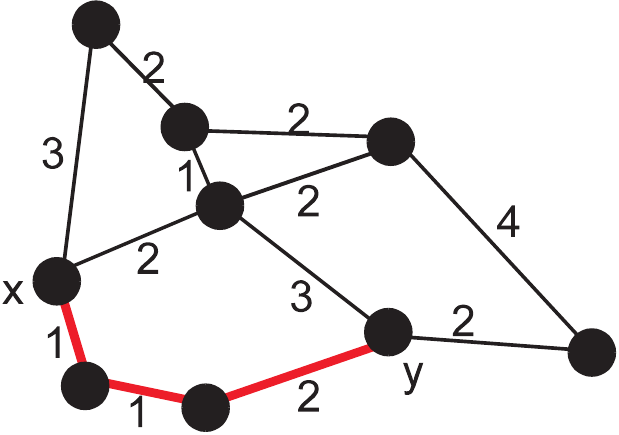}%
\end{center}
\end{figure}

The shortest chain (sum of weights of the edges) between $x$ and $y$ is a red
line and has a length of $4$.

The lowest chain (maximal weight of the edges) between $x$ and $y$ is a red
line and a maximal weight of $2$. A flooding between $x$ and $y$ would follow
this chain.

For this particular example, the shortest and lowest paths are identical.%

\end{frame}%
%

\begin{frame}%
%

\frametitle{The lexicographic distance or the cumulative effort for passing
the highest edges}%

\textbf{Toughness of a chain:} We call toughness of a chain the decreasing
list of altitudes of the highest edges met along this chain.%

\begin{figure}
[ptb]
\begin{center}
\includegraphics[
height=1.195in,
width=3.1406in
]%
{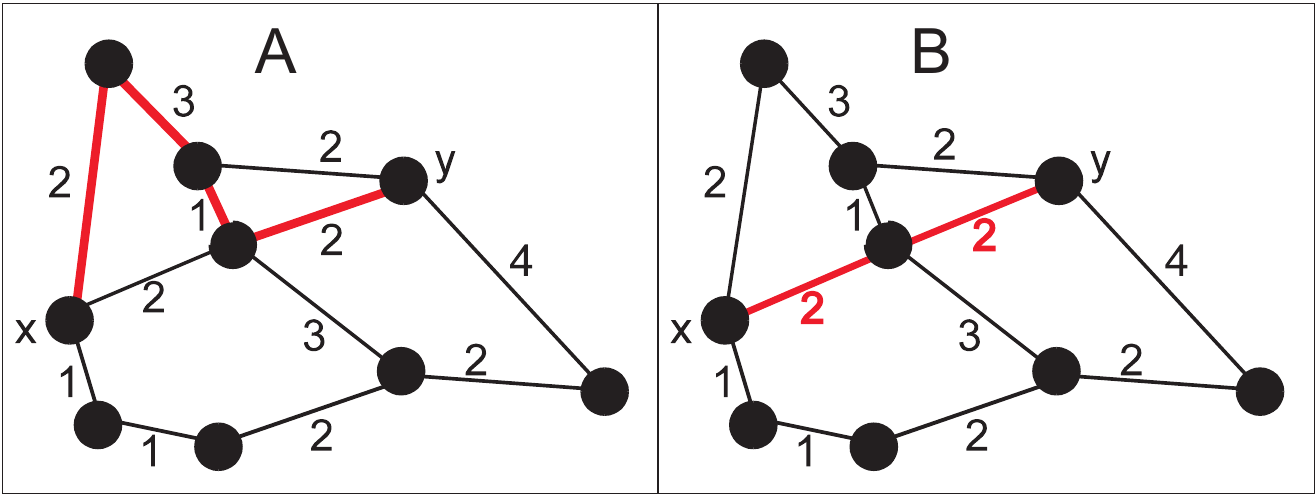}%
\end{center}
\end{figure}

In fig.A, along the bold chain between $x$ and $y$, the highest edge rises at
the altitude $3$. After crossing it, the highest edge on the remaining chain
rises at $2$. After crossing it, one is at destination.\ Hence the toughness
of the chain from $x$ to $y$ is $[3,2].\;$Distances are compared in a
lexicographic order. The shortest chain is the red chain in fig.B, it has two
edges with weight $2$ and its toughness is $[2,2].$%

\end{frame}%
%

\begin{frame}%
%

\frametitle{The lexicographic distance : a more formal definition}%

The lexicographic length $\Lambda(A)$ of a chain $A=e_{12}e_{23}...e_{n-1n}$
is constructed as follows.\ Following the chain from the origin towards its
end, one records the highest valuation of the chain, let it be $\lambda_{1}$,
then again the highest valuation $\lambda_{2}$ on the remaining part of the
chain and so on until the end is reached.\ One gets like that a series of
decreasing values: $\lambda_{1}\geq\lambda_{2}\geq$....$\geq\lambda_{n}$.\ 

The lexicographic distance between two nodes $x$ and $y$ will be equal to the
shortest lexicographic length of all chains between these two nodes and is
written $\operatorname*{lexdist}(x,y)$.\ \bigskip

\textbf{Remark: }The lexicographic length of a never increasing path (NAP), is
simply equal to the series of weights of its edges, as it is the highest edge
along the lowest path between $x$ and $y.$ Later we introduce shortest path
algorithms whose geodesics are NAPs.%

\end{frame}%
%

\begin{frame}%
%

\frametitle{The lexicographic distance of depth k}%

The lexicographic distance of depth $k$ between two nodes $x$ and $y$ is
written $\operatorname*{lexdist}_{k}(x,y)$ and obtained by retaining only the
$k$ first edges.in $\operatorname*{lexdist}(x,y).\ $

$\bigskip$

\textbf{Remark: }$\operatorname*{lexdist}_{1}(x,y)$ is the same as the
ultrametric flooding distance as it is the highest edge on the lowest path
between $x$ and $y.$\ %

\end{frame}%
%

\begin{frame}%

\begin{center}
{\Large \alert{Distances on a graph}}

\bigskip

\textbf{Case of node weighed graphs}
\end{center}

%

\end{frame}%
%

\begin{frame}%
%

\frametitle{Constructing distances on a node weighted graph.}%

Distances on a node weighted graph have paths as support :

1) Definition of the "length" of a path, as a measure derived from the node
weights of the path elements (example : sum, maximum, etc.)

2) Comparison of two paths by their length.\ The path with the smallest length
is called the shortest.\ 

The distance $d(x,y)$ between two nodes $x$ and $y$ of a graph is $\infty$ if
there is no path linking these two nodes and equal to the length of the
shortest path if such a path exists.

Given three nodes $(x,y,z)$ the concatenation of the shortest path $\pi_{xy} $
between $x$ and $y$ and the shortest path $\pi_{yz}$ between $y$ and $z$ is a
path $\pi_{xz}$ between $x$ and $z$, whose length is smaller or equal to the
length of the shortest path between $x$ and $z.$ To each distance corresponds
a particular triangular inequality : $d(x,z)\leq length( $ $\pi_{xy}\rhd$
$\pi_{yz})$ where $\pi_{xy}\rhd$ $\pi_{yz}$ represents the concatenation of
both paths.%

\end{frame}%
%

\begin{frame}%
%

\frametitle{Distance on a node weighted graph based on the maximal node weight
along the path}%

The weights are assigned to the nodes, and represent their altitudes.

\textbf{Altitude of a path: }The altitude of a path is equal to the highest
weight of the nodes along the path.

\textbf{Flooding distance between two nodes: }The flooding distance
$\operatorname*{fldist}(x,y)$ between nodes $x$ and $y$ is equal to the
minimal altitude of all paths between $x$ and $y$.%

\end{frame}%
%

\begin{frame}%
%

\frametitle{Distance on a node weighted graph based on the cost for travelling
along the cheapest path}%

The weights are assigned to the nodes and not to the edges. Each node may be
considered as a town where a toll has to be paid.

\textbf{Cost of a path:} The cost of a path is equal to the sum of the tolls
to be paid in all towns encountered along the path (including or not one or
both ends).

\textbf{Cost between two nodes:} The cost for reaching node $y$ from node $x$
is equal to the minimal cost of all paths between $x$ and $y$. We write
$\operatorname*{tolldist}(x,y).\ $If there is no path between them, the cost
is equal to $\infty$.%

\end{frame}%
%

\begin{frame}%
%

\frametitle{Illustration of the cheapest path}%

\textbf{%
\begin{figure}
[ptb]
\begin{center}
\includegraphics[
height=1.4173in,
width=1.8952in
]%
{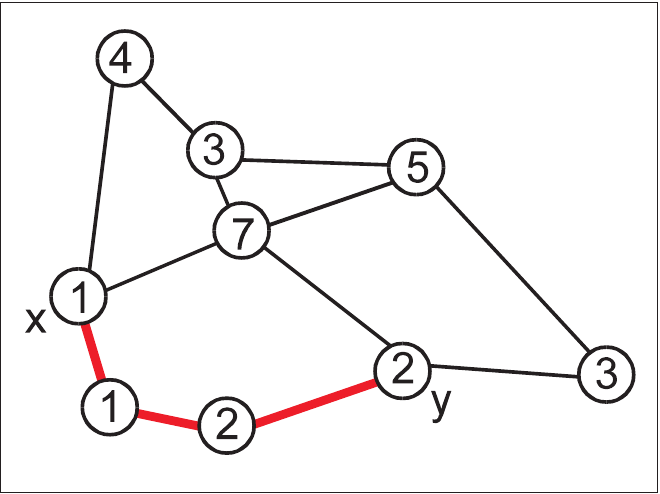}%
\end{center}
\end{figure}
}

In this figure, the cheapest chain between $x$ and $y$ is in red and the total
toll to pay is $1+1+2+2=6$%

\end{frame}%
%

\begin{frame}%

\begin{center}
{\Large \alert{Two adjunctions between edges and nodes on weighted graphs}}
\end{center}

%

\end{frame}%
%

\begin{frame}%
%

\frametitle{Two adjunctions between edges and nodes}%

\begin{definition}
\medskip We define two operators between edges and nodes :\newline\medskip- an
erosion $\left[  \varepsilon_{en}n\right]  _{ij}=n_{i}\wedge n_{j}$ and its
adjunct dilation $\left[  \delta_{ne}e\right]  _{i}=\underset
{(k\text{\thinspace neighbors\thinspace of\thinspace}\,i)}{\bigvee e_{ik}}%
$\newline- a dilation $\left[  \delta_{en}n\right]  _{ij}=n_{i}\vee n_{j}$ and
its adjunct erosion $\left[  \varepsilon_{ne}e\right]  _{i}=\underset
{(k\text{\thinspace neighbors\thinspace of\thinspace}\,i)}{\bigwedge e_{ik}}$
\end{definition}

\begin{lemma}
The operators we defined are pairwise adjunct or dual operators:\newline-
$\varepsilon_{ne}$ and $\delta_{en}$ are adjunct operators\newline-
$\varepsilon_{en}$ and $\delta_{ne}$ are adjunct operators\newline-
$\varepsilon_{ne}$ and $\delta_{ne}$ are dual operators \newline-
$\varepsilon_{en}$ and $\delta_{en}$ are dual operators
\end{lemma}

%

\end{frame}%
%

\begin{frame}%

\textbf{Let us prove that }$\delta_{en}$\textbf{\ and }$\varepsilon_{ne}%
$\textbf{\ are adjunct operators}

If $G=\left[  e,n\right]  $ and $\overline{G}=\left[  \overline{e}%
,\overline{n}\right]  $ are two graphs with the same nodes and edges, but with
different valuations on the edges and the nodes then\newline$\delta_{en}%
n\leq\overline{e}\Leftrightarrow\forall i,j:n_{i}\vee n_{j}\leq\overline
{e}_{ij}\Leftrightarrow\forall i,j:n_{i}\leq\overline{e}_{ij}\Leftrightarrow
\forall i,j:n_{i}\leq\underset{(j\text{\thinspace neighbors\thinspace
of\thinspace}\,i)}{\bigwedge\overline{e}_{ij}}=\left[  \varepsilon
_{ne}\overline{e}\right]  _{i}\Leftrightarrow n\leq\varepsilon_{ne}%
\overline{e}$\newline which establishes that $\delta_{en}$ and $\varepsilon
_{en}$ are adjunct operators.

\bigskip

\textbf{Let us prove that }$\varepsilon_{en}$\textbf{\ and }$\delta_{en}%
$\textbf{\ are dual operators\newline}$\left[  \varepsilon_{en}(-n)\right]
_{ij}=-n_{i}\wedge-n_{j}=-(n_{i}\vee n_{j})=-\left[  \delta_{en}n\right]
_{ij}$\newline Hence $\delta_{en}n=-\left[  \varepsilon_{en}\left(  -n\right)
\right]  $%

\end{frame}%
%

\begin{frame}%
%

\frametitle{The flooding adjunction}%

The dilation $\left[  \delta_{en}n\right]  _{ij}=n_{i}\vee n_{j}$ and its
adjunct erosion $\left[  \varepsilon_{ne}e\right]  _{i}=\underset
{(k\text{\thinspace neighbors\thinspace of\thinspace}\,i)}{\bigwedge e_{ik}}$
have a particular meaning in terms of flooding.

\bigskip

If $n_{i}$ and $n_{j}$ represent the altitudes of the nodes $i$ and $j,$ the
lowest flood covering $i$ and $j$ has the altitude $\left[  \delta
_{en}n\right]  _{ij}=n_{i}\vee n_{j}$

\bigskip

If $i$ represents a catchment basin, $e_{ik}$ the altitude of the pass points
with the neighboring basin $k,$ then the highest level of flooding without
overflow through an adjacent edge is $\left[  \varepsilon_{ne}e\right]
_{i}=\underset{(k\text{\thinspace neighbors\thinspace of\thinspace}%
\,i)}{\bigwedge e_{ik}}.$%

\end{frame}%
%

\begin{frame}%
%

\frametitle{Waterfall flooding on graphs : an erosion}%

The waterfall flooding is completely specified if one knows the level of flood
in each catchment basin%

\begin{columns}[5cm]%
\column{5cm}%
\begin{center}
\includegraphics[
height=1.5355in,
width=2.4856in
]%
{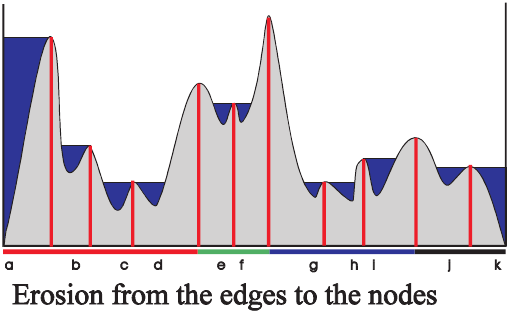}%
\end{center}
%

\column{5cm}%

\begin{itemize}
\item The waterfall flooding fills each catchment basin up to its lowest pass point.

\item In terms of graphs : the flooding in a node is equal to the weight of
its lowest adjacent edge.

\item It is an erosion between node weights and edge weights : $\left[
\varepsilon_{ne}e\right]  _{i}=\underset{(k\text{\thinspace
neighbors\thinspace of\thinspace}\,i)}{\bigwedge e_{ik}}$
\end{itemize}

%

\end{columns}%
%

\end{frame}%
%

\begin{frame}%
%

\frametitle{Waterfall flooding on graphs : an erosion}%

The waterfall flooding is completely specified if one knows the level of flood
in each catchment basin%

\begin{columns}[5cm]%
\column{5cm}%
\begin{figure}
[ptb]
\begin{center}
\includegraphics[
height=1.9095in,
width=1.9741in
]%
{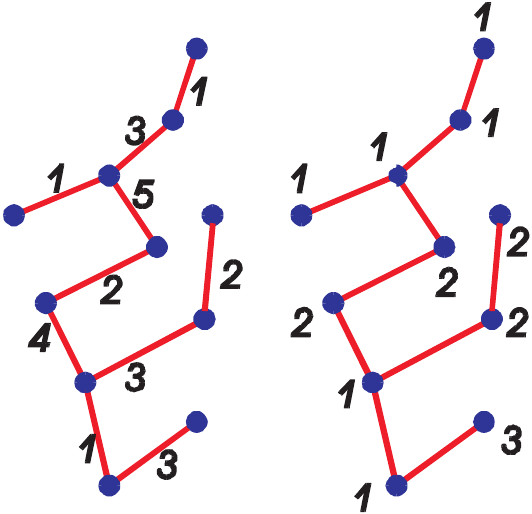}%
\end{center}
\end{figure}
%

\column{5cm}%

\begin{itemize}
\item The waterfall flooding fills each catchment basin up to its lowest pass point

\item In terms of graphs : the flooding in a node is equal to the weight of
its lowest adjacent edge

\item It is an erosion between node weights and edge weights : $\left[
\varepsilon_{ne}e\right]  _{i}=\underset{(k\text{\thinspace
neighbors\thinspace of\thinspace}\,i)}{\bigwedge e_{ik}}$
\end{itemize}

%

\end{columns}%
%

\end{frame}%
%

\begin{frame}%
%

\frametitle{Erosion between edges and nodes}%
%

\begin{figure}
[ptb]
\begin{center}
\includegraphics[
height=2.3909in,
width=2.4605in
]%
{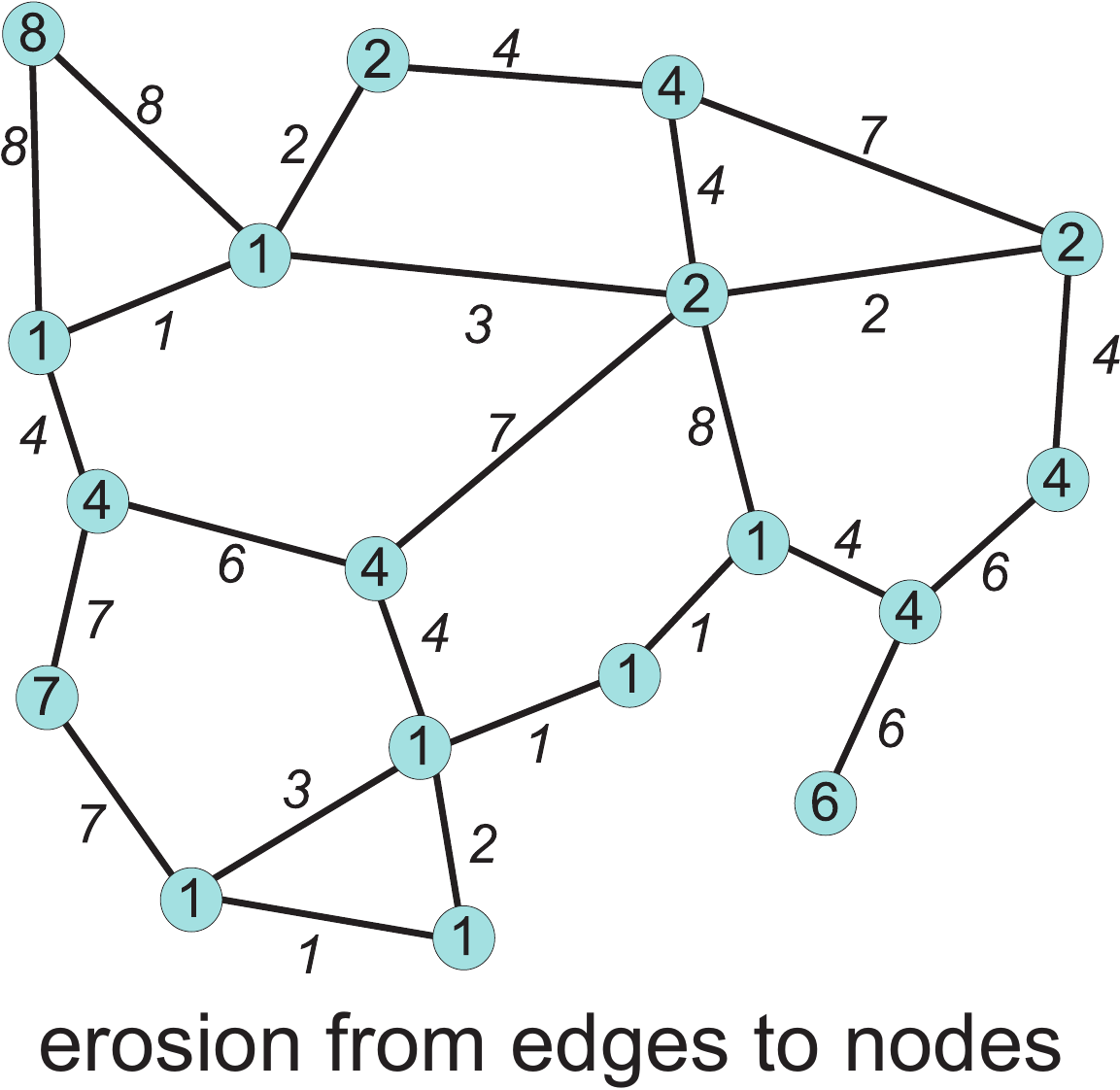}%
\end{center}
\end{figure}
%

\end{frame}%
%

\begin{frame}%
%

\frametitle{Dilation between nodes and edges}%

If $n_{i}$ and $n_{j}$ represent the altitudes of the nodes $i$ and $j,$ the
lowest flood covering $i$ and $j$ has the altitude $\left[  \delta
_{en}n\right]  _{ij}=n_{i}\vee n_{j}$%

\begin{figure}
[ptb]
\begin{center}
\includegraphics[
height=2.4303in,
width=2.4605in
]%
{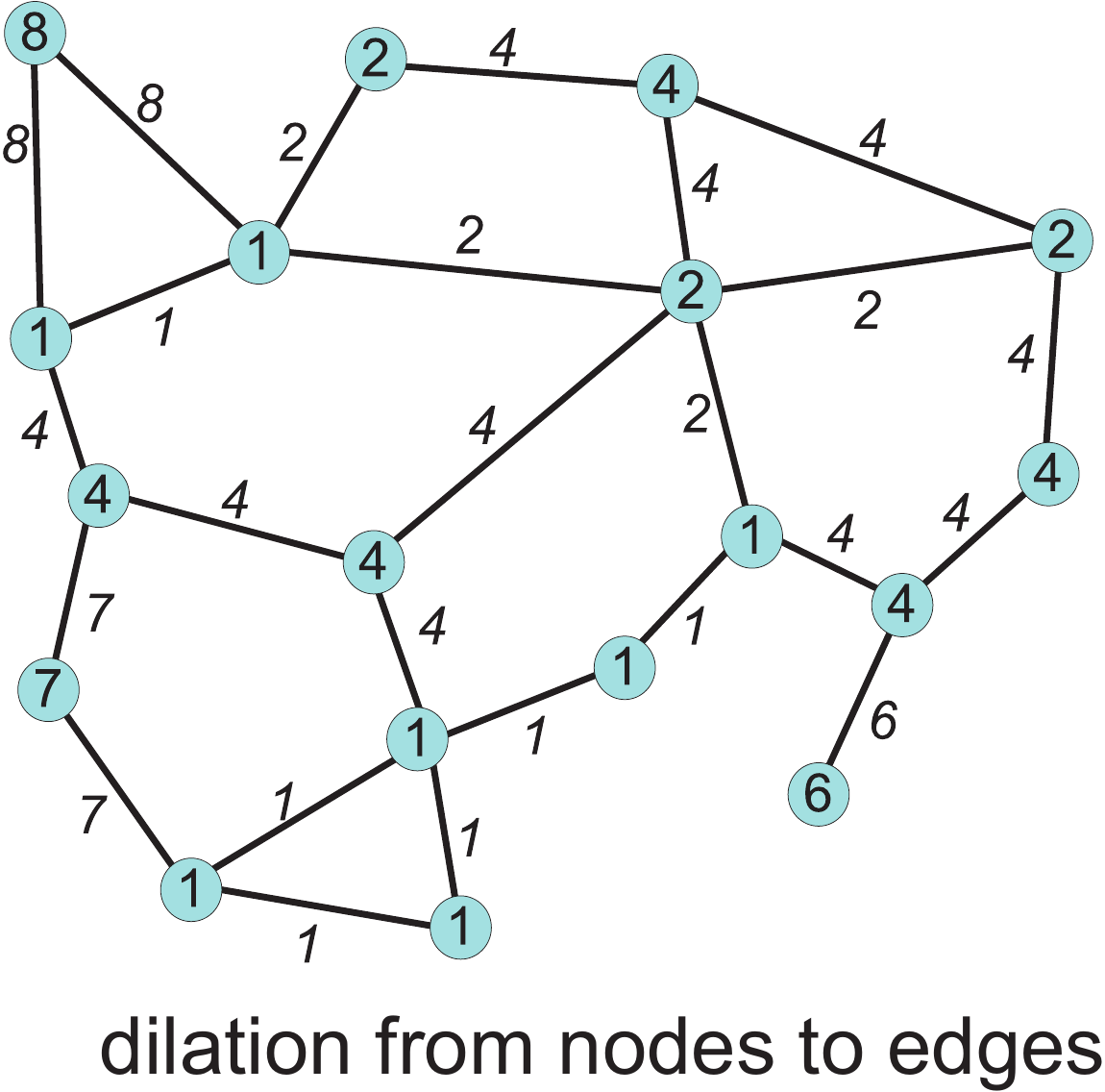}%
\end{center}
\end{figure}
%

\end{frame}%
%

\begin{frame}%
%

\frametitle{Erosion between edges and edges /\ between nodes and nodes}%

By concatenation of operators between edges and nodes we obtain :

\begin{itemize}
\item an adjunction between nodes and nodes : $(\varepsilon_{ne}%
\varepsilon_{en},\delta_{ne}\delta_{en})=(\varepsilon_{n},\delta_{n})$

\item an adjunction between edges and edges : $(\varepsilon_{en}%
\varepsilon_{ne},\delta_{en}\delta_{ne})=(\varepsilon_{e},\delta_{e})$
\end{itemize}

%

\end{frame}%
%

\begin{frame}%

\begin{center}
{\Large \alert{Opening and closing}}
\end{center}

%

\end{frame}%
%

\begin{frame}%
%

\frametitle{Opening and closing}%

As $\varepsilon_{ne}$ and $\delta_{en}$ are adjunct operators, the operator
$\varphi_{n}=\varepsilon_{ne}\delta_{en}$ is a closing on $n$ and $\gamma
_{e}=\delta_{en}\varepsilon_{ne}$ is an opening on $e$\ 

\bigskip

Similarly the operator $\varphi_{e}=\varepsilon_{en}\delta_{ne}$ is a closing
on $e$ and $\gamma_{n}=\delta_{ne}\varepsilon_{en}$ is an opening on $n$\ 

\bigskip

In the sequel, the flooding adjunction will play a key role, in particular the
associated opening $\gamma_{e}$ and closing $\varphi_{n}.$%

\end{frame}%
%

\begin{frame}%

\begin{center}
{\Large \alert{The opening $\gamma _{e}$}}
\end{center}

%

\end{frame}%
%

\begin{frame}%
%

\frametitle{Illustration in one dimension of the opening $\gamma_{e}$}%
\begin{figure}
[ptb]
\begin{center}
\includegraphics[
height=1.2579in,
width=4.3473in
]%
{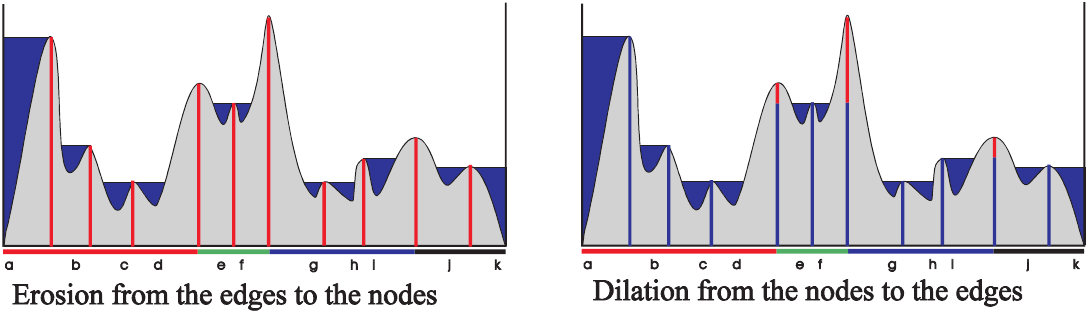}%
\end{center}
\end{figure}

$\gamma_{e}=\delta_{en}\varepsilon_{ne}$ : on the left, the result of the
erosion, filling each basin to its lowest pass point. On the right the
subsequent dilation. Some pass points have a reduced altitude : the amount of
reduction is indicated in red. These passpoints are those which are not the
lowest pass points of a catchment basin.%

\end{frame}%
%

\begin{frame}%
%

\frametitle{Illustration on an edge weighted tree of the opening $\gamma_{e}$}%
%

\begin{center}
\includegraphics[
height=2.0672in,
width=3.2714in
]%
{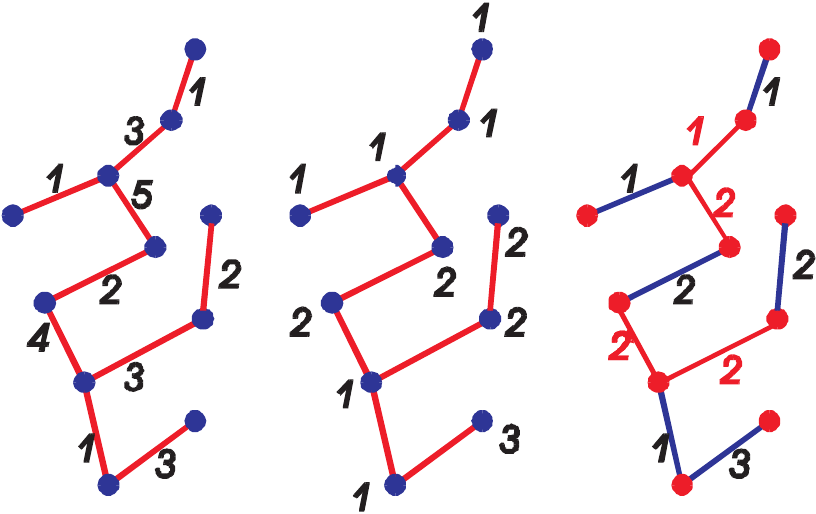}%
\end{center}

$\gamma_{e}=\delta_{en}\varepsilon_{ne}$ : From left to right: 1) an edge
weighted graph, in the centre, 2) the result of the erosion $\varepsilon_{ne}%
$, 3) the subsequent dilation produces an opening. The edges in red are those
whose weight has been reduced by the opening.\ These edges are not the lowest
edges of one or their extremities.%

\end{frame}%
%

\begin{frame}%
%

\frametitle{The invariants of the opening $\gamma_{e}$}%

Two possibilities exist for an edge $(i,j)$ with a weight $\lambda:$

\begin{itemize}
\item the edge $(i,j)$ has lower neighboring edges at each extremity. Hence
$\varepsilon_{ne}(i)<\lambda$ and $\varepsilon_{ne}(j)<\lambda$ ; hence
$\delta_{en}\varepsilon_{ne}(i,j)=\varepsilon_{en}(i)\vee\varepsilon
_{en}(j)<\lambda$

\item the edge $(i,j)$ is the lowest edge of the extremity $i.$ Then
$\varepsilon_{ne}(i)=\lambda$ and $\varepsilon_{ne}(j)\leq\lambda$ ; hence
$\delta_{en}\varepsilon_{ne}(i,j)=\varepsilon_{en}(i)\vee\varepsilon
_{en}(j)=\lambda$
\end{itemize}

\bigskip

\textbf{Conclusion: }the edges invariant by the opening\textbf{\ }$\gamma_{e}$
are the edges which are the lowest edge of one of their extremities. All edges
with lower adjacent edges at their extremities have their weight lowered by
the opening $\gamma_{e}$%

\end{frame}%
%

\begin{frame}%
%

\frametitle{Extracting from an edge weighted graph a partial graph invariant
by the opening $\gamma_{e}$}%

The relation $\varepsilon_{ne}=\varepsilon_{ne}(\delta_{en}\varepsilon_{ne})$
shows that all edges which are not invariant by the opening $\gamma_{e}%
=\delta_{en}\varepsilon_{ne}$ play no role in the erosion. As a matter of
fact, if the opening lowers the valuation of an edge $(i,j)$ and if the
subsequent erosion $\varepsilon_{ne}$ is not modified, it means that the
presence of this edge with its weight plays no role in the erosion
$\varepsilon_{ne}$.\ %

\end{frame}%
%

\begin{frame}%
%

\frametitle{The waterfall graph, partial graph of the lowest adjacent edges}%

Suppressing in an arbitrary graph $g$ all edges which are not invariant by the
opening $\gamma_{e}$ produces a graph $g^{\prime},$ invariant for the opening
$\gamma_{e}$ and called \textbf{waterfall graph}.\ The operator keeping for
each node only its lowest adjacent edges is written $\downarrow\ :(e,\diamond
)\rightarrow\ \downarrow G$

\textbf{Properties:}

\begin{itemize}
\item $\downarrow G$ spans all the nodes (each node has at least one lowest
neighboring edge ; such edges are invariant by $\gamma_{e}$)

\item And $\varepsilon_{ne}(G)=\varepsilon_{ne}(\downarrow G)$ ($\varepsilon
_{ne}$ assigns to each node the weight of its lowest adjacent edge, which is
the same in $G$ as in $\downarrow G)$
\end{itemize}

%

\end{frame}%
%

\begin{frame}%
%

\frametitle{Illustration on a planar edge weighted graph of the opening
$\gamma_{e}$.\ }%
%

\begin{figure}
[ptb]
\begin{center}
\includegraphics[
height=1.5397in,
width=4.6987in
]%
{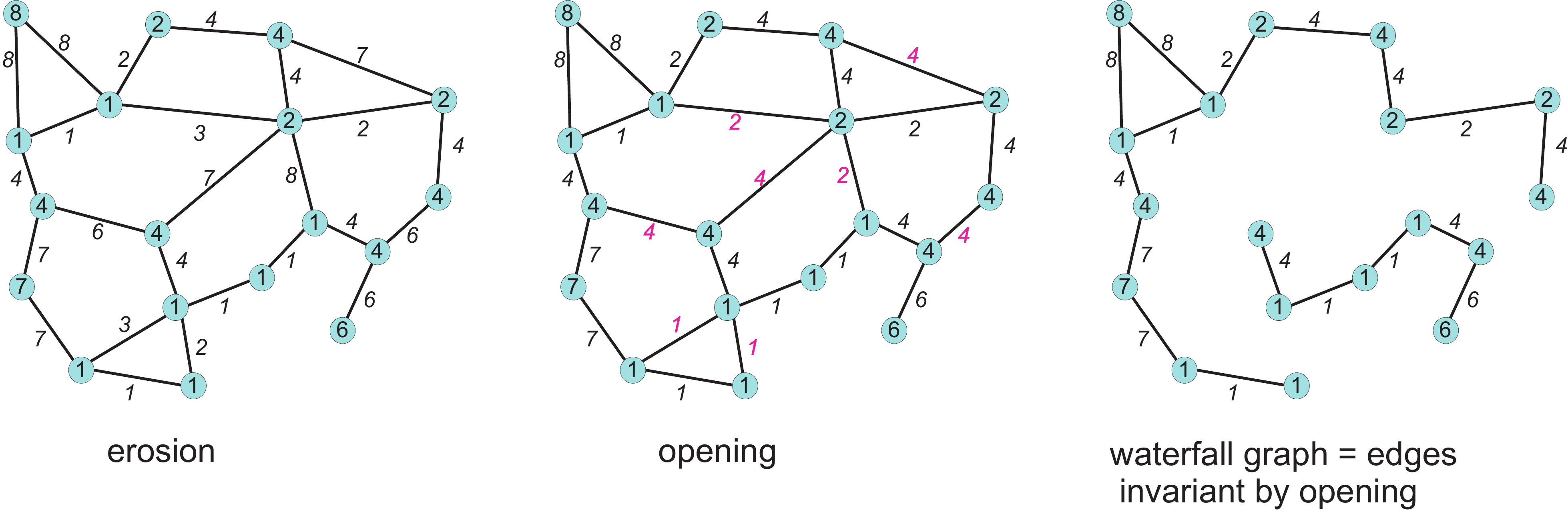}%
\end{center}
\end{figure}
%

\end{frame}%
%

\begin{frame}%
%

\frametitle{Erosion from edges to nodes on the graph $\downarrow G$}%
%

\begin{figure}
[ptb]
\begin{center}
\includegraphics[
height=2.0705in,
width=4.0907in
]%
{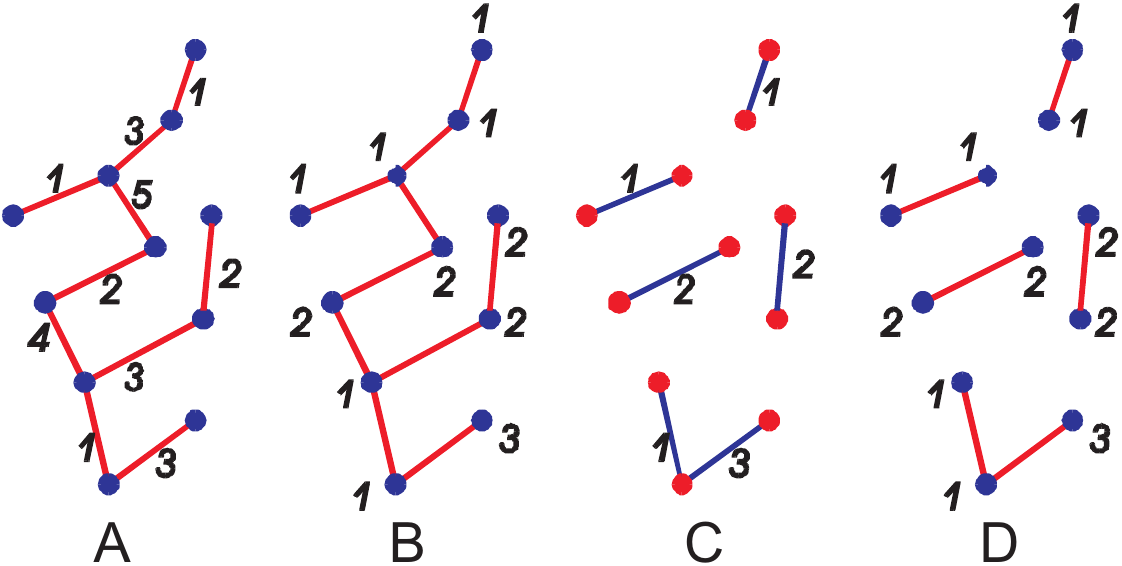}%
\end{center}
\end{figure}

A: Initial graph $G$

B: Erosion $\varepsilon_{ne}$ from the edges to the nodes on $G$

C: The graph $\downarrow G$

D: Erosion $\varepsilon_{ne}$ from the edges to the nodes on $\downarrow G$ :
$\varepsilon_{ne}(G)=\varepsilon_{ne}(\downarrow G)$%

\end{frame}%
%

\begin{frame}%
%

\frametitle{Erosion from edges to nodes on the graph $\downarrow G$}%

Another illustration of $\varepsilon_{ne}(G)=\varepsilon_{ne}(\downarrow G) $%

\begin{figure}
[ptb]
\begin{center}
\includegraphics[
height=2.2173in,
width=4.7331in
]%
{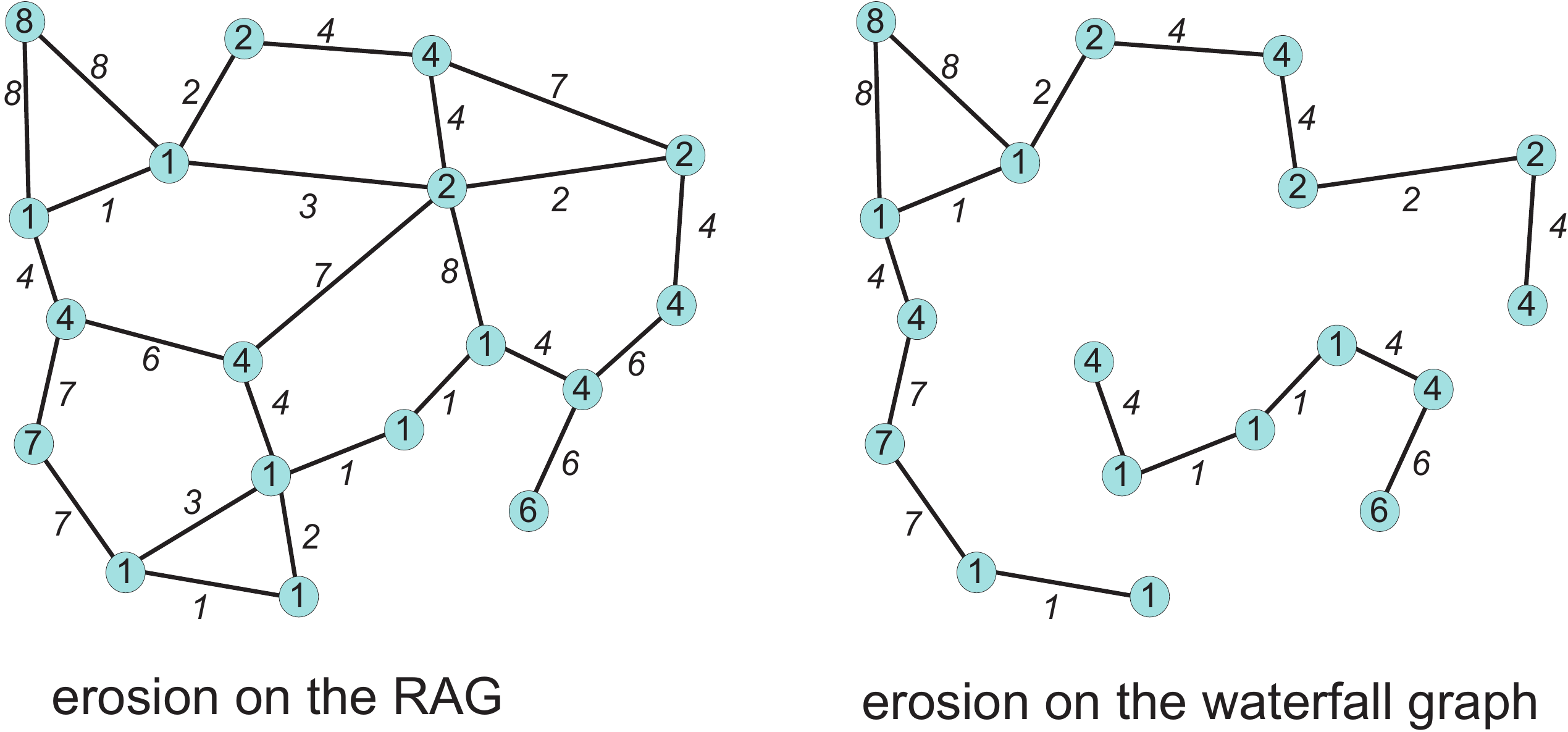}%
\end{center}
\end{figure}
%

\end{frame}%
%

\begin{frame}%

\begin{center}
{\Large \alert{The invariants of the opening $\gamma _{e}$}}
\end{center}

%

\end{frame}%
%

\begin{frame}%
%

\frametitle{The invariants of the opening $\gamma_{e}$}%

If $(e_{1},\diamond)$ and $(e_{2},\diamond)$ are two edge weight distributions
which are invariant for $\gamma_{e},$ then $(e_{1}\vee e_{2},\diamond)$ also
is invariant for $\gamma_{e}.$

\bigskip

An arbitrary graph may be transformed into a flooding graph:

* For an arbitrary edge weight distribution $(e,\diamond):(\gamma
_{e}e,\diamond)\in\operatorname*{Inv}(\gamma_{e})$

* For an arbitrary node weight distribution $(-,n):(\delta_{en}n,n)\in
\operatorname*{Inv}(\gamma_{e})$ as $\gamma_{e}\delta_{en}n=\delta
_{en}\varepsilon_{ne}\delta_{en}n=\delta_{en}n$%

\end{frame}%
%

\begin{frame}%
%

\frametitle{The invariants of the opening $\gamma_{e}$}%

For a connected graph $g=(e,\diamond)\in\operatorname*{Inv}(\gamma_{e})$:

\begin{itemize}
\item any partial spanning graph belongs to $\operatorname*{Inv}(\gamma_{e})$

\item any subgraph belongs to $\operatorname*{Inv}(\gamma_{e})$
\end{itemize}

\bigskip

In particular $\downarrow\ :G=(e,\diamond)\rightarrow\ \downarrow G$
containing for each node only its lowest adjacent edges

and $\chi\downarrow G$ keeping only one lowest adjacent edge for each node,

both belong to $\operatorname*{Inv}(\gamma_{e})$%

\end{frame}%
%

\begin{frame}%
%

\frametitle{The regional minima of the opening $\gamma_{e}$}%

\begin{theorem}
If $G=(e,\diamond)\in\operatorname*{Inv}(\gamma_{e})$ and $m=(e,\diamond)$ is
the subgraph of its regional minima, then $\varepsilon_{ne}m=(-,\varepsilon
_{ne}e)$ is the subgraph of the regional minima of the graph $\varepsilon
_{ne}G=(-,\varepsilon_{ne}e)$
\end{theorem}

\textbf{Proof: }A regional minimum $m_{k}$ of the graph $G=(e,\diamond
)\in\operatorname*{Inv}(\gamma_{e})$ is a plateau of edges with altitude
$\lambda$, with all external edges having a weight $>\lambda.\ $If a node $i$
belongs to this regional minimum, its adjacent edges have a weight
$\geq\lambda$ but it has at least one neighboring edge with weight $\lambda$ :
hence $\varepsilon_{ne}e(i)=\lambda.$ Consider now an edge $(s,t)$ outside the
regional minimum, with the node $s$ inside and the node $t$ outside the
minimum.\ Then $e_{st}>\lambda$.\ As $G=(e,\diamond)\in\operatorname*{Inv}%
(\gamma_{e}),$ the edge $(s,t)$ is then one of the lowest edges of the nodes
$t$ : thus $\varepsilon_{ne}(t)=e_{st}>\lambda.\ $This shows that the nodes
spanned by the regional minimum $m_{k}$ form a regional minimum of the graph
$\varepsilon_{ne}G=(-,\varepsilon_{ne}e)$%

\end{frame}%
%

\begin{frame}%
%

\frametitle{Inverse of $\varepsilon_{ne}$ on the invariants of the opening
$\gamma_{e}$}%

On $\operatorname*{Inv}(\gamma_{e}):\delta_{en}\varepsilon_{ne}=Identity$
showing that on $\operatorname*{Inv}(\gamma_{e}):=\delta_{en}=\varepsilon
_{ne}^{-1}$%

\end{frame}%

\begin{frame}%

\begin{center}
{\Large \alert{The closing $\phi _{n}$}}
\end{center}

%

\end{frame}%
%

\begin{frame}%
%

\frametitle{The closing $\varphi_{n}$}%

The closing $\varphi_{n}$ is obtained by a dilation $\delta_{en}$ of the node
weights followed by an erosion $\varepsilon_{ne}.\ $One remarks on the
following figure that the node weights remain the same, except the isolated
regional minima, which take the weight of their lowest neighboring node.\ We
give the proof in the next slide.%

\begin{figure}
[ptb]
\begin{center}
\includegraphics[
height=1.8349in,
width=2.489in
]%
{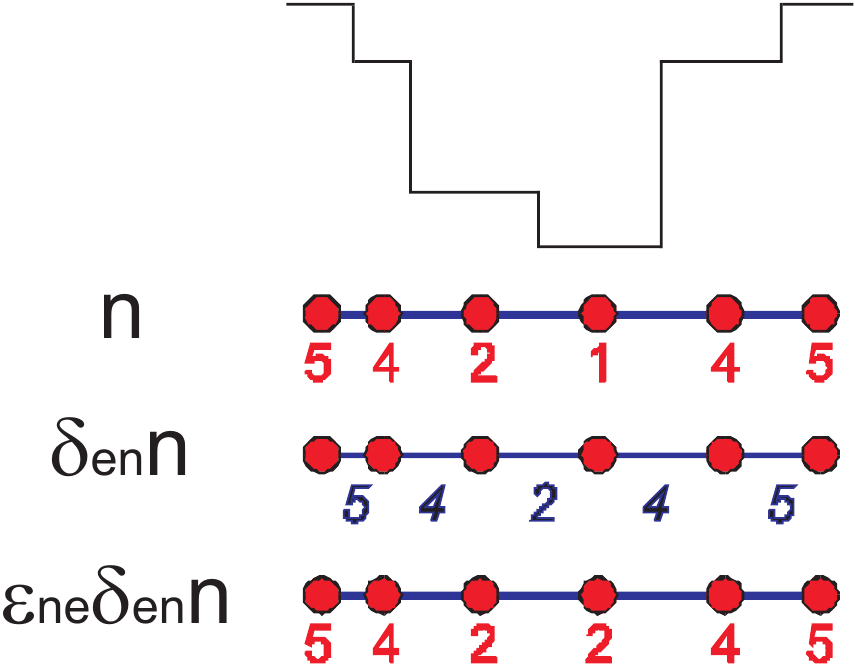}%
\end{center}
\end{figure}
%

\end{frame}%
%

\begin{frame}%
%

\frametitle{The invariants of the closing $\varphi_{n}$}%

Two possibilities exist for a node $i$ with a weight $\lambda:$

\begin{itemize}
\item the node $i$ is an isolated regional minimum.\ Then $\delta_{en}$
assigns to all edges adjacent to $i$ a weight bigger than $\lambda.\ $The
subsequent erosion $\varepsilon_{ne}$ assigns to $i$ the smallest of these weights.\ 

\item the node $i$ has a neighbor $j$ with a weight $\mu\leq\lambda.$ Then
$\delta_{en}$ assigns to the edge $(i,j)$ the weight $\lambda;$ whatever the
weight of the other adjacent edges$.\ $The subsequent erosion $\varepsilon
_{ne}$ assigns to $i$ the smallest of these weights, that is $\lambda.$
\end{itemize}

\bigskip

The closing $\varphi_{n}$ replaces each isolated node constituting a regional
minimum by its lowest neighboring node and leaves all other nodes unchanged.%

\end{frame}%
%

\begin{frame}%

\begin{center}
{\Large \alert{The invariants of the closing $\phi _{n}$}}
\end{center}

%

\end{frame}%
%

\begin{frame}%
%

\frametitle{The invariants of the closing $\varphi_{n}$}%

A\ node weighted graph is invariant for the closing $\varphi_{n}$ iff it does
not contain isolated regional minima.

\bigskip

For an arbitrary node weight distribution $g=(-,n):$ $\multimap g$ creates for
each isolated regional minimum $i$ a dummy node with the same weight linked by
an edge with $i.$ Hence $\multimap g\in\operatorname*{Inv}(\varphi_{n}) $%

\end{frame}%
%

\begin{frame}%
%

\frametitle{The invariants of the closing $\varphi_{n}$}%

If $(-,n_{1})$ and $(-,n_{1})$ are two node weight distributions which are
invariant for $\varphi_{n},$ then $(-,n_{1}\wedge n_{2})$ also is invariant
for $\varphi_{n}.$

For an arbitrary node weight distribution $(-,n):(-,\varphi_{n}n)\in
\operatorname*{Inv}(\varphi_{n})$

For an arbitrary edge weight distribution $(e,\diamond):(e,\varepsilon
_{ne}e)\in\operatorname*{Inv}(\varphi_{n})$ as $\varphi_{n}\varepsilon
_{ne}e=\varepsilon_{ne}\delta_{en}\varepsilon_{ne}e=\varepsilon_{ne}e$%

\end{frame}%
%

\begin{frame}%
%

\frametitle{The invariants of the closing $\varphi_{n}$}%

For a graph $g=(e,\diamond)\in\operatorname*{Inv}(\varphi_{n})$: any partial
or subgraph which does not create an isolated regional minimum also belongs to
$\operatorname*{Inv}(\varphi_{n})$\ \bigskip

Each node $i$ belonging to a regional minimum of $g$ is linked with at least
another node $j$ in this minimum through an edge with the same weight. This
edge is one of the lowest adjacent edges of $i$ and links $i$ with one of its
lowest neighboring nodes.\bigskip

For this reason after pruning, the graphs $\downarrow g$ and $\Downarrow g$
still belong to $\operatorname*{Inv}(\varphi_{n}).$%

\end{frame}%
%

\begin{frame}%
%

\frametitle{The regional minima of the closing $\varphi_{n}$}%

\begin{theorem}
If $G=(-,n)\in\operatorname*{Inv}(\varphi_{n})$ and $m=(-,n)$ is the subgraph
of its regional minima, then $\delta_{en}m=(\delta_{en}n,\diamond)$ is the
subgraph of the regional minima of the graph $\delta_{en}G=(\delta
_{en}n,\diamond)$
\end{theorem}

\textbf{Proof:} A regional minimum $m_{i}$ of a graph $G=(-,n)\in
\operatorname*{Inv}(\varphi_{n})$ is a plateau of pixels with altitude
$\lambda$, containing at least two nodes (there are no isolated regional
minima in $\operatorname*{Inv}(\varphi_{n})$).\ All internal edges of the
plateau get the valuation $\lambda$ by $\delta_{en}n.$ If an edge $(i,j)$ has
the extremity $i$ in the minimum and the extremity $j$ outside, then
$\delta_{en}n(i,j)>\lambda.\ $Hence, for the graph $(\delta_{en}n,\diamond),$
the edges spanning the nodes of $m_{i}$ form a regional minimum.\ %

\end{frame}%
%

\begin{frame}%
%

\frametitle{Inverse of $\delta_{en}$ on the invariants of the closing
$\varphi_{n}$}%

On $\operatorname*{Inv}(\varphi_{n}):\varepsilon_{ne}\delta_{en}=Identity$
showing that on $\operatorname*{Inv}(\varphi_{n}):\varepsilon_{ne}=\delta
_{en}^{-1}$%

\end{frame}%
%

\begin{frame}%

\begin{center}
{\Large \alert{The flooding graphs}}
\end{center}

%

\end{frame}%
%

\begin{frame}%
%

\frametitle{The flooding graph}%

\begin{definition}
An edge and node weighted spanning graph $G=[N,E]$ is a flooding graph iff its
weight distribution $(n,e)$ verify the relations:\newline- $\delta_{en}%
n=e$\newline- $\varepsilon_{ne}e=n$
\end{definition}

\begin{corollary}
For a flooding graph weight distribution $(n,e):$\newline- $n\in
\operatorname*{Inv}(\varphi_{n})$\newline- $e\in\operatorname*{Inv}(\gamma
_{e})$
\end{corollary}

\begin{proof}
$e=\delta_{en}n=\delta_{en}\varepsilon_{ne}e=\gamma_{e}e$ and $n=\varepsilon
_{ne}e=\varepsilon_{ne}\delta_{en}n=\varphi_{n}n$
\end{proof}

%

\end{frame}%
%

\begin{frame}%
%

\frametitle{Properties of the flooding graph}%

As $G$ is invariant by $\gamma_{e},$ all its edges are the lowest edge of one
of their extremities.

As $G$ is invariant by $\varphi_{n},$ it has no isolated regional minimum.%

\end{frame}%
%

\begin{frame}%
%

\frametitle{The lowest adjacent edges of each node in a flooding graph}%

In a flooding graph, $\varepsilon_{ne}e=n,$ hence all edges adjacent to a node
have weights which are higher or equal than this node and at least one of them
has the same weight.\ \bigskip

On the other hand, each node $i$ has at least one neighbor $j$ which is lower
or equal (otherwise it would be an isolated regional minimum).\ The weight
$e_{ij}$ verifies $e_{ij}=\delta_{en}n(i,j)=n_{i}\vee n_{j}=n_{i}.$ This shows
that the edges linking a node with lower or equal nodes have the same weight
than this node.\ \bigskip

In particular, the edges linking a node $i$ to its lowest neighboring nodes
belong to the lowest adjacent edges of this node and have the same weight:
$\Downarrow G\subset\ \downarrow G$ and $\downarrow\Downarrow G=\Downarrow
\downarrow G=\ \Downarrow G$%

\end{frame}%
\begin{frame}%
%

\frametitle{Constructing flooding graphs.\ }%

If $G=(e,\diamond)\in\operatorname*{Inv}(\gamma_{e})$ then $(e,\varepsilon
_{ne}e)$ is a flooding graph (since $e=\gamma_{e}e=\delta_{en}\varepsilon
_{ne}e=\delta_{en}n)$

\bigskip

If $G=(-,n)\in\operatorname*{Inv}(\varphi_{n})$ then $(\delta_{en}n,n)$ is a
flooding graph (since $n=\varphi_{n}n=\varepsilon_{ne}\delta_{en}%
n=\varepsilon_{ne}e)$%

\end{frame}%
%

\begin{frame}%
%

\frametitle{Deriving flooding graphs from ordinary graphs}%

If $G=(e,\diamond)$ is an arbitrary edge weighted graph, $\downarrow
(e,\diamond)\in\operatorname*{Inv}(\gamma_{e}),$ as the edges lowered by
$\gamma_{e},$ have been suppressed, the other edges keeping their weights.
Recall that $\varepsilon_{ne}e=\varepsilon_{ne}\downarrow e.\ $ The derived
flooding graph simply is $(\downarrow e,\varepsilon_{ne}\downarrow e)$

\bigskip

If $G=(-,n)$ is an arbitrary node weighted graph,$(-,\multimap n)\in
\operatorname*{Inv}(\varphi_{n}),$ as isolated regional minima, if any, have
been duplicated.\ The derived flooding graph simply is $(\delta_{en}\multimap
n,\multimap n).$%

\end{frame}%
%

\begin{frame}%
%

\frametitle{Partial graph of a flooding graph}%

Suppressing edges in a flooding graph, but leaving at least one lower
neighboring edge for each node (like that, no isolated regional minima are
created) produces a partial graph which also is a flooding graph, with the
same distribution of weights on the nodes and on the remaining edges.\ %

\end{frame}%
%

\begin{frame}%

\begin{center}
{\Large \alert{Regional minima of flooding graphs}}
\end{center}

%

\end{frame}%
%

\begin{frame}%
%

\frametitle{Regional minima of a flooding graph}%

We proved earlier these theorems:

\begin{itemize}
\item If $G=(e,\diamond)\in\operatorname*{Inv}(\gamma_{e})$ and $m=(e,\diamond
)$ is the subgraph of its regional minima, then $\varepsilon_{ne}%
m=(-,\varepsilon_{ne}e)$ is the subgraph of the regional minima of the graph
$\varepsilon_{ne}G=(-,\varepsilon_{ne}e).$

\item If $G=(-,n)\in\operatorname*{Inv}(\varphi_{n})$ and $m=(-,n)$ is the
subgraph of its regional minima, then $\delta_{en}m=(\delta_{en}n,\diamond)$
is the subgraph of the regional minima of the graph $\delta_{en}G=(\delta
_{en}n,\diamond).$
\end{itemize}

As in a flooding graph $n=\varepsilon_{ne}e$ and $e=\delta_{en}n$ we derive:

\begin{theorem}
If $G$ is a flooding graph with the weight distribution $(e,n)$, then the node
weighted graph $(-,n)$ and the edge weighted graph $(e,\diamond)$ have the
same regional minima subgraph
\end{theorem}

%

\end{frame}%
%

\begin{frame}%
%

\frametitle{The regional minima on the edge or node graph within the flooding
graph}%
%

\begin{figure}
[ptb]
\begin{center}
\includegraphics[
height=2.234in,
width=3.9196in
]%
{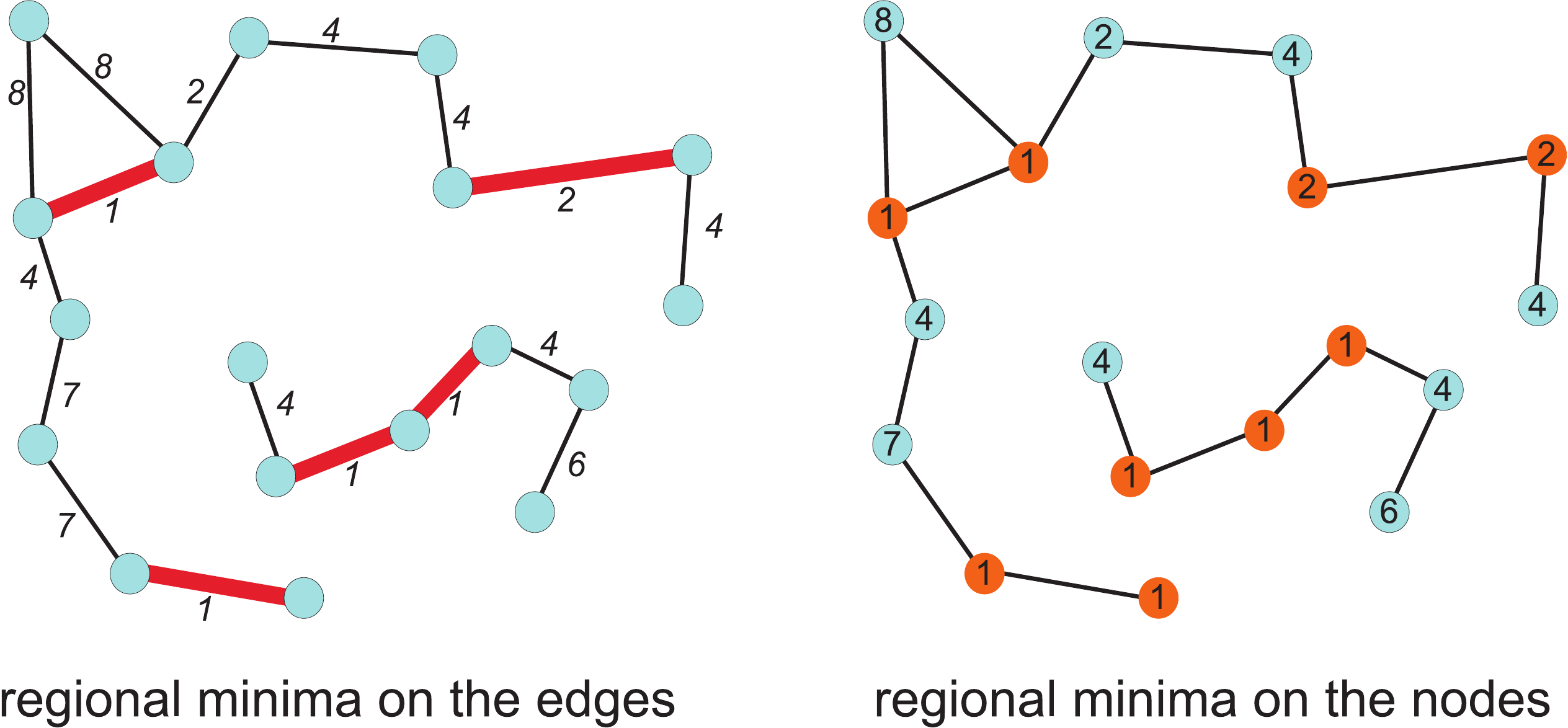}%
\end{center}
\end{figure}
%

\end{frame}%
%

\begin{frame}%
%

\frametitle{Labeling the regional minima on the edge or node graph within the
flooding graph}%

As the minima are identical on the nodes or the edges of a flooding graph, it
is possible to assign the same labels to nodes or to edges.%

\begin{figure}
[ptb]
\begin{center}
\includegraphics[
height=2.2173in,
width=2.0244in
]%
{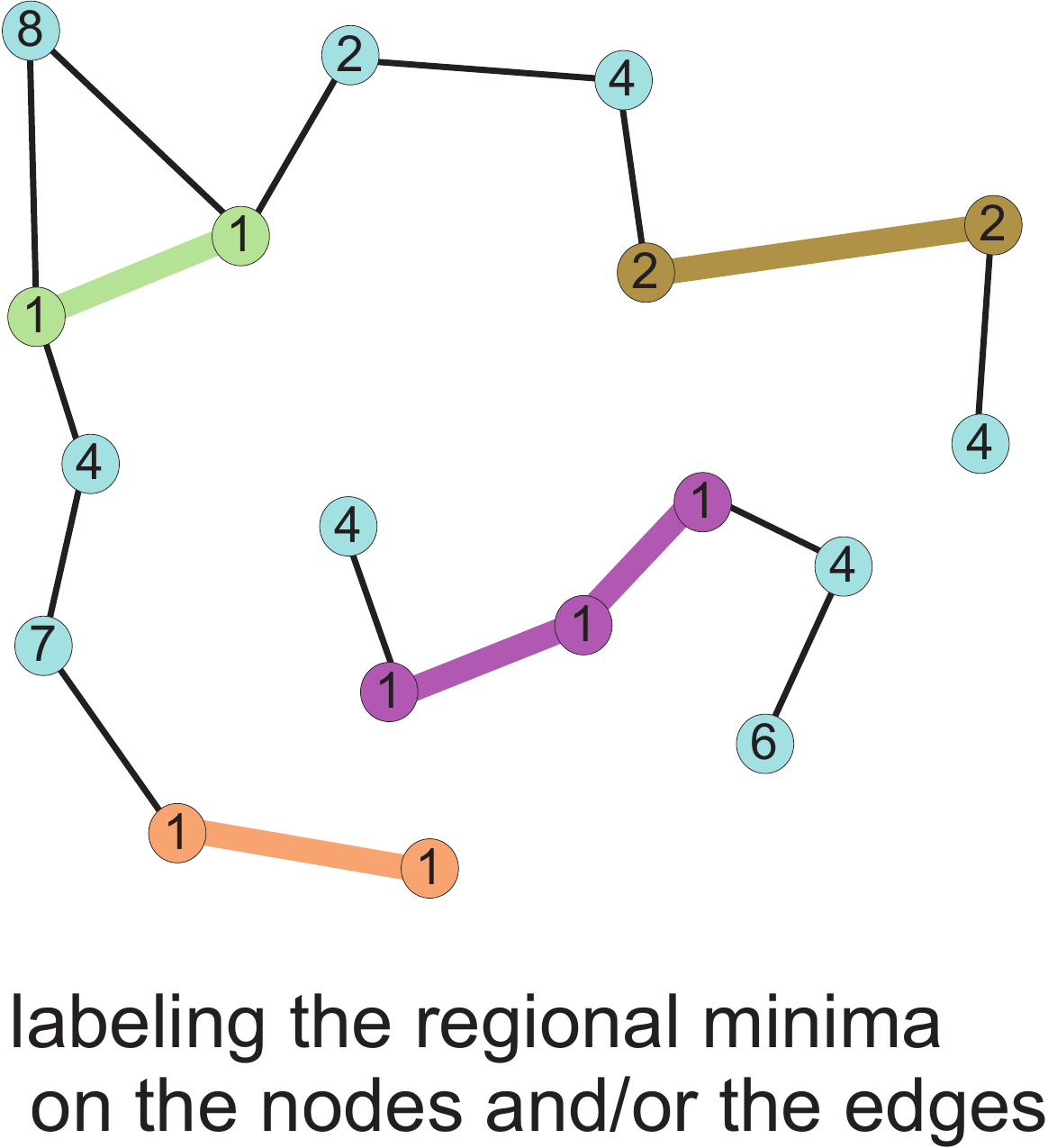}%
\end{center}
\end{figure}
%

\end{frame}%
%

\begin{frame}%

\begin{center}
{\Large \alert{Paths of steepest descent and catchment basins}}
\end{center}

%

\end{frame}%
%

\begin{frame}%
%

\frametitle{The catchment basins of the minima}%

\begin{lemma}
From each node outside a regional minimum starts a never ascending path to a
regional minimum in a flooding graph.
\end{lemma}

\textbf{Proof: }Any node $i$ outside a regional minimum has a lower
neighboring node, if it does not belong to a plateau. Otherwise it belongs to
a plateau, containing somewhere a node $j$ with a lower neighboring node
outside, as the plateau is not a regional minimum.\ The plateau being
connected, there exists a path of constant altitude in the plateau between $i$
and $j$.\ Following this path, it is possible, starting at node $i$ to reach
the lower node $k.$ This shows that for each node there exists a never
ascending path to a lower neighboring node. Taking this new node as starting
node, a still lower node may be reached. The process may be repeated until a
regional minimum is reached.

Thanks to this lemma, it is possible to define the catchment basins of the minima.\ %

\end{frame}%
%

\begin{frame}%
%

\frametitle{The catchment basins of the minima}%

\begin{definition}
The catchment basin of a minimum $m$ is the set of all nodes from which starts
a never ascending path towards $m.$
\end{definition}

As in a flooding graph, each node and its lower neighboring edges have the
same weight, each node in such a never ascending path is followed by an edge
with the same weight except the last one belonging to a regional minimum.\ For
this reason the catchment basins based on the node weights or on the edge
weights are identical.

\textbf{The watershed zones }are the nodes belonging to more than one
catchment basin.

\textbf{The restricted catchment basins} are the nodes which belong to only
one catchment basin : it is the difference between the catchment basin of a
minimum $m$ and the union of catchment basins of all other minima. From a node
in a restricted catchment basin, there exists a unique non ascending path
towards a unique regional minimum.\ %

\end{frame}%
%

\begin{frame}%
%

\frametitle{M-flooding graphs}%

The catchment basins rely entirely on the non ascending paths of the graph
reaching a regional minimum.\ The altitude of the regional minimum has no
importance.\ If we consider only the end points of such a path, the fact that
the minimum is an isolated node or not has no importance either. For this
reason, we may relax the definition of the flooding graphs for which
$\delta_{en}n=e$ and $\varepsilon_{ne}e=n$ for all edges and nodes.\ 

Consider a flooding graph where all non regional minima nodes have a positive
weight.\ Assigning to all minima a weight $0$ does not invalidate the relation
$\delta_{en}n=e.\ $The relation $\varepsilon_{ne}e=n$ also remains true,
except for isolated regional minima.\ We call M-flooding graphs, the node and
edge weighted graphs verifying $\delta_{en}n=e$ everywhere and $\varepsilon
_{ne}e=n$ for all nodes which are not isolated regional minima.\ 

In what follows we consider NAP which end with a node in a regional
minimum.\ Whether this minimum is isolated or not has no importance.\ %

\end{frame}%
%

\begin{frame}%

\begin{center}
{\Large \alert{Extending the restricted cathment basins and reducing the
watershed zone : steep, steeper, steepest flooding graphs}}
\end{center}

%

\end{frame}%
%

\begin{frame}%
%

\frametitle{Catchment basins and segmentation.\ }%

The watershed transform is mainly used for segmentation, with the aim to
create a partition representing precisely the extension of each object. Large
watershed zones are ambiguous as they separate restricted catchment basins
without precise localisation of the contour separating them. \ 

For obtaining precise segmentations, it is important to reduce these zones and
even suppress them completely if possible. Reducing the number of never
ascending paths from each node to a regional minimum would help constraining
the construction of the watershed partition.\ Ideally, if only one such path
remains for each node outside the regional minima, the solution would be
unique, the restricted catchment occupy the whole space and the watershed
zones be empty.\ 

In this section we show how to reduce the number of paths and keep only those
which have some degree of steepness. Increasing the steepness reduces the
number of paths.%

\end{frame}%
%

\begin{frame}%
%

\frametitle{Flooding tracks}%

In a flooding graph each edge is the lowest edge of one of its extremities and
has the same weigh: we call such a pair made of a node and adjacent edge with
the same weight \textbf{flooding pair}.\ 

Two couples $(i,ij)$ and $(j,jk)$ are chained if the node in the second couple
is an extremity of the edge of the first couple.

A \textbf{flooding track} is a list of chained couples of never increasing weight.

The \textbf{lexicographic weight} of a flooding track is the list of never
increasing weights of its pairs.

Two flooding tracks may then be compared by comparing their lexicographic
weights using the \textbf{lexicographic order relation}.%

\end{frame}%
%

\begin{frame}%
%

\frametitle{Pruning the flooding graph to get steeper paths.}%

We define a pruning operator $\downarrow^{k}$ operating on a flooding graph
$G.$ The pruning $\downarrow^{k}$ considers each node outside the regional
minima and suppresses its adjacent edges, if they are not the highest edge of
a flooding track verifying:

\begin{itemize}
\item their lexicographic weight is minimal.

\item their length is $k.$ It may be shorter if its last couple belongs to a
regional minimum
\end{itemize}

After pruning, each node outside the regional minima is the origin of one or
several k-steep flooding tracks (remark that the pruning only suppresses the
highest edge of the track).\ If there are several of them, they have the same weights.\ 

We call them the k-steep adjacent paths to the node $i.$ We say that the graph
$\downarrow^{k}G$ has a k-steepness or is k-steep.%

\end{frame}%
%

\begin{frame}%
%

\frametitle{Nested k-steep graphs}%

As the steepness degree increases, the pruning becomes more and more severe,
producing a decreasing series of partial graphs, hence for $k>l:\ \downarrow
^{k}G\subset\downarrow^{l}G.\ $Furthermore $\downarrow^{k}\downarrow
^{l}G=\downarrow^{l}\downarrow^{k}G=\downarrow^{k\vee l}G.$\bigskip

The pruning $\downarrow^{1}$ does nothing as each edge is the lowest edge of
one of its extremities in any graph invariant by $\gamma_{e}.$ On the
contrary, $\downarrow^{1}=\downarrow$ transforms an arbitrary graph into a
graph invariant by $\gamma_{e}.$\bigskip

The pruning $\downarrow^{2}$ keeps for each node $i$ the adjacent edges which
are followed by a second couple of minimal weight. These edges are those
linking $i$ with one of its lowest neighboring nodes.%

\end{frame}%
%

\begin{frame}%

\begin{center}
{\Large \alert{A morphological characterization of k-steep graphs}}
\end{center}

%

\end{frame}%
%

\begin{frame}%
%

\frametitle{Eroding k-steep graphs}%

Consider a k-steep graph $\downarrow^{k}G=G=(e,n),$ with $k\geq2.$ We define
the erosion $\varepsilon G=(\varepsilon_{e}e,\varepsilon_{n}n),$ where
$\varepsilon_{e}=\varepsilon_{en}\varepsilon_{ne}$ and $\varepsilon
_{n}=\varepsilon_{ne}\varepsilon_{en}.$

$\varepsilon^{(1)}G=\varepsilon G$ and $\varepsilon^{(m)}G=\varepsilon
\varepsilon^{(m-1)}G.$\bigskip

It does not change the catchment basins if we assign to the regional minima
the weight $0$, whereas all other nodes and edges have weights $>0$ (like that
as soon the erosion assigns the value $0\ $to a node or an edge, they remain
equal to $0$ for all subsequent erosions).%

\end{frame}%
%

\begin{frame}%
%

\frametitle{Eroding k-steep graphs}%

Consider a k-steep lexicographic track $\tau$ of a flooding graph $G$ made of
a series of non increasing flooding pairs $(i_{1},e_{1}),(i_{2},e_{2}%
),\ldots,(i_{k},e_{k}).$ Node and edge of each flooding pair have the same
weights.\ As $\tau$ has a minimal lexicographic weight, each of its flooding
pairs, except the first one constitutes one of the lowest adjacent flooding
pair of the previous flooding pair.\ For this reason the erosion $\varepsilon
G$ assigns to the edge $e_{h}$ the weight of the adjacent edge $e_{h+1}$ and
to the node $i_{h}$ the weight of the adjacent node $i_{h+1}$.\ In other
workds successive erosions $(\varepsilon_{e}e,\varepsilon_{n}n)$ let glide the
value of each pair upwards in the track. If this track is of length $k,$ after
$k-1$ erosion, the weight of the ultimate pair will have reached the first one.

If such a track is of length $l<k,$ it ends with a pair in a regional minimum,
with the value $0.\ $Successive erosions $(\varepsilon_{e}e,\varepsilon_{n}n)$
also let glide the value of each pair upwards in the track and the last pair,
with value $0$ also moves upwards and reaches the pair $(i,ij)$ after $l-1$
erosions.\ During the next erosions, the value of the pair $(i_{1},e_{1})$
remains stable and equal to $0.$%

\end{frame}%
%

\begin{frame}%
%

\frametitle{Eroding k-steep graphs}%

Consider a k-steep flooding graph $G=\downarrow^{k}G.\ $During the $k-1$
successive erosions, the values of the flooding pairs glide upwards along the
k-steep lexicographic path.\ For the erosion $l<k,$ the $(l+1)th$ pair
$(s,st)$ has reached the pair $(i,ij)$. Hence the edge $ij$ remains one of the
lowest adjacent edge of the node $i$ and both share identical weights.\ 

But $(s,st)$ belongs to the flooding graph and verifies $n_{s}=\varepsilon
_{ne}e(s)$ and $n_{s}=\varepsilon_{ne}e(s)$ and so does $(i,ij).\ $

This shows that the graph $\varepsilon^{(l)}G$ still is a flooding graph.\ 

\begin{theorem}
For an ordinary flooding graph $G,$ $\downarrow^{k}G$ is a k-steep flooding
graph and for $l<k,$ $\varepsilon^{(l)}\downarrow^{k}G$ still is a flooding graph.
\end{theorem}

%

\end{frame}%
%

\begin{frame}%
%

\frametitle{Eroding flooding graphs}%

Consider a flooding graph $G=(e,n).$ The erosion $\varepsilon_{n}%
=\varepsilon_{ne}\varepsilon_{en}$ enlarges the regional minima.\ Hence if $G$
is invariant by $\varphi_{n}$ it is still the case for the graph
$(e,\varepsilon_{n}n)$

On the contrary, after the erosion $\varepsilon_{e}=\varepsilon_{en}%
\varepsilon_{ne},$ the graph $(\varepsilon_{e}e,n)$ is not necessarily
invariant for $\gamma_{e}.$ One has to prune the edges which are not the
lowest edges of one of their extremities: $\downarrow\varepsilon_{e}e.$

Repeating the operator $\zeta G=$\ $(\downarrow\varepsilon_{e}e,\varepsilon
_{n}n)$ produces a decreasing series of partial graphs $\zeta^{(n)}%
G=\zeta\zeta^{(n-1)}G$ we will now characterize.%

\end{frame}%
%

\begin{frame}%
%

\frametitle{Two equivalent modes of pruning.}%

\begin{theorem}
For $m\geq1,$ and defining $\zeta^{(0)}=identity,$ the operators $\zeta$
applied to $\zeta^{(m-1)}G$ and $\downarrow^{m+1}$ applied to $\downarrow
^{m}G$ keep and discard the same edges.\ The operators $\zeta^{(k)}.$ and
$\downarrow^{k+1}$ produce partial spanning graphs with the same edges and nodes.
\end{theorem}

\textbf{Proof: }We prove it by induction.

a) $m=1:$ $\zeta$ and $\downarrow^{2}$ select the edges linking a node with
its lowest node (this covers the case where the edge belongs to a regional minimum)%

\end{frame}%
%

\begin{frame}%
%

\frametitle{Two equivalent modes of pruning.}%

b) We suppose that $\zeta^{(m-1)}G$ and $\downarrow^{m}G$ have the same edges
and show that it is still the case for $m=m+1.$

Consider a couple $(i,ij)$ of $\zeta^{(m-1)}G$.\ It belongs to a
$m-1$-steepest track.\ The second pair $(j,jk)$ of this track also belongs to
a $\left(  m-1\right)  -$steepest track and holds the weight of the lowest
pair of the track.\ The operator $\zeta$ applied to $\zeta^{(m-1)}G$ assigns
this weight to the edge $ij$ but not necessarily to the node $i\ $if there
exists another pair $(i,il)$ with a lower weight after this erosion.\ In this
case the edge $ij$ is discarded by the operator $\zeta.\ \ $But it is also
discarded by the operator $\downarrow^{m+1}G,$ as no $\left(  m+1\right)
$-steepest path passes through $ij.$%

\end{frame}%
%

\begin{frame}%
%

\frametitle{Constructing k-steepest graphs}%

The operator $\downarrow^{m}G$ is of theoretical interest but of poor
practical value, as it is based on a neighborhood of size $m.\ $On the
contrary the operator $\zeta$ is purely local and uses a neighborhood of size
1 : an erosion from node to node, an erosion from edge to edge and the
suppression of any edge which is not the smallest edge of one of its extremities.

Repeating $m$ times the operator $\zeta$ produces a graph which has the same
edges as the operator $\downarrow^{m+1}G.\ $It is then sufficient to restore
the original weights of edges and nodes of the graph $G\ $onto the graph
$\zeta^{(m)}G$ to obtain the same result as $\downarrow^{m+1}G.$%

\end{frame}%
%

\begin{frame}%
%

\frametitle{Characterizing k-steepest graphs}%

\begin{theorem}
A graph $G$ is a k-steepest graph if and only if for each $l<k$ ,
$\varepsilon^{(l)}G$ is a flooding graph
\end{theorem}

\textbf{Proof: }If a graph $G$ is a k-steepest graph, then $G=\downarrow^{k}G$
and we have already established that for $l<k,$ $\varepsilon^{(l)}%
\downarrow^{k}G$ is a flooding graph.

Inversely suppose that for each $l<k$ , $\varepsilon^{(l)}G$ is a flooding
graph, then $\varepsilon^{(l)}G$ has the same edges as $G.\ $As a matter of
fact $\varepsilon^{(l)}G$ is invariant by $\gamma_{e}$ and each edge is the
lowest edge of a node: $\downarrow\varepsilon^{(l)}G=\varepsilon^{(l)}G.$

But $\varepsilon^{(l)}G=\varepsilon\varepsilon^{(l-1)}G$ and $\downarrow
\varepsilon\varepsilon^{(l-1)}G=(\downarrow\varepsilon_{e}e,\varepsilon
_{n}n)\varepsilon^{(l-1)}G=\zeta\varepsilon^{(l-1)}G$ (considering
$G=\varepsilon^{(0)}G$). It follows that for $l<k,$ we have $\varepsilon
^{(l)}G=\zeta^{(l)}G.$

The operator $\zeta$ applied to $k$ times to $G$ never suppresses an edge from
$G.$ But we know that $\zeta^{(k-1)}G$ and $\downarrow^{k}G$ have the same
nodes and edges, showing that indeed $G$ is a k-steepest graph.\ %

\end{frame}%
%

\begin{frame}%

\begin{center}
{\Large \alert{Deriving an algorithm for delineating the catchment basins}}
\end{center}

%

\end{frame}%
%

\begin{frame}%
%

\frametitle{The catchment basins of k-steepest paths}%

It is now possible to imagine an algorithm associated to the operator
$\zeta.\ $If $G$ is a flooding graph, we assign a a distinct label to each
regional minimum and a weight $0$.\ The operator $\zeta$ propagates the
weights upwards along the steepest tracks as it is repeated. In parallel, we
also propagate the labels: every time a pair $(i,ij)$ takes the weight $0$
from a pair containing a labeled node $j,$ the node $i$ takes the label of
$j.$ Like that, as the labeled zones with weight $0$ expand, their labels also expand.

We now have to consider two cases.\ In the first case, we construct restricted
catchment basins separated by watershed zones.\ In the second, we create a
partition by assigning each node to one and only one basin, at the price of
arbitrary choices.%

\end{frame}%
%

\begin{frame}%
%

\frametitle{Catchment basins and watershed zones}%

During the successive operations $\zeta,$ every time a pair $(i,ij)$ takes the
weight $0$ from a pair containing the node $j,$ the node $i$ takes the label
of $j$ provided the pair $(i,ij)$ is unique. If there are two equivalent pairs
$(i,ij)$ and $(i,il)$, it means that there exist two minimal lexicographic
tracks starting at $i$ towards one or two distinct minima. If $j$ and $k$ hold
the same label, this label is assigned to $i.\ $If on the contrary they are
distinct, we have 2 possibilities:

\begin{itemize}
\item we assign to $i$ a label $Z$, indicating that it belongs to a watershed
zone.\ The upstream of $i$ also will get this same label $Z.\ $The regions
with label $Z$ belong to the watershed zone and the other labeled regions are
restricted catchment basins.

\item we assign to $i$ one of the labels (either randomly, or by applying an
additional rule for braking the ties), producing a partition in catchment
basins with an empty watershed zone.\ The result is not unique and depends on
the succession of choices which have been made.\ 
\end{itemize}

%

\end{frame}%
%

\begin{frame}%
%

\frametitle{Creating a partition}%
%

\begin{figure}
[ptb]
\begin{center}
\includegraphics[
height=1.2252in,
width=4.4865in
]%
{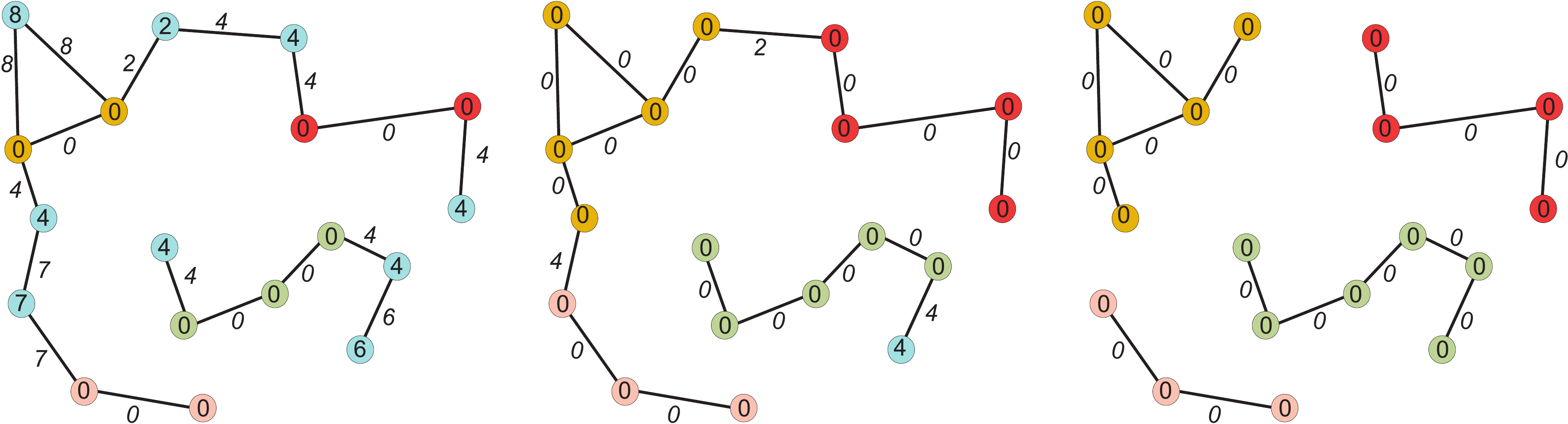}%
\end{center}
\end{figure}

On the left a flooding graph where the minima have weights equal to $0$ and
their labels are indicated by distinct colors. The next two figures show the
propagation of the weights and of the labels as the operator $\zeta$ has been
applied twice in a sequence.%

\end{frame}%
%

\begin{frame}%
%

\frametitle{Partial conclusion}%

Starting with a node or edge weighted graph we have shown how to extract from
it a flooding graph where nodes and edges are weighted.

The operators $\downarrow^{m}G$ extract from the graph $G$ partial graphs
which are steeper and steeper flooding graphs. The series of partial graphs
$\downarrow^{m}G$ is decreasing with $m.$

The operator $\zeta$ permits an iterative construction of $\downarrow^{m}G$,
using only small local neighborhood transformations.\ %

\end{frame}%
%

\begin{frame}%

\begin{center}
{\Large \alert{The scissor operator and the watershed partitions}}
\end{center}

In the previous section we introduced pruning operators which extract from a
flooding graph k-steep flooding graphs. These operators do not make any choice
among the k-steep flooding graphs, they take them all. \bigskip

We now introduce operators aiming at creating partitions : they extract
minimum spanning forests from the flooding graph by pruning. Contrarily to the
preceding operators, they do make arbitrary choices among equivalent edges to
be suppressed.%

\end{frame}%
%

\begin{frame}%
%

\frametitle{The flooding pairs}%

The previous section has shown how to extract from a node or edge weighted
graph partial graphs which are flooding graphs. In a flooding graph, each node
has at least one adjacent edge with the same weight.\ We consider here the
nodes and edges outside the regional minima.\ There exists at least one (in
general several) one to one correspondance between each node outside a
regional minimum and one of its adjacent edges with the same weight.\ For each
such one to one correspondance, we call flooding pair, the couple of node and
edge which have been associated. They hold the same weight.\ %

\end{frame}%
%

\begin{frame}%
%

\frametitle{The flooding pairs}%

Let us show how to construct such a one to one correspondance.\ If an edge is
the lowest adjacent edge of only one of its extremities, then they form a
flooding pair ; this is in particular the case if the other extremity of the
edge has a lower weight.

If on the contrary, it has the same weight as both its extremities, it belongs
to a plateau.\ As this plateau is not a regional minimum, there exists a pair
of neighboring nodes, a node $s$ inside the plateau and a lower node $t$,
outside.\ This edge $(s,t)$ forms a flooding pair with the node $s.\ $As the
plateau is connected, there exists a tree spanning its nodes, having $s$ as
root.\ There exists a unique path between $s$ and each node $i$ of the
plateau.\ The last edge on this path before reaching $i$ and $i$ itself form a
flooding pair.%

\end{frame}%
%

\begin{frame}%
%

\frametitle{The scissor operator creates a partition of the nodes}%

We have shown that there exists at least one (in general several) one to one
correspondance between each node outside a regional minimum and one of its
adjacent edges with the same weight.\ For each such one to one correspondance,
we call flooding pair, the couple of node and edge which have been associated.

To each such one to one correpondance we associate a a pruning operator $\chi
$, called scissor.$\ $This operator suppresses all edges outside a regional
minimum which do not form a flooding pair with one of their
extremities.\ \bigskip

After applying the operator $\chi\ $to a flooding graph, there exists one and
only one path from each node to a regional minimum: the catchment basins of
the minima partition the nodes, and the watershed zones are empty%

\end{frame}%
%

\begin{frame}%

\begin{center}
{\Large \alert{The drainage minimum spanning forest}}
\end{center}

%

\end{frame}%
%

\begin{frame}%
%

\frametitle{The drainage minimum spanning forest}%

After contracting all edges in the regional minima of a flooding graph and
applying the scissor operator $\chi$ one gets a spanning forest, where each
tree is rooted in a minimum:

\begin{itemize}
\item the resulting graph is a partition where each connected component
contains a regional minimum node : there exists a path linking each node
outside the minima with a minimum.

\item each connected component is a tree as the number of nodes is equal to
the number of edges (the number of flooding pairs) plus one (the regional
minimum node).

\item it is a minimum spanning forest : each node is linked to its tree
through one of its lowest neighboring edge.\ The total weight of each tree is
thus equal to the total sum of the nodes outside the regional minima.\ This
total weight is independent of the particular choice made by $\chi.$
\end{itemize}%

\end{frame}%
%

\begin{frame}%
%

\frametitle{The drainage minimum spanning forest}%

Expanding again the regional minima and replacing them by a MST of each
regional minimum creates again a minimum spanning forest, identical to the
preceding one outside the minima.

Its total weight is computed as follows:

\begin{itemize}
\item if $M_{i}$ is a regional minimum with $n$ nodes of weight $\lambda_{i},$
the weight of its MST is equal to $(n-1)\ast\lambda_{i}$

\item each other node contributes by its weight, which is also the weight of
its lowest adjacent edge.
\end{itemize}

Each particular forest is based on a particular scissor operator $\chi$ and of
a particular MST in each regional minimum.\ We call $\phi$ the operator which
extracts from a flooding graph such a MSF (minimum spanning forest).%

\end{frame}%
%

\begin{frame}%
%

\frametitle{The catchment basins}%

$G$ = a flooding graph, $C=\phi G$ a minimum spanning forest:

\begin{itemize}
\item Each node belongs to a regional minimum or is the origin of a unique
never ascending path leading to a regional minimum

\item each tree of $C$ contains a unique regional minimum ; its nodes form the
catchment basin of this minimum.\ 
\end{itemize}

%

\end{frame}%
%

\begin{frame}%
%

\frametitle{Catchment basins of increasing steepness}%

$G$ = a flooding graph, $\downarrow^{m}G$ still is a flooding graph and
$C_{m}=\phi\downarrow^{m}G$ a minimum spanning forest of steepness $m.$

As for $m>l:$ $\downarrow^{m}G\subset\downarrow^{l}G$ , any minimum spanning
forest $C_{m}=\phi\downarrow^{m}G$ is also a minimum spanning forest of
$\downarrow^{l}G.\ $

For increasing values of $m,$ the number of forests of steepness $m$ decreases.%

\end{frame}%
%

\begin{frame}%

\begin{center}
{\Large \alert{Illustration}}
\end{center}

%

\end{frame}%
%

\begin{frame}%
%

\frametitle{A flooding graph associated to a distance function}%

We chose as topographic surface the distance function expressed on the nodes
of a hexagonal grid to two binary connected sets, encoded with the value
$0.\ $The edge weights are obtained by the dilation $\delta_{en}.\ $Like that
the lowest neighboring edges of a node connects it with its neighbors with
equal or lower weights. In the following figure, the edges with the same
weight have the same color (cyan = 4, magenta = 3, green = 2).%

\begin{figure}
[ptb]
\begin{center}
\includegraphics[
height=0.7933in,
width=1.2646in
]%
{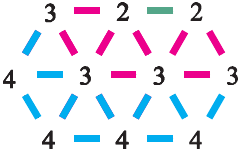}%
\end{center}
\end{figure}
%

\end{frame}%
%

\begin{frame}%
%

\frametitle{The flooding graph of a distance function}%

As the minima have 2\ nodes each, the pixel graph is invariant by $\varphi
_{n}$.\ The dilation $\delta_{en}$ assigns weights to the edges and creates a
flooding graph.

\begin{figure}[ptb]
\begin{center}
\includegraphics[
natheight=2.938600in,
natwidth=3.496400in,
height=2.2206in,
width=2.6374in
]{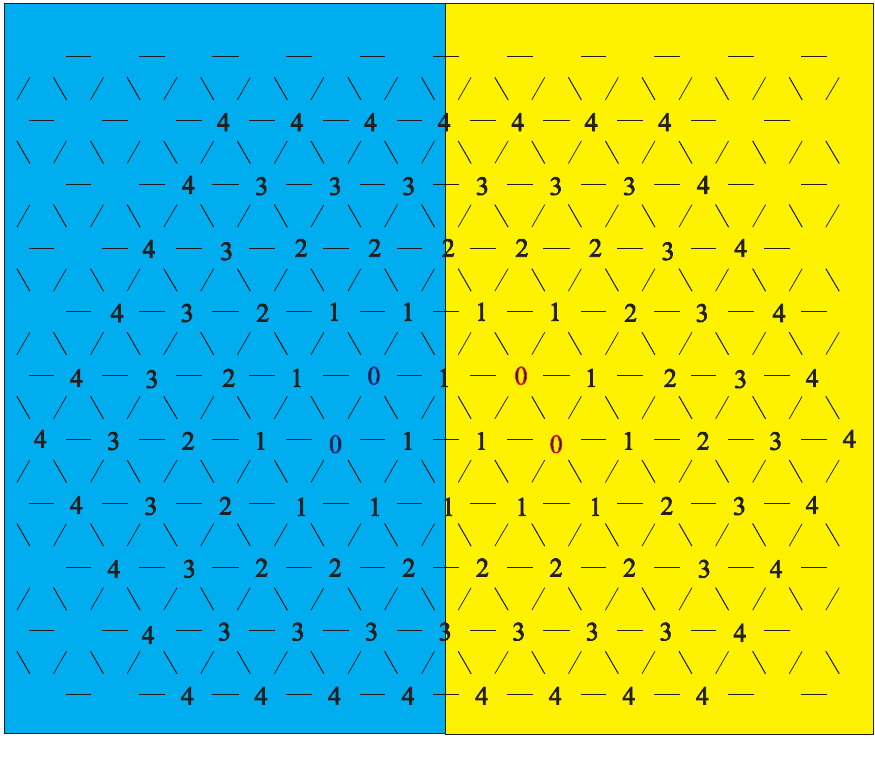}
\end{center}
\end{figure}
%

\end{frame}%
%

\begin{frame}%
%

\frametitle{The minimum spanning forest by keeping one lowest neighboring edge
for each node}%

The scissor $\chi$ leaves one lowest neighboring edge for each node. There
exists a huge number of choices for $\chi.\ $The following illustration shows
a particular scissor $\chi$ producing an unexpected partition in two catchment
basins (the catchment basins should be the half plane separated by the
mediatrix of both binary sets, as illustrated in the previous slide).\ %

\begin{figure}
[ptb]
\begin{center}
\includegraphics[
height=1.8072in,
width=2.4965in
]%
{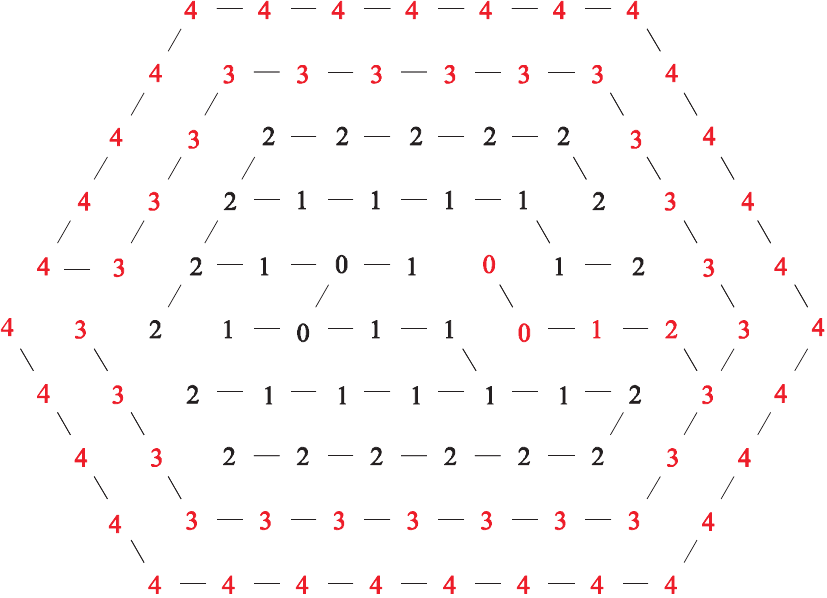}%
\end{center}
\end{figure}
%

\end{frame}%
%

\begin{frame}%

Two geodesics for the flooding ultrametric distance, leading to an unexpected partition.\ 

\begin{figure}[ptb]
\begin{center}
\includegraphics[
natheight=2.608300in,
natwidth=3.373600in,
height=1.9749in,
width=2.5452in
]{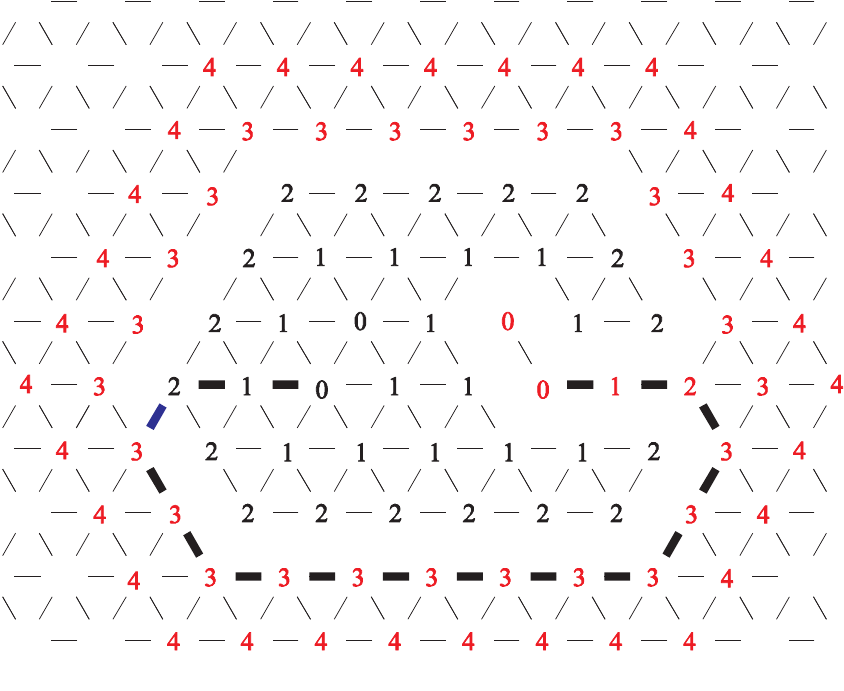}
\end{center}
\end{figure}
%

\end{frame}%
%

\begin{frame}%
%

\frametitle{Looking one node further}%

The operator $\downarrow^{2}G$ leaves only the edges linking a node to its
lowest neighbors.\ The following figure shows the remaining edges.\ The green
zone represents a restricted catchment basin.\ The yellow zone is an extended
catchment basin containing the watershed zone.%

\begin{figure}
[ptb]
\begin{center}
\includegraphics[
height=2.2122in,
width=2.6324in
]%
{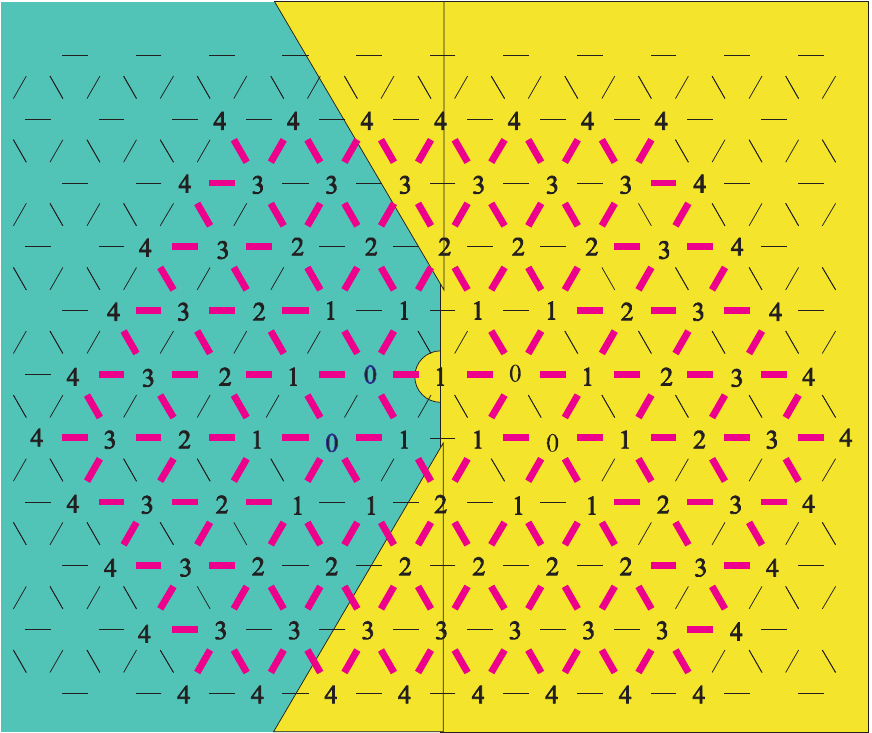}%
\end{center}
\end{figure}
%

\end{frame}%
%

\begin{frame}%
%

\frametitle{The minimum spanning forest by keeping one edge towards a lowest
neighboring node}%

The pruning $\chi\Downarrow=\chi\downarrow^{2}$ leaves for each node one edge
towards a lowest neighboring node.\ The following figure shows one such
solution and the resulting partition.%

\begin{figure}
[ptb]
\begin{center}
\includegraphics[
height=1.808in,
width=2.4856in
]%
{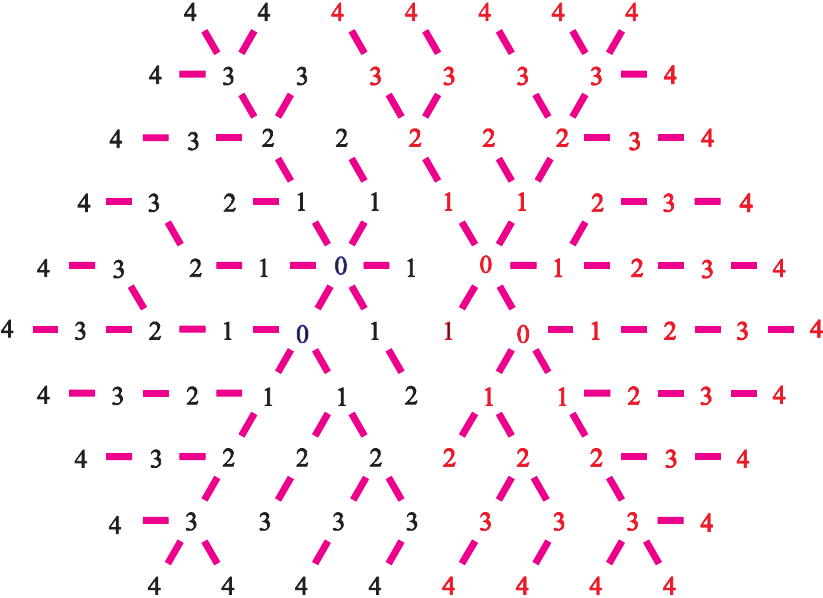}%
\end{center}
\end{figure}
%

\end{frame}%
%

\begin{frame}%

\begin{center}
{\Large \alert{The catchment basins, skeletons by zone of influence for lexicographic distance functions}}%

\end{center}

%

\end{frame}%
%

\begin{frame}%
%

\frametitle{The geodesics of the k-steepest graphs}%

The operator $\downarrow^{k}G$ prunes the flooding graph and leaves only NAPs
with a steepness equal to $k$.\ From each node of the graph starts a NAP whose
$k$ first edges have a minimal lexicographic weight. If one follows such a
path, the next $k$ edges following each node as one goes downwards along the
path also has a minimal lexicographic weight. This shows that each NAP is a
geodesic line for a lexicographic distance function $\operatorname*{lexdist}%
_{k}$ we define below.

After the pruning $\downarrow^{k}$ some nodes belong to two or more catchment
basins.\ In order to obtain a partition, one applies the scissor operator
$\chi$ leaving for each node only one lowest adjacent edge.\ Like that the
thick watershed zones are suppressed and the catchment basins form a partition.%

\end{frame}%
%

\begin{frame}%
%

\frametitle{Constructing a watershed partition as the skeleton by zones of
influence of the minima}%

It is possible to obtain the same result, if one labels the regional minima
and computes for the other nodes the shortest lexicographic distance
$\operatorname*{lexdist}_{k}$ to the minima.\ If one applies a greedy
algorithm, one may in addition propagate the labels of the minima all along
the geodesics and construct a partition of the space.\ If a node is at the
same lexicographic distance of two nodes, a greedy algorithm will arbitrarily
assign to it one of the labels of the minima. This is quite similar to the
operator $\chi$ which also does arbitrary choices.\ %

\end{frame}%
%

\begin{frame}%
%

\frametitle{Defining a lexicographic distance along non ascending paths.\ }%

Consider a flooding graph.\ In a NAP, each node except the last one forms with
the following edge a flooding pair, i.e.\ they have the same weights.\ And
these weights are decreasing as one follows the NAP downwards. The $k$ first
values, starting from the top, may be considered as a lexicographic distance
of depth $k.\ $In what follows we give a precise definition of such distances.\ 

The shortest distances and their geodesics may be computed wih classical
algorithms.\ Propagating the labels of the minima along the geodesics during
their construction constructs the zones of influence of the minima, i.e a
partition into catchment basins. The solution is not necessarily unique.
However the number of solutions decreases with the depth of the lexicographic
distance which is considered.%

\end{frame}%
%

\begin{frame}%

\begin{center}
{\Large \alert{Lexicographic distances of depth k}}
\end{center}

%

\end{frame}%
%

\begin{frame}%
%

\frametitle{Comparing NAPs with the lexicographic order.\ }%

Let $S$ be the set of sequences of edge or node weights, i.e.\ elements of
$\mathcal{T}$ . Let $S_{k}$ be the set of sequences with a maximal number $k$
of elements.\ For a sequence $s\in S_{k},$ we define the lexicographic weight
$w_{k}(s)$ : it is equal to $\infty$ if $s$ is not a NAP and equal to $s$
itself otherwise.\ 

We define an operator $\operatorname*{first}_{k}$ which keeps the $k$ first
edges and nodes of any NAP, or the NAP completely if its length is smaller
than $k.$ The operator $\operatorname*{first}_{k}$ maps any sequence of $S$
into $S_{k}.$

We define on the NAPs of $S_{k}$ the usual lexicographic order relation, which
we will note $\prec$, such that: $(\lambda_{1},\lambda_{2},\dots,\lambda
_{k})\prec(\mu_{1},\mu_{2},\dots,\mu_{k})$ if either $\lambda_{1}<\mu_{1}$ or
$\lambda_{i}=\mu_{i}$ until rank $s$ and $\lambda_{s+1}<\mu_{s+1}$. We define
$a\preceq b$ as $a\prec b$ or $a=b$.%

\end{frame}%
%

\begin{frame}%
%

\frametitle{Comparing NAPs with the lexicographic order.\ }%

Like that, it is possible to compare any two sequences $s_{1}$ and $s_{2}$ of
$S_{k}$ by comparing their weights:

\begin{itemize}
\item if $s_{1}$ and $s_{2}$ are not NAPs, then $w_{k}(s_{1})=w_{k}%
(s_{2})=\infty$ and we consider that $s_{1}\equiv s_{2}$ (they are equivalent)

\item if one of them, say $s_{1},$ is a NAP and not the other, then
$w_{k}(s_{1})\prec w_{k}(s_{2})=\infty$ and $s_{1}\prec s_{2}$

\item if both of them are NAP, then they compare as their weights: $s_{1}\prec
s_{2}\Leftrightarrow\left\{  w_{k}(s_{1})\prec w_{k}(s_{2})\right\}  $ and
$s_{1}\equiv s_{2}\Leftrightarrow\left\{  w_{k}(s_{1})=w_{k}(s_{2})\right\}  $
\end{itemize}

If $s_{1}$ and $s_{2}$ are NAPs belonging to $S,$ we compare them with:
$s_{1}\prec s_{2}\Leftrightarrow$ $w_{k}\left[  \operatorname*{first}%
_{k}(s_{1})\right]  \prec w_{k}\left[  \operatorname*{first}_{k}%
(s_{2})\right]  $%

\end{frame}%
%

\begin{frame}%
%

\frametitle{Defining an "addition" operator, comparing sequences}%

For NAPs of $S$ we define the operator $\boxplus_{k}$ called \textquotedblleft
addition\textquotedblright, which operates as a minimum:\newline$a\boxplus
_{k}b=%
\begin{array}
[c]{c}%
a\text{ if }a\preceq b\\
b\text{ if }b\preceq a
\end{array}
\qquad\forall a,b\in S$\newline The $\boxplus_{k}$ operation is associative,
commutative and has a neutral element $\infty$ called the zero element:
$a\boxplus_{k}\infty=\infty\boxplus_{k}a=a$%

\end{frame}%
%

\begin{frame}%
%

\frametitle{Defining a "multiply" operator, concatenating sequences}%

The operator $\boxtimes_{k}$, called \textquotedblleft
multiplication\textquotedblright,\ permits to compute the lexicographic length
$\Lambda(A)$ of a sequence, obtained by the concatenation of two sequences.
Let $a=(\lambda_{1},\lambda_{2},\dots,\lambda_{k})$ and $b=(\mu_{1},\mu
_{2},\dots,\mu_{k})$ we will define $a\boxtimes_{k}b$ by:

\begin{itemize}
\item if $a$ or $b$ is not a NAP then $a\boxtimes_{k}b=\infty$

\item if $a$ and $b$ are NAPs and $\lambda_{k}<\mu_{1}$ then $a\boxtimes
_{k}b=\infty$

\item if $a$ and $b$ are NAPs and $\lambda_{k}\geq\mu_{1}$ then $a\boxtimes
_{k}b=w_{k}\left[  \operatorname*{first}_{k}(a\rhd b)\right]  ,$ where $a\rhd
b$ is the concatenation of both sequences.
\end{itemize}

In particular $a\boxtimes_{k}\infty=\infty\boxtimes_{k}a=\infty$. so that the
zero element is an absorbing element for $\boxtimes_{k}$.%

\end{frame}%
%

\begin{frame}%

\begin{center}
{\Large \alert{Algebraic shortest paths algorithms}}
\end{center}

%

\end{frame}%
%

\begin{frame}%
%

\frametitle{A path algebra on a dio\"{\i}d structure}%

The operator $\boxtimes_{k}$ is associative and has a neutral element $0$
called unit element: $a\boxtimes_{k}0=a.$ The multiplication is distributive
with respect to the addition both to the left and to the right.

The structure $(S,\boxplus_{k},\boxtimes_{k})$ forms a dio\"{\i}d, on which
Gondran and Minoux defined a path algebra \cite{gondranminoux}, where shortest
paths algorithms are expressed as solutions of linear systems.%

\end{frame}%
%

\begin{frame}%
%

\frametitle{Transposing the dio\"{\i}d to square matrices}%

Addition and multiplication of square matrices of size $n$ derives from the
laws $\boxplus$ and $\boxtimes_{k}$ : for $A=(a_{ij})$, $B=(b_{ij})$,
$i,j\in\left[  1,n\right]  $ :

\begin{center}
$C=A\boxplus B=(c_{ij})\Leftrightarrow c_{ij}=a_{ij}\boxplus b_{ij}%
\qquad\forall i,j$

$C=A\boxtimes_{k}B=(c_{ij})\Leftrightarrow c_{ij}=\sum\limits_{1\leq k\leq
n}^{\boxplus}a_{ik}\boxtimes_{k}b_{kj}\qquad\forall i,j$
\end{center}

where $\sum\limits^{\boxplus}$ is the sum relative to $\boxplus.$

As there is no ambiguity, for $\sum\limits_{1\leq k\leq n}^{\boxplus}%
a_{ik}\boxtimes_{k}b_{kj},$ we simply write $\sum\limits_{1\leq k\leq n}%
a_{ik}b_{kj}$%

\end{frame}%
%

\begin{frame}%
%

\frametitle{Unity and zero matrices }%

With these two laws, the square matrices also become a dio\"{\i}d with
\newline zero matrix $\widehat{\varepsilon}=\left[
\begin{array}
[c]{c}%
\varepsilon...........\varepsilon\\
.............\\
.............\\
.............\\
.............\\
\varepsilon...........\varepsilon
\end{array}
\right]  $\newline and unity matrix $E=\left[
\begin{array}
[c]{c}%
e...........\varepsilon\\
..e..........\\
....e........\\
......e.......\\
.........e....\\
\varepsilon...........e
\end{array}
\right]  $\newline We write $A^{0}=E$%

\end{frame}%
%

\begin{frame}%
%

\frametitle{Lexicographic length of paths in a graph}%

If $G=\left[  X,U\right]  $ is a weighted graph with

\begin{itemize}
\item a set $X$ of nodes, numbered $i=1,......,N$.

\item a set $U$ of edges $u=(i,j$) with weights $s_{ij}\in S.$
\end{itemize}

The incidence matrix $A=(a_{ij})$ of the graph is given by:

\begin{center}
$a_{ij}=\left\{
\begin{array}
[c]{c}%
s_{ij}\qquad\text{si}\qquad(i,j)\in U\\
\varepsilon\qquad\text{sinon}%
\end{array}
\right\}  $
\end{center}

%

\end{frame}%
%

\begin{frame}%
%

\frametitle{A path algebra on a dio\"{\i}d structure}%

To each path $\mu=\left(  i_{1},i_{2},....i_{k}\right)  $ of the graph, one
associated its k-lexicographic weight $w(\mu)=s_{i_{1}i_{2}}\boxtimes
_{k}s_{i_{2}i_{3}}\boxtimes_{k}....\boxtimes_{k}s_{i_{k-1}i_{k}},$ which is
different from $\infty$ only if the path is a NAP.

If $\pi$ is a never increasing sequence, then $w_{k}\left[
\operatorname*{first}_{k}(\pi)\right]  =\operatorname*{first}_{k}(\pi)$%

\end{frame}%
%

\begin{frame}%
%

\frametitle{Shortest paths in the graph}%

The shortest paths for the lexicographic distance between any couple of nodes
may be computed thanks to $A^{n}$ or $A^{\left(  n\right)  }=E\boxplus
A^{1}\boxplus A^{2}\boxplus.........A^{n}$ .

\begin{lemma}
$A_{ij}^{n}=\sum\limits_{\mu\in C_{ij}^{n}}^{\boxplus}w(\mu)$ and
$A_{ij}^{\left(  n\right)  }=\sum\limits_{\mu\in C_{ij}^{\left(  n\right)  }%
}^{\boxplus}w(\mu)$
\end{lemma}

where:

\begin{itemize}
\item $C_{ij}^{n}$ is the family of paths between $i$ and $j,$ containing
$n+1$ nodes (not necessarily distinct)

\item $C_{ij}^{\left(  n\right)  }$ is the family of paths between $i$ and
$j,$ containing at most $n+1$ nodes (not necessarily distinct)
\end{itemize}

Defining $A^{\prime}=E\boxplus A$, as $\boxplus$ est idempotent ($a\boxplus
a=a\qquad\forall a\in S$), we have: $A^{\prime n}=A^{\left(  n\right)  }$%

\end{frame}%
%

\begin{frame}%
%

\frametitle{A path algebra on a dio\"{\i}d structure}%

In a graph with $N$ nodes, an elementary path has at most $N$ nodes, separated
by $N-1\ $edges. Hence, necessarily $A^{\left(  N\right)  }=A^{\left(
N-1\right)  }$ and $A^{\ast},$ the limit of $A^{\left(  n\right)  }$ for
increasing $n,$ is also equal to $A^{\left(  N-1\right)  }.$

Thanks to $A^{\ast},$ the computation of the catchment basins is
immediate.\ If $m_{1}$ is a regional minimum, then a node $i$ belongs to its
catchment basin if and only if $A_{im_{1}}^{\ast}\leq\sum\limits_{m_{j}\neq
m_{1}}^{\boxplus}A_{im_{j}}^{\ast}$

$A^{\ast}$ may be obtained by the successive multiplications :

$A,$ $A^{2}=A\boxtimes_{k}A,$ $A^{4}=A^{2}\boxtimes_{k}A^{2},...A^{2^{i}%
}=A^{2^{i-1}}\boxtimes_{k}A^{2^{i-1}}$ until $2^{i}\geq N-1$,, i.e. $i\geq
\log\left(  N-1\right)  $%

\end{frame}%
%

\begin{frame}%
%

\frametitle{A path algebra on a dio\"{\i}d structure}%

$A^{\ast}$ verifies\ $A^{\ast}=E\boxplus A\boxtimes_{k}A^{\ast}.$

Multiplying by a matrix $B$: $A^{\ast}B=B\boxplus A\boxtimes_{k}A^{\ast}B$.
Defining $Y=A^{\ast}B$ shows that $A^{\ast}B$ is solution of the equation
$Y=B\boxplus A\boxtimes_{k}Y.\;$Furthermore, it is the smallest solution.\ 

With varying $B$ it is possible to solve all types of shortest distances:

\begin{itemize}
\item The smallest solution of $Y=E\boxplus A\boxtimes_{k}Y$ yields $A^{\ast
}E=A^{\ast}$

\item with $B=\left[
\begin{array}
[c]{c}%
\varepsilon\\
\vdots\\
0\\
\vdots\\
\varepsilon
\end{array}
\right]  ,$ one gets $A^{\ast}B$, the i-th column of the matrix $A^{\ast},$
that is the distance of all nodes to the node $i.$\ 
\end{itemize}

%

\end{frame}%
%

\begin{frame}%
%

\frametitle{Solving linear systems}%

Gondran and Minoux have shown that most of the classical algorithms solving
systems of linear equations (Gauss, Gauss-Seidel, etc.) are still valid in
this context and correspond often to known shortest paths algorithms defined
on graphs.\ We now give a few examples.%

\end{frame}%
%

\begin{frame}%
%

\frametitle{The Jacobi algorithm}%

Setting $B=\left[
\begin{array}
[c]{c}%
0\\
\varepsilon\\
\vdots\\
\vdots\\
\varepsilon
\end{array}
\right]  \ $and $Y^{\left(  0\right)  }=\left[
\begin{array}
[c]{c}%
\varepsilon\\
\vdots\\
\varepsilon\\
\vdots\\
\varepsilon
\end{array}
\right]  $, the iteration $Y^{\left(  k\right)  }=A\ast Y^{\left(  k-1\right)
}\boxplus B$ converges to $Y^{\left(  N\right)  }=A^{\ast}\ast B$%

\end{frame}%
%

\begin{frame}%
%

\frametitle{The Gauss Seidel algorithm}%

$A$ is decomposed as $A=L\boxplus E\boxplus U$, where $L\ $is an inferior
triangular matrix, $E$ the unity matrix $E,$ and $U$ an upper triangular
matrix.\ The upper part of $L$ and lower part of $U$ have the value
$\varepsilon.$

The solution of $Y=A\ast Y\boxplus B$ is obtained by the iteration :
$Y^{(k)}=LY^{(k-1)}\boxplus UY^{(k)}\boxplus B$

This algorithm is faster as Jacobi's algorithm, as the product $UY^{(k)}$
already uses intermediate results, freshly computed during the current
iteration $Y^{(k)}.$%

\end{frame}%
%

\begin{frame}%
%

\frametitle{The Jordan algorithm}%

The Jordan algorithm is used in classical linear algebra for inverting
matrices, by successive pivoting.\ In our case, where the shortest paths are
elementary path the algorithms is:

For $k$ from $1$ to $N:$

\qquad For each $i$ and $j$ from $i$ to $N$, do: $a_{ij}=a_{ij}\boxplus
a_{ik}\ast a_{kj}$%

\end{frame}%
%

\begin{frame}%
%

\frametitle{The greedy algorithm of Moore-Dijkstra}%

Gondran established the algebraic counterpart of the famous shortest path
algorithm by Moore-Dijkstra.

\textbf{Theorem} (Gondran): Let $\overline{Y}=A^{\ast}B$ be the solution of
$Y=AY\boxplus B$, for an arbitrary matrix $B$.\ There exists then an index
$i_{0}$ such that $\overline{y_{i_{0}}}=\sum\limits^{\boxplus}b_{i}.$

The smallest $b$ is such solution : $\overline{y_{i_{0}}}=b_{i_{0}}$

Each element of $Y=AY\boxplus B$ is computed by $y_{k}=\sum\limits_{j\neq
k}^{\boxplus}a_{kj}\ast y_{j}\boxplus b_{k}=\sum\limits_{j\neq k,i_{0}%
}^{\boxplus}a_{kj}\ast y_{j}\boxplus a_{ki_{0}}y_{i_{0}}\boxplus b_{k}$

Suppressing the line and column of rank $i_{0}$ and taking for $b$ the vector
$b_{k}^{(1)}=a_{ki_{0}}y_{i_{0}}\boxplus b_{k},$ one gets a new system of size
$N-1$ to solve.

The Moore Dijkstra algorithm can also directly be computed on a flooding graph
$G,$ as presented below.%

\end{frame}%
%

\begin{frame}%

\begin{center}
{\Large \alert{Shortest paths algorithms on the graph}}
\end{center}

\begin{itemize}
\item The shortest path algorithm of Dijkstra is first presented
\cite{moore}.\ We show that for a lexicographic distance of depth $1,$ it
becomes algorithm for constructing the minimum spanning forest of Prim.

\item The core expanding algorithms take advantage of the particular structure
of the lexicographic distances.\ They are faster than the Dijkstra algorithm
and better suited to hardware implementations.\ For a lexicographic distance
of depth $2,$ they produce the same geodesics as the topographic distance.\ 
\end{itemize}

%

\end{frame}%
%

\begin{frame}%
%

\frametitle{Architecture of shortest path algorithms.}%

The shortest path algorithms below are applied on a graph which is invariant
by the opening $\gamma_{e}$ (the regional minima may ad libitum be contracted
beforehand).\ A domain $D$ is used and expanded, containing at each stage of
the algorithms the nodes for which the shortest distance to the minima is
known.\ Initially the minima are labeled and put in $D$ with a value $0.$ We
say that a flooding pair $(j,jl)$ is on the boundary of the domain $D,$ if $j$
is outside $D$ and $l$ inside $D.$ In this case we say that "$j$ floods $l",$
or $"l$ is flooded by $j".$ We say that $j$ belongs to the outside boundary
$\partial^{+}D$ of $D$ and $l$ to the inside boundary $\partial^{-}D$ of $D$.

The domain $D$ is progressively expanded by progressive incorporation of nodes
in $\partial^{+}D$ belonging to flooding pairs on the boundary of $D.$%

\end{frame}%
%

\begin{frame}%
%

\frametitle{The Moore Dijkstra algorithm}%

The shortest path algorithm by Moore-Dijkstra constructs the distances of all
nodes to the minima in a greedy manner. It uses a domain $D$ which contains at
each stage of the algorithm the nodes $i$ for which the shortest distance
$\delta_{k}^{\ast}(i)$ and label $\lambda(i)$ is known.

\textbf{Initialisation: }

The nodes of the regional minima are labeled and put in the\ domain $D.$

\textbf{Repeat until the domain }$D$\textbf{ contains all nodes: }

For each flooding pair $(j,(jl)$ on the boundary of $D,$ estimate the shortest
path as $\delta_{k}(j)=e_{jl}\boxtimes_{k}\delta_{k}^{\ast}(l).\ $

The node with the lowest estimate is correctly estimated. If the corresponding
flooding pair is $(s,st):$

\begin{itemize}
\item $D=D\cup\left\{  s\right\}  $

\item $\delta_{k}^{\ast}(s)=\delta_{k}(s)$

\item $\lambda(s)=\lambda(t)$
\end{itemize}

%

\end{frame}%
%

\begin{frame}%
%

\frametitle{Correctness of the Moore Dijkstra algorithm}%

The node with the lowest estimate is correctly estimated and is introduced in
the domain $D.\ $If is necessarily the shortest path, as any other path would
have to leave $D$ through another boundary flooding pair with a higher estimate.%

\end{frame}%
%

\begin{frame}%
%

\frametitle{Controlling the Moore Dijkstra algorithm}%

The Moore Dijkstra may be advantageously controlled by a hierarchical queue
structure.\ Each node, as it gets is estimate is put into the HQ. As long the
HQ\ is not empty, the extracted node is among the nodes with the lowest
estimate, the one which has been introduced in the HQ first.

The HQ has thus the advantage to correctly sequence the treatment in the
presence of plateaus: it treats the nodes in the plateaus from their lower
boundary inwards.\ The processing order is proportional to the distance of
each node to the lower boundary of the plateau.%

\end{frame}%
%

\begin{frame}%
%

\frametitle{The Moore Dijkstra algorithm : case where $k=1$}%

The algorithms remains the same but the computations are simplified as
$\delta_{k}(j)=e_{jl}\boxtimes_{1}\delta_{1}^{\ast}(l)$ is simply the weight
of the node and the edge in the flooding pair $(j,(jl).\ $In other terms, the
domain $D$ is expanded by introducing into $D$ the flooding pair on the
boundary of $D$ with the lowest weight.\ This corresponds exactly to the
algorithm of PRIM\ for constructing minimum spanning forests.\ The same
algorithm has been used in \cite{waterfallsbs} for constructing the waterfall hierarchy.

This is not surprising, as for $k=1,$ the lexicographic weight of a NAP simply
is the weight of the first edge.\ The distance is in this case the ultrametric
flooding distance.\ 

\textbf{Remark:} There is a complete freedom in the choice of the flooding
pairs which are introduced at any time into $D.$ A\ huge number of solutions
are compatible with this distance, some of them quite unexpected as
illustrated by an example given above.%

\end{frame}%
%

\begin{frame}%
%

\frametitle{The core expanding shortest distance algorithm}%

Due to the particular structure of lexicographic distances, it is possible to
identify another type of nodes for which the shortest distance may immediately
be computed.\ Let $\underbrace{\partial^{-}D}$ be the set of nodes in the
boundary $\partial^{-}D$ with the lowest valuation $\operatorname*{first}%
_{k-1}(d_{k}^{\ast})$.\ For each $t\in\underbrace{\partial^{-}D}$ and for each
$s\in\partial^{+}D$ flooding $t$ we do $e_{ts}\boxtimes_{k}%
\operatorname*{first}_{k-1}\delta_{k}^{\ast}(t)$ producing a NAP of length
$k.$ The value of $\delta_{k}^{\ast}(t)\operatorname*{first}_{k-1}$ is simply
obtained by appending the weight of $s$ to $\operatorname*{first}_{k-1}%
\delta_{k}^{\ast}(t).\ $As $t$ belongs to $\underbrace{\partial^{-}D},$ this
value is the smallest possible.\ 

This analysis shows that each node of $\underbrace{\partial^{-}D}$ permits to
introduce into $D,$ all its neighbors by which it is flooded, whatever their
weight.\ This algorithm is more "active" as Dijkstra's algorithm, as each node
inside $D_{k}$ is able to label and introduce into $D$ all its flooding
neighbors at the same time.\ %

\end{frame}%
%

\begin{frame}%
%

\frametitle{The core expanding shortest distance algorithm}%

The algorithm is the following;

\textbf{Initialisation: }

The nodes of the regional minima are labeled and put in the\ domain $D.$

\textbf{Repeat until the domain }$D$\textbf{ contains all nodes: }

Let $\underbrace{\partial^{-}D}$ be the subdomain of $D$ of nodes flooded by
nodes outside $D$ with the lowest valuation $\operatorname*{first}_{k-1}%
(d_{k}^{\ast})$.\ For each $t\in\underbrace{\partial^{-}D}$ and $t$ is flooded
by $\in\partial^{+}D$ :

\begin{itemize}
\item $D=D\cup\left\{  s\right\}  $

\item $\delta_{k}^{\ast}(s)=e_{ts}\boxtimes_{k}\operatorname*{first}%
_{k-1}\delta_{k}^{\ast}(t)$

\item $\lambda(s)=\lambda(t)$
\end{itemize}

%

\end{frame}%
%

\begin{frame}%
%

\frametitle{Controlling the core expanding shortest distance algorithm}%

The core expanding shortest distance algorithm may also be advantageously
controlled by a hierarchical queue structure.\ Each node, as it is introduced
in $D$ is put into the HQ. As long the HQ\ is not empty, the extracted node is
among the nodes with the lowest estimate, the one which has been introduced in
the HQ first. This node gives its label and the correct distances to all its
neighbors outside $D$ by which it is flooded.\ 

If a node extracted from the HQ belongs to a plateau, as it has been
introduced in the HQ, it is closer to the lower boundary of the plateau than
other nodes which may have been introduced in the HQ later.\ 

The algorithm is particularly suitable for a hardware implementation: each
node is entered in the HQ only once.\ On the contrary, with the Moore-Dijkstra
algorithm a node may be introduced several times in the HQ, as its estimate
may vary before it gets its definitive value.\ %

\end{frame}%
%

\begin{frame}%
%

\frametitle{The core expanding shortest distance algorithm: case where $k=1$}%

The domain $\underbrace{\partial^{-}D}$ has been defined as the subdomain of
$D$ of nodes flooded by nodes outside $D$ with the lowest valuation
$\operatorname*{first}_{k-1}(d_{k}^{\ast}).\ $For $k=1,$ we have
$\operatorname*{first}_{k-1}(d_{k}^{\ast})$ is empty, and $\underbrace
{\partial^{-}D}$ occupies the whole domain $\partial^{-}D.\ \ $This means that
any node in $\partial^{-}D$ can be expanded by appending one of the outside
node through which it is flooded. The algorithms for constructing graph cuts
by J.Cousty find also their place in this context \cite{Coustywshedcut}

We have illustrated this situation earlier, showing that the minimum spanning
forests with $0$ steepness may be absolutely unexpected.

Controling the algorithm with a hierarchical queue limits the anarchy to some extend.%

\end{frame}%
%

\begin{frame}%
%

\frametitle{The core expanding shortest distance algorithm: case where $k=2$}%

If $k=2,$ then the valuation $\operatorname*{first}_{k-1}(d_{k}^{\ast\prime
}i))=\operatorname*{first}_{1}(d_{2}^{\ast}(i))$ is the valuation of the node
$i$ itself.\ This means that $\underbrace{\partial^{-}D}$ contains the nodes
with lowest weight belonging to $\partial^{-}D.\ $If $i$ is such a node, it
introduces into $D$ each of its neighbors $j$ belonging to $\partial^{+}D,$
each with a valuation $\operatorname*{first}_{1}(d_{2}^{\ast}(j))$ equal to
its weight $n_{j}.$\ 

If in addition, one uses a HQ\ for controlling the process, we get the
classical algorithm for constructing catchment basins \ \cite{meyer91}%
,\cite{meyer94}.%

\end{frame}%
%

\begin{frame}%
%

\frametitle{The classical watershed algorithm}%

Label the nodes of the minima and them in the domain $D,$ each with a priority
with a weight.\ 

As long as the HQ\ is not empty, extract the node $j$ with the highest
priority from the HQ:

For each unlabeled neighboring (on the flooding graph) node $i$ of $j:$

\qquad* $label(i)=label(j)$

\qquad* put $i$ in the queue with priority $\ \nu_{i}$

As a matter of fact, this algorithm has first been derived from the watershed
line, as zone of influence of the minima for the topographic distance, defined below.\ %

\end{frame}%
%

\begin{frame}%
%

\frametitle{The topographic distance}%

Consider an arbitrary path $\pi=(x_{1},x_{2},...,x_{p})$ between two nodes
$x_{1}$ and $x_{n}.$ The weight $\nu_{p}$ at node $x_{p}$ can be written:

$\nu_{p}=\nu_{p}-\nu_{p-1}+\nu_{p-1}-\nu_{p-2}+\nu_{p-2}-\nu_{p-3}%
+....+\nu_{2}-\nu_{1}+\nu_{1}$

The node $k-1$ is not necessarily the lowest node of node $k,$ therefore
$\nu_{k-1}\geq\varepsilon_{n}\nu_{k}$ and $\nu_{k}-\nu_{k-1}\leq\nu
_{k}-\varepsilon_{n}\nu_{k}.$

Replacing each increment $\nu_{k}-\nu_{k-1}$ by $\nu_{k}-\varepsilon_{n}%
\nu_{k}$ will produce a sum $\nu_{p}-\varepsilon_{n}\nu_{p}+\nu_{p-1}%
-\varepsilon_{n}\nu_{p-1}+....+\nu_{2}-\varepsilon_{n}\nu_{2}+\nu_{1}$ which
is larger than $\nu_{p}.\ $It is called the topographic length of the path
$\pi=(x_{1},x_{2},...,x_{p}).$ The path with the shortest topographic length
between two nodes is called the topographic distance between these nodes.\ %

\end{frame}%
%

\begin{frame}%
%

\frametitle{The topographic distance}%

The path with the shortest topographic length between two nodes is called the
topographic distance between these nodes \cite{Najman199499},\cite{meyer94}%
.\ It will only be equal to $\nu_{p}$ in the case where the path $(x_{1}%
,x_{2},...,x_{p})$ precisely is a path of steepest descent, from each node to
its lowest neighbor. In other terms the topographic distance and the distance
$\operatorname*{lexdist}_{2}$ have the same geodesics.\ 

If we define the toll to pay along a path $\pi=(x_{1},x_{2},...,x_{p})$ as
$\nu_{1}$ for the first node and $\nu_{i}-\varepsilon_{n}\nu_{i}$ for the
others, then the lowest toll distance for a node $x_{p}$ to a regional minimum
will be $\nu_{p}$ if there exists a path of steepest descent from $x_{p}$ to
$x_{1}.\ $In other terms, $x_{p}$ and $x_{1}$ belong to the same catchment
basins if one considers the topographic distance.\ But they also belong to the
same catchment basin if one considers the depth $2$ lexicographic distance,
as, by construction, $x_{k}$ is the lowest neighbor of $x_{k-1}.$%

\end{frame}%
%

\begin{frame}%
%

\frametitle{Distance on a node weighted graph based on the cost for travelling
along the cheapest path}%

We recall the toll distance presented earlier.

The weights are assigned to the nodes and not to the edges. Each node may be
considered as a town where a toll has to be paid.

\textbf{Cost of a path:} The cost of a path is equal to the sum of the tolls
to be paid in all towns encountered along the path (including or not one or
both ends).

\textbf{Cost between two nodes:} The cost for reaching node $y$ from node $x$
is equal to the minimal cost of all paths between $x$ and $y$. We write
$\operatorname*{tolldist}(x,y).\ $If there is no path between them, the cost
is equal to $\infty$.%

\end{frame}%
%

\begin{frame}%
%

\frametitle{Illustration of the cheapest path}%

\textbf{%
\begin{figure}
[ptb]
\begin{center}
\includegraphics[
height=1.4944in,
width=1.9976in
]%
{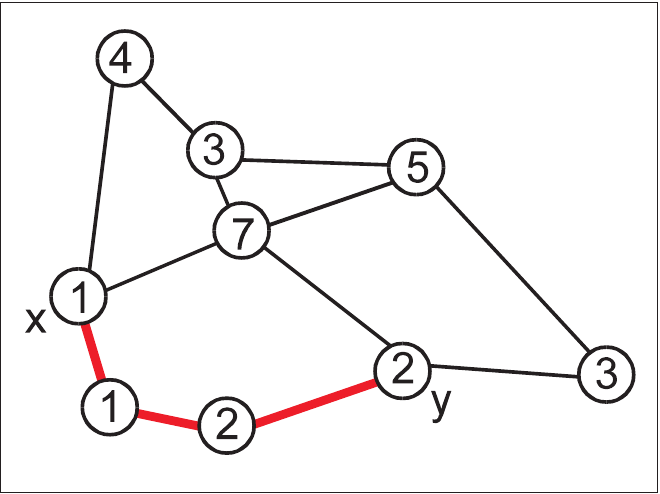}%
\end{center}
\end{figure}
}

In this figure, the shortest chain between $x$ and $y$ is in red and the total
toll to pay is $1+1+2+2=6$%

\end{frame}%
%

\begin{frame}%
%

\frametitle{Reconstruction of an image by integration}%

Finally, any image $f$ may be considered as the global toll of its pixel graph
if one takes:

\begin{itemize}
\item as reference nodes the regional minima of the image. Each of them has as
toll its altitude.

\item as local toll for all other nodes, the difference between their altitude
and the altitude or their lowest neighbor: $g=f-\varepsilon f$
\end{itemize}

If in addition, we give a different label to each regional minimum, we may as
previously propagate this label along each smallest toll path. We obtain like
that a tessellation: to each minimum is ascribed a catchment basin: the set of
all nodes which are closer to this minimum then to any other minimum.%

\end{frame}%
%

\begin{frame}%
%

\frametitle{Inversely: the catchment basins of the cheapest paths for a
distribution of tolls. }%

Inversely we may chose a number of starting nodes called roots, with a toll to
pay and for all other nodes, the toll to pay for reaching or crossing this
node. The toll to pay constitutes a topographic surface where each root is a
regional minimum.\ In addition we propagate the labels of the minima along the
geodesics of the cheapest distance, and get a partition of the nodes. As a
result one get a partition of the space, where each region with a given label
is the catchment basin of the root with the same label for this distance function.\ %

\end{frame}%
%

\begin{frame}%
%

\frametitle{Assigning to each node the global toll for reaching it, starting
from one of the roots.}%
%

\begin{figure}
[ptb]
\begin{center}
\includegraphics[
height=2.1904in,
width=4.422in
]%
{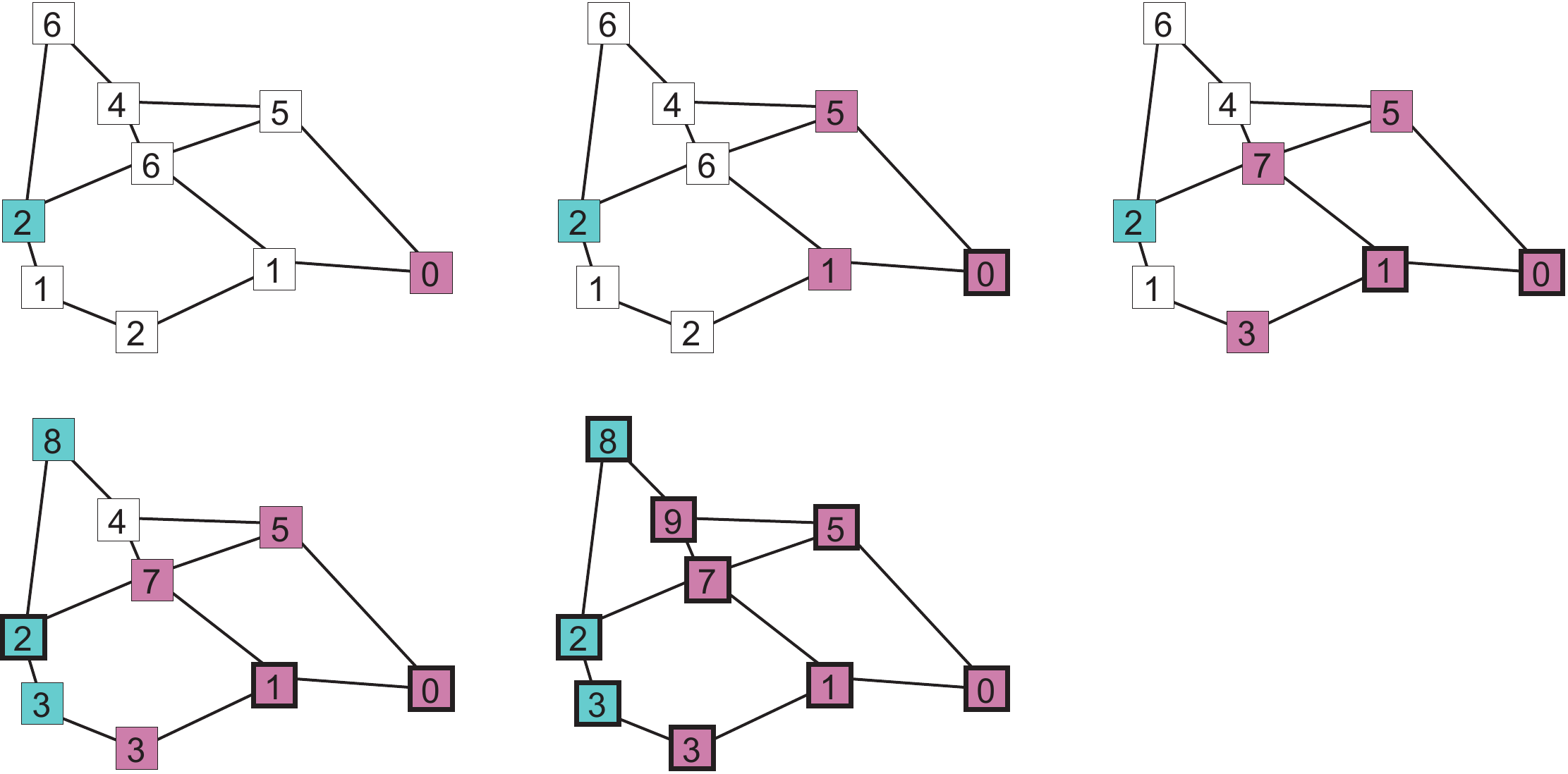}%
\end{center}
\end{figure}
%

\end{frame}%
%

\begin{frame}%
%

\frametitle{Illustration of the cheapest path}%

For each node we have indicated the local toll value (left value) and the
global toll to pay to reach the closest reference node, who shares the same
color (label):

\begin{itemize}
\item Each reference node became a regional minimum of the graph.

\item The local toll of any other node is equal to the difference between its
global toll and the global toll of its lowest neighbor.

\item Each non reference node got its value and label from one of its lowest neighbors.

\item The values of the nodes are computed and the labels propagated along a
path of steepest descent.
\end{itemize}

%

\begin{figure}
[ptb]
\begin{center}
\includegraphics[
height=1.0147in,
width=1.429in
]%
{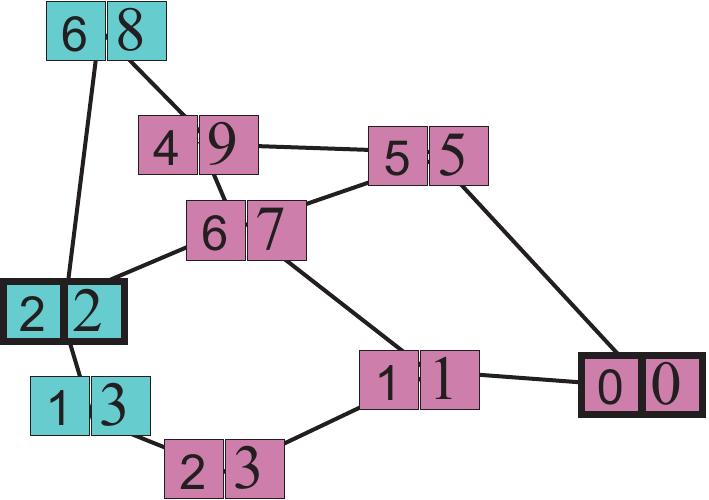}%
\end{center}
\end{figure}

The catchment basins of this surface are the SKIZ of the minima, both for the
topographic distance and for the depth 2 lexicographic distance.%

\end{frame}%
%

\begin{frame}%
%

\frametitle{Illustration of the cheapest path}%
%

\begin{figure}
[ptb]
\begin{center}
\includegraphics[
height=1.0147in,
width=1.429in
]%
{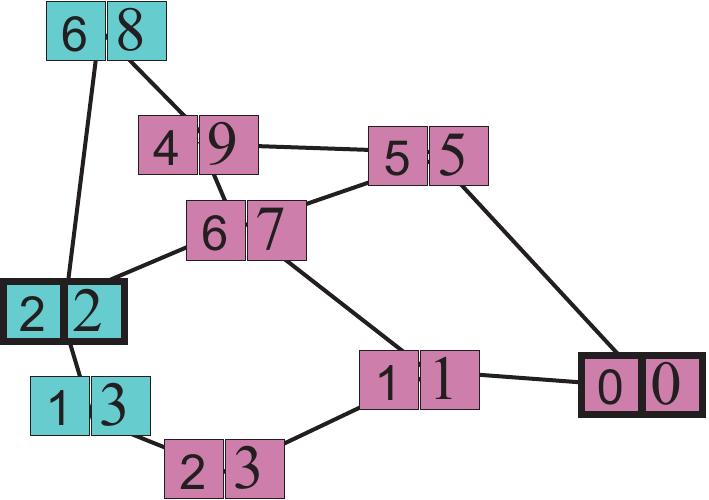}%
\end{center}
\end{figure}

The catchment basins of this surface are the SKIZ of the minima, both for the
topographic distance and for the depth 2 lexicographic distance. Each node is
the extremity of a geodesic line, which is the same both for the toll distance
and for the lexicographic distance of depth $2$ computed on the same
topographic surface.\ %

\end{frame}%
%

\begin{frame}%

\begin{center}
{\Large \alert{Top down or bottom up}}
\end{center}

%

\end{frame}%
%

\begin{frame}%
%

\frametitle{The influence of the depth k}%

For increasing values of $k,$ the domain $\underbrace{\partial^{-}D}$ becomes
smaller and smaller, indicating that the number of equivalent catchment basins
compatible with a given lexicographic depth is reduced as the value of $k$
becomes bigger. This is in accordance with the fact that with increasing
values of $k,$ the pruning $\downarrow^{k}$ becomes more and more severe.%

\end{frame}%
%

\begin{frame}%
%

\frametitle{Obtaining catchment basins with $k-steepness$}%

After applying the operator $\downarrow^{k}$ on a flooding graph $G,$ there
remain only NAP with $k$ steepness.\ 

The same is true if we apply the operator $\zeta^{(k-1)}$ in order to prune
edges of $G$ and subsequently restore the initial weights of the edges onto
the remaining edges.\ 

If the only NAPs\ remaining in the graph have a $k$ steepness, any shortest
path algorithm with a lower steepness will extract them In particular the most
simple algorithms for the distances $d_{1}$ or $d_{2}$ will extract catchment
basins of steepness $k.\ $%

\end{frame}%
%

\begin{frame}%
%

\frametitle{Illustration : lexicographic distances }%

The follownig 3 images show respectively the shortest lexicographic distances
of depth 1, 2 and 3.\ If a path is the shortest path for a lexicographic
distance of depth $k,$ it also is a shortest path for a lexicographic distance
of smaller depth.\ The following three figures present the lexicographic
distances of depth 1, 2 and 3.\ The partition of catchment basins for the
distance 3 is also solution for the distance \ 2 and 1.\ Similarly the
partition for distance 2 is also solution for distance 1. The number of
solutions decreases with the lexicographic depth.%

\end{frame}%
%

\begin{frame}%
%

\frametitle{Illustration : lexicographic distance of depth 1}%
%

\begin{figure}
[ptb]
\begin{center}
\includegraphics[
height=0.9564in,
width=3.7102in
]%
{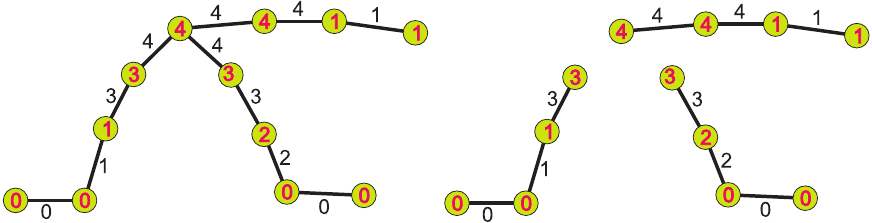}%
\end{center}
\end{figure}
%

\end{frame}%
%

\begin{frame}%
%

\frametitle{Illustration : lexicographic distance of depth 2}%
%

\begin{figure}
[ptb]
\begin{center}
\includegraphics[
height=0.9564in,
width=3.7102in
]%
{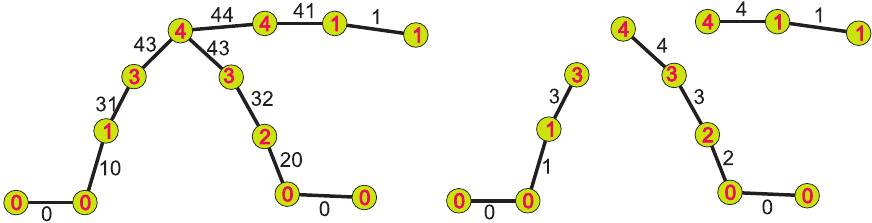}%
\end{center}
\end{figure}
%

\end{frame}%
%

\begin{frame}%
%

\frametitle{Illustration : lexicographic distance of depth 3}%
%

\begin{figure}
[ptb]
\begin{center}
\includegraphics[
height=0.9597in,
width=3.6945in
]%
{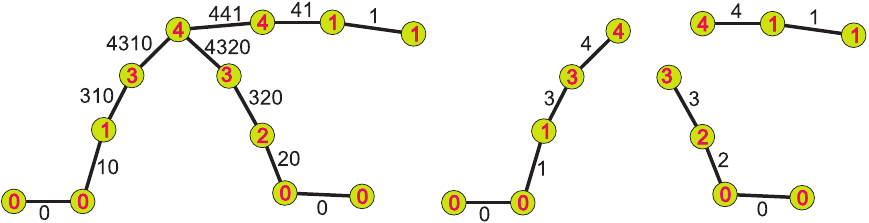}%
\end{center}
\end{figure}
%

\end{frame}%
%

\begin{frame}%
%

\frametitle{Erosion and pruning}%

Repeating the operator $\zeta G=$\ $(\downarrow\varepsilon_{e}e,\varepsilon
_{n}n)$ produces a decreasing series of partial graphs $\zeta^{(n)}%
G=\zeta\zeta^{(n-1)}G,$ which are steeper and steeper.\ In the following
figures, we present in red the edge which is not the lowest edge of one of its
extremities.\ After pruning this edge, the operator is applied
again.\ Applying $\zeta$ a number $n$ of times is equivalent to constructing
the partitions compatible with a SKIZ for a lexicographic distance of depth
$n+1.$ In our case, $\zeta^{(3)}G\ $produces a graph with the same edges as
the pruning $\chi$ applied to the graph where each edge has been weighted by
its lexicographic distance of depth $4$ to the nearest minimum, illustrated in
the previous figure.\ %

\end{frame}%
%

\begin{frame}%
%

\frametitle{Erosion and pruning}%
%

\begin{figure}
[ptb]
\begin{center}
\includegraphics[
height=1.9244in,
width=3.6247in
]%
{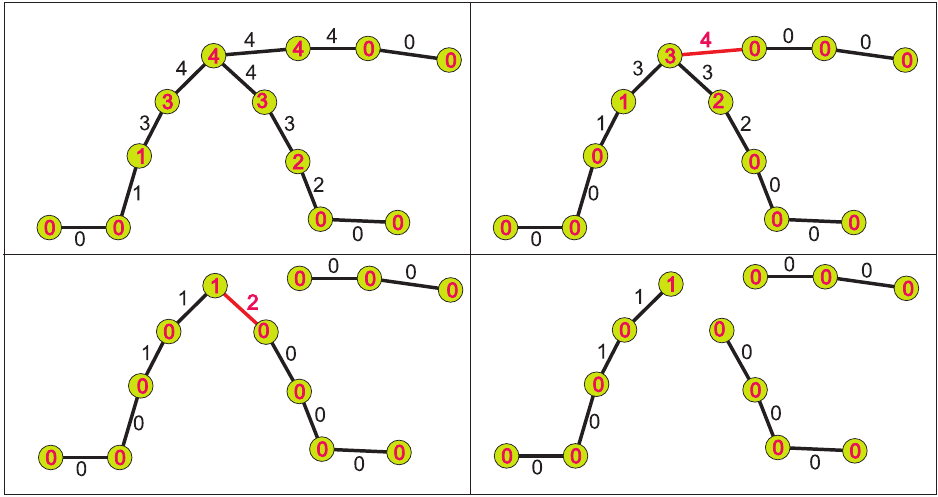}%
\end{center}
\end{figure}
%

\end{frame}%
%

\begin{frame}%

\begin{center}
{\Large \alert{The hierarchy of  nested catchment basins }}
\end{center}

%

\end{frame}%
%

\begin{frame}%
%

\frametitle{Watershed and waterfalls}%

The waterfall hierarchy has been introduced by S.Beucher in order to obtain a
multiscale segmentation of an image \cite{beucher90},\cite{waterfalls94}%
,\cite{waterfalls94}. Given a topographic surface $S_{1}$, typically a
gradient image of the image to segment, a first watershed transform produces a
first partition $\pi_{1}$.\ 

The waterfall flooding of $S_{1}$ consists in flooding each catchment basin up
to its lowest neighboring pass point, producing like that a topographic
surface with less minima.\ The watershed segmentation of this surface produces
a coarser partition $\pi_{2}$ where each region is the union of a number of
regions of $\pi_{1}.$%

\end{frame}%
%

\begin{frame}%
%

\frametitle{Chaining the waterfall floodings}%

Flooding each catchment basin of $S_{1}$ up to its lowest pass point produces
a new topographic surface $S_{2}$ which will be submitted to the same
treatment as the initial surface $S_{1}.$ Its watershed transform produces a
second partition $\pi_{2}.\ $The partition $\pi_{2}$ is coarser than $\pi_{1}$
as each of its tile is a union of tiles of $\pi_{1}.${}The following figure
shows how the flooding of $S_{1}$ produces $S_{2},$ which, in the second
figure, has also been flooded.

\begin{center}
\includegraphics[
natheight=1.127700in,
natwidth=5.732800in,
height=61.2727pt,
width=303.2121pt
]{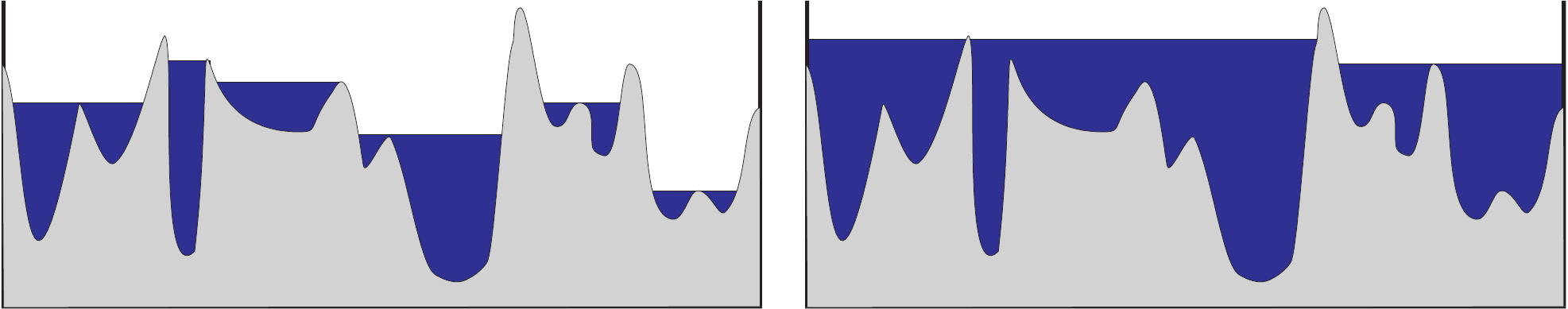}
\end{center}

The same process can be repeated several times until a completely flat surface
is created.\ The partitions obtained by the watershed construction on the
successive waterfall floodings are coarser and coarser : they form a hierarchy.%

\end{frame}%
%

\begin{frame}%
%

\frametitle{The waterfall hierarchy}%

A\ 1 dimensional topographic surface is represented through the altitude of
its pass points in the figure below.\ The first flooding assigns weights to
the nodes and its watershed construction produces 4 catchment basins,
separated by 3 watershed lines. The second flooding has only 2 catchment
basins separated by 1 watershed line.\ The last image orders the watershed
lines of the initial image into 3 categories, in cyan the watershed lines
which disappeared during the first flooding, in dark blue those which
disappeared after the second flooding and in red the one which survived the
second flooding. \
\begin{center}
\includegraphics[
height=1.3602in,
width=3.6555in
]%
{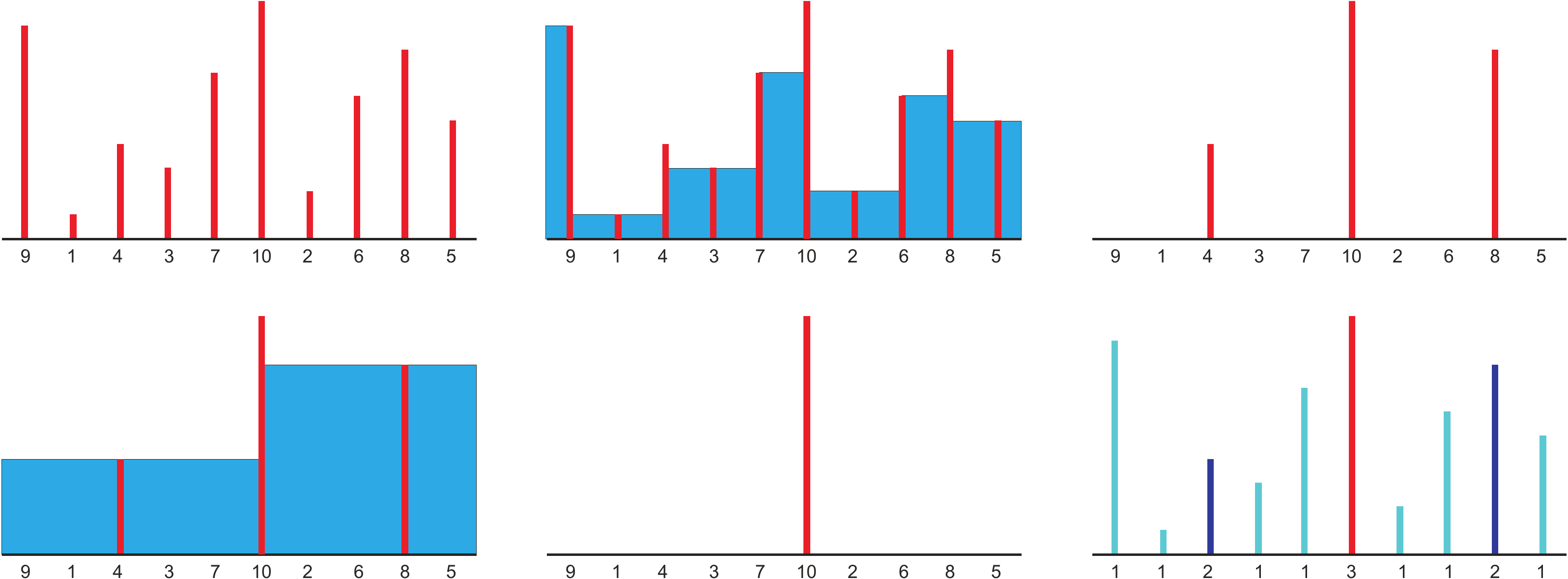}%
\\
{}%
\end{center}
%

\end{frame}%
%

\begin{frame}%
%

\frametitle{Watershed and waterfalls}%

Let us come back to the watershed on weighted graphs. The watershed of the
topographic surface produces a partition $\pi_{1}$ into catchment basins,
represented by its region adjacency graph $RAG_{1}$. The first flooding floods
each catchment basin up to its lowest neighboring pass point.\ This
corresponds to the erosion $\varepsilon_{ne}$ of the graph $RAG_{1}.$ The
theory of the watershed on weighted graphs can now be applied on this
graph.\ The resulting watershed appears in form of minimum spanning forest
$MSF_{1}$; each tree of the forest spans a catchment basin of the partition
$\pi_{2}.$

The next level of the hierarchy may be represented by a new region adjacency
graph $RAG_{2},$ whose nodes are obtained by contracting each tree of the
forest $MSF_{1}$ in the graph $RAG_{1}.$ Repeating the same treatment to the
graph $RAG_{2}$ produces the next level of the hierarchy, where each tree of
the minimum spanning forest $MSF_{2}$ has been obtained by merging several
trees of the $MSF_{1}.$ The process is then repeated until a graph is created
with only one regional minimum.\ %

\end{frame}%
%

\begin{frame}%
%

\frametitle{Construction of the level 1 of the hierarchy}%

We start with an arbitrary node or edge weighted and connected graph $G.\ $As
explained earlier, we extract a graph flooding graph $G^{\prime}.\ $For a
steepness $k,$ we prune $G^{\prime}$ and get $\downarrow^{k}G\prime$.\ The
scissor operator $\chi$ produces a minimum spanning forest $F_{1}%
=\chi\downarrow^{k}G\prime$, spanning the finest watershed partition, the
lowest level of the hierarchy.\ 

The graph representing the second level of the hierarchy is obtained by
contracting all edges of the forest $F_{1}$ ; each tree becomes one node.\ The
nodes are connected by edges of the graph $G.\ $The result of this contraction
$G_{1}^{\prime}=\kappa(G,F_{1}).$ is again a connected graph, to which the
same treatment as previously can be applied.\ %

\end{frame}%
%

\begin{frame}%
%

\frametitle{Construction of the level 2 of the hierarchy}%

Only the nodes of the graph $G_{1}^{\prime}$ are weighted : $G_{1}^{\prime}=$
$(e_{1},\diamond).\ $The nodes will be weighted with $\varepsilon_{ne}$ and we
get the graph $(e_{1},\varepsilon_{ne}e_{1}),$ which becomes a flooding graph
of steepness $k$ in $\downarrow^{k}(e_{1},\varepsilon_{ne}e_{1}).\ $A final
scissor operator creates a forest $F_{2}$ spanning nodes of $G_{1}^{\prime
}.\ F_{2}$ represents the level 2 of the hierarchy, and the nodes of $G$
spanned by each of its trees constitutes a catchment basin of level \ 2.\ 

The contraction $\kappa(G_{1}^{\prime},F_{1})$ produces again a connected
graph $G_{2}^{\prime}=(e_{1},\diamond)$ to which the same treatment may be applied.

This process is repeated until the graph $\kappa(G_{m}^{\prime},F_{m})$
contains a unique regional minimum of edges.\ %

\end{frame}%
%

\begin{frame}%
%

\frametitle{Emergence of a minimum spanning tree}%

Each new minimum spanning forest $F_{n}$ makes use of new edges of the initial
graph.\ The union of all MSF constructed up to the iteration $m$, $%
{\textstyle\bigcup\limits_{k\leq m}}
F_{k},$ is a minimum spanning forest of the initial graph $G,$ which converges
to a MST of $G$ as $m$ increases.\ Hence the union of all edges present in the
series $F_{n}$ is a minimum spanning tree of the graph $G.$

The particular minimum spanning tree which emerges depends upon the depth $k$
of the pruning operator $\downarrow^{k}$, and of the particular choices among
alternative solutions made by the pruning operators $\chi$ used at each
step.\ We call the operator extracting the minimum spanning tree $\mu(G).$

The figure in the next slide presents a RAG, its flooding graph in the first
row, and in the second row the MSFs $%
{\textstyle\bigcup\limits_{k\leq m}}
F_{k}$ for $m=1,2,3.\ $The last one being the MST.%

\end{frame}%
%

\begin{frame}%
%

\frametitle{Emergence of a minimum spanning tree}%
%

\begin{figure}
[ptb]
\begin{center}
\includegraphics[
height=2.6902in,
width=4.1293in
]%
{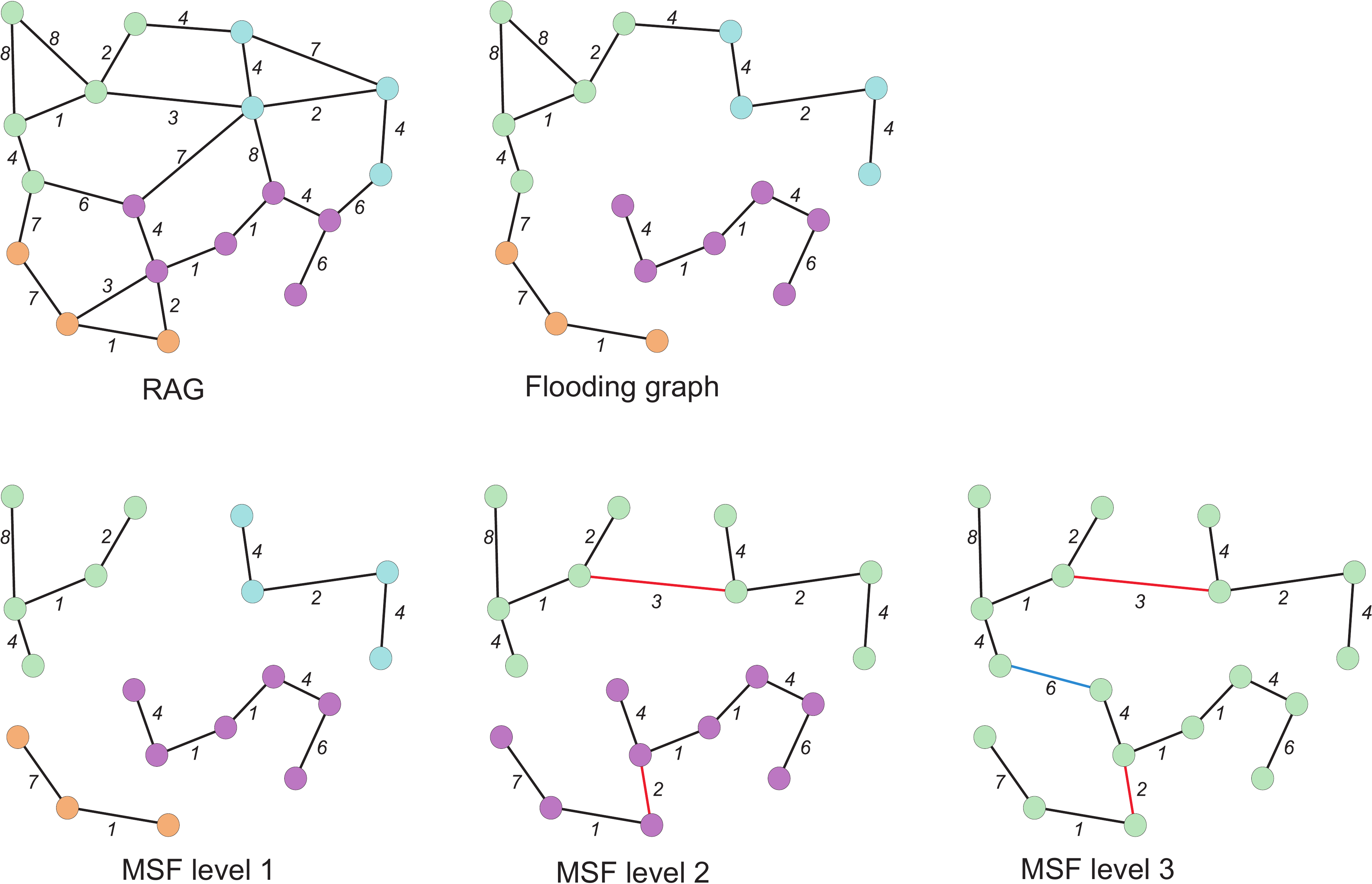}%
\end{center}
\end{figure}
%

\end{frame}%
%

\begin{frame}%
%

\frametitle{Emergence of a minimum spanning tree}%

If $G$ and $G^{\prime}$ are two flooding trees, $G^{\prime}$ being a partial
tree of $G$ included in $G,$ then the pruning operators $\chi$ have less
choices for pruning $G\prime$ as for pruning $G.\ \ $For this reason the
family of MST derived for $G^{\prime}$ is included in the family of MST
derived for $G$: $\left\{  \mu(G^{\prime})\right\}  \subset\left\{
\mu(G)\right\}  .$

In particular if one considers the decreasing sequence of minimum spanning
forests $\chi\downarrow^{k}G$, one obtains decreasing families of MST : for
$k<l,$ we have $\left\{  \mu(\downarrow^{l}G)\right\}  \subset\left\{
\mu(\downarrow^{k}G)\right\}  .$ The choices are more and more
constrained.\ One gets minimum spanning trees which are steeper and steeper
for increasing $k$ in $\downarrow^{k}G.$%

\end{frame}%
%

\begin{frame}%
%

\frametitle{Discussion }%

All MSTs of an edge weighted graph share a fundamental property : between any
two nodes of the graph, there exists a unique path in the MST\ which links
them, and this path has a minimal flooding weight. Thanks to this property,
replacing the graph $G$ by its MST $T$ permits to construct the catchment
basins linked to the $\operatorname*{lexdist}_{1}.$ This procedure is fast but
has no control on the quality of the result if one chooses an MST at random
(the one which is produced by the preferred MST extraction algorithm).\ Fast,
as there are as the number of edges is strongly reduced (in a tree : number of
edges = number of nodes -1). Limiting, as one has no control on the steepness
of the watershed which is ultimately extracted.\ %

\end{frame}%
%

\begin{frame}%
%

\frametitle{Discussion }%

In order to get high quality partitions and segmentations, one has to
carefully chose the MST. The preceding slide has shown that the MSTs form a
decreasing family of trees as their steepness increases. Higher quality
segmentations will be obtained for a higher steepnes.\ In order to improve
things, one may prune the graph $G$ with the operator $\zeta_{k}$ before
extracting one of its MST, yielding a MST of k-steepness.

One has then to consider a trade-off between the required speed and the
quality of results one targets. This obviously also depends on the type of
graphs. Some graphs contain only very few MST ; in the extreme case, a graph
where all edges have distinct weights, contains a unique MST.%

\end{frame}%
%

\begin{frame}%

\begin{center}
{\Large \alert{Conclusion}}
\end{center}

%

\end{frame}%
%

\begin{frame}%
%

\frametitle{Outcome}%

Starting with the flooding adjunction, we have introduced the flooding graphs,
for which node and edge weights may be deduced one from the other.

Each node weighted or edge weighted graph may be transformed in a flooding
graph, showing that there is no superiority in using one or the other, both
being equivalent.

We have introduced pruning operators $\downarrow^{k}$ and $\zeta^{(k)}$ which
extract subgraphs of increasing steepness.\ For an increasing steepness, the
number of never ascending paths becomes smaller and smaller.\ This reduces the
watershed zone, where catchment basins overlap.

The scissor operator $\chi,$ associating to each node outside the regional
minima one and only one edge choses a particular watershed partition. Again,
with an increasing steepness, the number of equivalent solutions becomes
smaller. Ultimaterly, for natural image, an infinite steepness leads to a
unique solution, as it is not likely that two absolutely identical non
ascending paths of infinite steepness connect a node with two distinct minima.%

\end{frame}%
%

\begin{frame}%
%

\frametitle{Outcome}%

Finally, we have shown that the NAP paths remaining after the pruning
$\downarrow^{k}$ or $\zeta^{(k)}$ are identical with geodesics of
lexicographic distances.

We have shown how the path algebra may be adapted to lexicographic distances.

We have presented the Moore-Dijkstra algorithm for constructing skeletons by
zone of influence for lexicographic distances.\ 

For the depth 1, the lexicographic distance becomes the flooding distance and
the algorithms become classical algorithms for constructing minimum spanning forests.

For the depth 2, it is equivalent with the topographic distance.\ 

The beneficial effect of hierarchical queue algorithms is highlighted, as it
permits to correctly divide plateaus among neighboring catchment basins.

We also presented the core expanding algorithms which are particularly
efficient for lexicographic distances.%

\end{frame}%
%

\begin{frame}%
%

\frametitle{Outcome}%

The waterfall hierarchy is obtained by contracting the trees of a first
watershed construction and the new graph submitted to a novel watershed
construction.\ The process is iterated until only one region remains.

The union of the edges of all forests produces constitute a minimum spanning
tree of the initial graph.

MST and MSF\ defined before only for edge weighted graphs may be extended to
node weighted graphs.

MST and MSF\ may be ordered into nested classes according to their steepness.

The classical order is thus inverted : classically a MST is extracted from the
graph as a first step, and a MSF derived from it, by cutting a number of its
edges.\ This way of doing leaves not much control as the quality of the result
depends upon the choice of the MST.\ 

As MST and MSF\ may be ordered into nested classes according to their steepness.%

\end{frame}%
%

\begin{frame}%
%

\frametitle{Outcome}%

The algebraic approach to the watershed presented here sets the stage in which
a number of earlier definitions and algorithms may be reinterpreted : 

\begin{itemize}
\item the waterfall algorithm \cite{waterfalls94}

\item the watershed line defined a zone of influence of the minima for various
distances, in particular the flooding distance and the topographic distance
\cite{meyer94},\cite{Najman199499}.\ In the case of the flooding distance
applied on an edge weighted graph, one finds the graph cuts
\cite{Coustywshedcut}, and the algorithm for constructing a waterfall
hierarchy described in \cite{waterfallsbs}

\item the role of the hierarchical queues for a correct division of plateaus
between catchment basins if one uses myopic distances \cite{meyer91}
\end{itemize}

We also explored the place of the choice for constructing a watershed
partition.\ The choice becomes more and more restricted as one considers
lexicographic distances with increasing depth. 

The waterfall hierarchy transposed  to MST provides a hierarchical
decomposition of the MST. %

\end{frame}%
%

\begin{frame}%
%

\frametitle{Outcome}%

Old algorithms may be reinterpreted, in particular those relying on pruning
and labeling graphs \cite{Bieniek2000907}, \cite{lemonth}, \cite{Noguet}%
,\cite{marcinwshed}.\ 

The classical algorithm for constructing watersheds derived from the
hierarchical queues is a particular case of core expanding algorithm. 

In order to reduce the number of equivalent watershed partitions one is able
to extract from a flooding graph, one may use a mixture of pruning
$\zeta^{(k)}$ up to a given depth, and on the resulting graph apply a myopic
algorithm, with a lexicographic depth of 1\ or 2.%

\end{frame}%
%

\begin{frame}%
%

\frametitle{References}%

The watershed transform introduced by S.Beucher and C.\ Lantuejoul
\cite{beucher79} is one of the major image segmentation tools, used in the
community of mathematical morphology and beyond. If the watershed is a
successful concept, there is another side of the coin: a number of definitions
and algorithms coexist, claiming to construct a wartershed line or catchment
basins, although they obviously are not equivalent. We have presented how the
idea was conceptualized and implemented as algorithms or hardware solutions in
a brief note :" The watershed concept and its use in segmentation : a brief
history" (arXiv:1202.0216v1), which contains an extensive bibliography. Here
we give a more restricted list of references.%

\end{frame}%


\setbeamertemplate{bibliography item}[text]

\end{document}